\documentclass[lettersize,journal]{IEEEtran}
\usepackage[caption=false,font=footnotesize]{subfig}

\usepackage[
    colorlinks=true,
    linkcolor=blue,
    citecolor=blue,
    urlcolor=blue
]{hyperref}

\usepackage{amsmath,amsfonts}
\usepackage{algorithmic}
\usepackage{algorithm}
\usepackage{array}
\usepackage{textcomp}
\usepackage{url}
\usepackage{verbatim}
\usepackage{graphicx}
\usepackage{cite}
\usepackage{changepage}
\usepackage{booktabs}
\usepackage{multirow}
\usepackage{amssymb} 
\usepackage{cases}
\usepackage{tikz}
\usetikzlibrary{angles,quotes,calc}

\newtheorem{theorem}{Theorem}

\newtheorem{lemma}{Lemma}

\newtheorem{remark}{Remark}

\newtheorem{definition}{Definition}

\newtheorem{problem}{Problem}
\newtheorem{proposition}{Proposition}
\newtheorem{example}{Example}
\usepackage[colorlinks=true, linkcolor=blue, citecolor=blue, urlcolor=blue]{hyperref}
\usepackage[capitalise]{cleveref}

\usepackage[capitalise]{cleveref}

\crefname{figure}{Fig.}{Figs.}
\Crefname{figure}{Fig.}{Figs.}
\crefname{subfigure}{Fig.}{Fig.}
\Crefname{subfigure}{Fig.}{Fig.}

\crefrangelabelformat{subfigure}{%
  #3#1#4--#5(\crefstripprefix{#1}{#2}#6}

\begin{document}

\title{\textbf{Finite-Time Curvature-Constrained Vector Field for  Saturation-Free Motion Planning of Nonholonomic Robots}}

\author{
Zhouru Xiao$^*$, Sha Luo$^*$, Yang Lu, H\'{e}ctor Garc\'{ı}a de Marina, Zhenyang Xu, Chaosong Gong,  Yaonan Wang and Weijia Yao


\thanks{Zhouru Xiao, Zhenyang Xu, Chaosong Gong, Yaonan Wang and  Weijia Yao are with the School of Artificial Intelligence and Robotics, Hunan University. China. Sha Luo is with the College of Information Science and Engineering, Hunan Normal University, China. Yang Lu is with the College of Systems Engineering and the State Key Laboratory of Digital Intelligent Modeling and Simulation, National University of Defense Technology, China. H\'{e}ctor Garc\'{ı}a de Marina is with the Department of Computer Engineering, Automation \& Robotics, and the Institute of Mathematics (IMAG), University of Granada, 18012 Granada, Spain. (e-mail: xzr798@hnu.edu.cn; luosha01085@hunnu.edu.cn; luyang18@mail.sdu.edu.cn; hgdemarina@ugr.es; 13884349948@hnu.edu.cn; gongchaosong@hnu.edu.cn; yaonan@hnu.edu.cn; wjyao@hnu.edu.cn). }%
\thanks{*These authors contributed to the work equally and should be regarded as co-first authors.}
}



\maketitle
\begin{abstract}
Accurately guiding a robot to a target configuration is a fundamental task in many engineering systems, which can be particularly challenging for nonholonomic mobile robots. Vector fields (VFs) offer a natural framework for motion planning under nonholonomic constraints. In addition to specifying the desired direction of motion throughout the workspace, a VF can be incorporated directly into the control law, thereby establishing a close connection between motion planning and feedback control. However, most existing VF-based algorithms have difficulty generating reference trajectories that explicitly satisfy curvature constraints. Actuator limits are therefore often handled through input saturation, which, unless incorporated into the controller design, may invalidate the theoretical stability guarantees and degrade closed-loop performance. Moreover, these algorithms typically guarantee only asymptotic convergence and do not provide an explicit bound on the settling time. To overcome these limitations, we propose a generalized motion planning and control framework composed of a finite-time curvature-constrained vector field (FT-C2VF) and a saturation-free control law. The proposed framework steers the robot either to the target configuration within finite time or through the target configuration periodically, depending on the specified motion objective. First, we construct the FT-C2VF using complementary gains to achieve finite-time convergence while ensuring that the curvature of the integral curve is continuous, bounded, and monotonically decreasing in the radial ratio. Second, we develop an almost globally $C^1$-smooth saturation-free control law that tracks the FT-C2VF without requiring Jacobian information, and guarantees that the control inputs remain within the prescribed actuator limits. Third, using dynamical systems theory, we establish the almost-global finite-time stability of the target equilibrium for the resulting closed-loop system. Finally, numerical simulations demonstrate the superior performance of the proposed method over representative VF-based approaches. Outdoor experiments on an Ackermann-steered vehicle further verify the effectiveness and robustness of our algorithm.
\end{abstract}


\section{Introduction}
\IEEEPARstart{S}{everal} robot motion planning tasks, including the autonomous navigation of unmanned ground vehicles \cite{11181562}, the aerial cruising of fixed-wing aircraft \cite{10534850}, and environmental monitoring by underwater vehicles \cite{0Path}, require feasible reference trajectories and corresponding control inputs to guide these robots to target configurations. However, such robots are subject to \textit{nonholonomic constraints}, which typically manifest as an inability to move instantaneously in arbitrary directions \cite{11415403,10113481,10960653}, thereby limiting their maneuverability within the configuration space.

Furthermore, \textit{actuator limits} impose physical bounds on the kinematics of robots. For example, the maximum angular velocity of differential-drive vehicles is constrained by the rotational speed limits of their motors \cite{10938343}, whereas the minimum turning radius of the Ackermann-steered vehicle is determined by the maximum front-wheel steering angle \cite{9451188}. Similarly, the minimum turning radius of fixed-wing aircraft is restricted by the maximum achievable roll angle and flight speed \cite{10980039,6341809}, while the steering capabilities of underwater vehicles are limited by the maximum available yaw moment \cite{9129770,article}. These physical limits establish an upper bound on the trajectory curvature, commonly referred to as the \textit{curvature constraint}. Nonholonomic constraints couple the robot’s position and orientation kinematics, while curvature constraints further restrict the set of admissible trajectories, thereby increasing the complexity of motion-planning design.

Among various motion planning methods, sampling-based and search-based approaches have been extensively studied \cite{844730,doi:10.1177/0278364911406761,10802168,10758212,6094900}. However, these approaches typically impose significant computational overhead, and the generated trajectories often require post-processing to satisfy kinematic constraints. Geometric approaches, including Dubins curves and Reeds-Shepp curves \cite{johnson1974application,reeds1990optimal}, can produce reference trajectories with bounded curvature, but the resulting curvature is discontinuous. Moreover, such methods generally do not integrate control input design for trajectory tracking and are usually implemented in an open-loop manner, rendering them sensitive to disturbances and necessitating frequent replanning. Although some optimization-based methods offer feedback capability and can explicitly incorporate multiple constraints \cite{9105082,9293348}, they are susceptible to local minima and often incur high online computational costs. Taken together, the simultaneous presence of nonholonomic and curvature constraints exposes the limitations of these approaches, making it difficult to achieve a favorable trade-off between closed-loop robustness and computational efficiency.

In contrast, VF-based methods impose minimal computational overhead and provide closed-loop feedback through the construction of reference direction fields, making them particularly effective for handling nonholonomic constraints. An instance of a non-gradient vector field is the dipole-like VF \cite{6160922}, whose normalized form exhibits favorable finite-time convergence properties. Furthermore, this field is naturally suited for tasks requiring simultaneous position and heading planning, as illustrated in Fig.~\ref{fig:intro1}. Building upon this foundation, subsequent studies \cite{6907210,7484276} extended this framework to unicycle models and multi-robot coordination. Additionally, in \cite{he2025novel}, the dipole-like VF was generalized to three-dimensional motion planning, enabling nonholonomic robots to reach a target position with a specified heading from almost all initial states (excluding a singular set of measure zero). Nonetheless, the geometric construction of  the dipole-like VF may produce unnecessarily circuitous reference trajectories, resulting in excessively long paths. Moreover, since the target position is a singular point of the dipole-like VF, the field direction becomes ill-defined near the goal, complicating controller design and implementation.

An alternative approach that circumvents the aforementioned limitations is the guiding vector field (GVF) \cite{7942030}. By superimposing the normal and tangent orthogonal vectors of a manifold, GVFs generate control inputs for robot kinematic models  such as single-integrator and double-integrator models \cite{5504176, doi:10.2514/1.34896}. Although initially developed for path-following tasks, GVFs can be adapted for motion planning by constructing a desired manifold (e.g., a circle \cite{11300826}) tangent to the target configuration and designing appropriate tracking controllers (see Fig.~\ref{fig:intro2}). Furthermore, recent research \cite{9312173} has demonstrated that singularities within GVFs can be completely eliminated through dimensional lifting. However, in contrast to the finite-time convergence properties offered by dipole-like VFs, most GVF formulations guarantee only asymptotic convergence. This implies that exact convergence of the generated trajectories to the desired manifold is achieved only as time approaches infinity \cite{9312173,11127300,9969449,8619780,7942030,10402025}, precluding an explicit settling-time guarantee and potentially limiting the temporal predictability of the motion-planning process.

\begin{figure}[!t]
    \centering
    \subfloat[]{
        \includegraphics[width=0.43\linewidth]{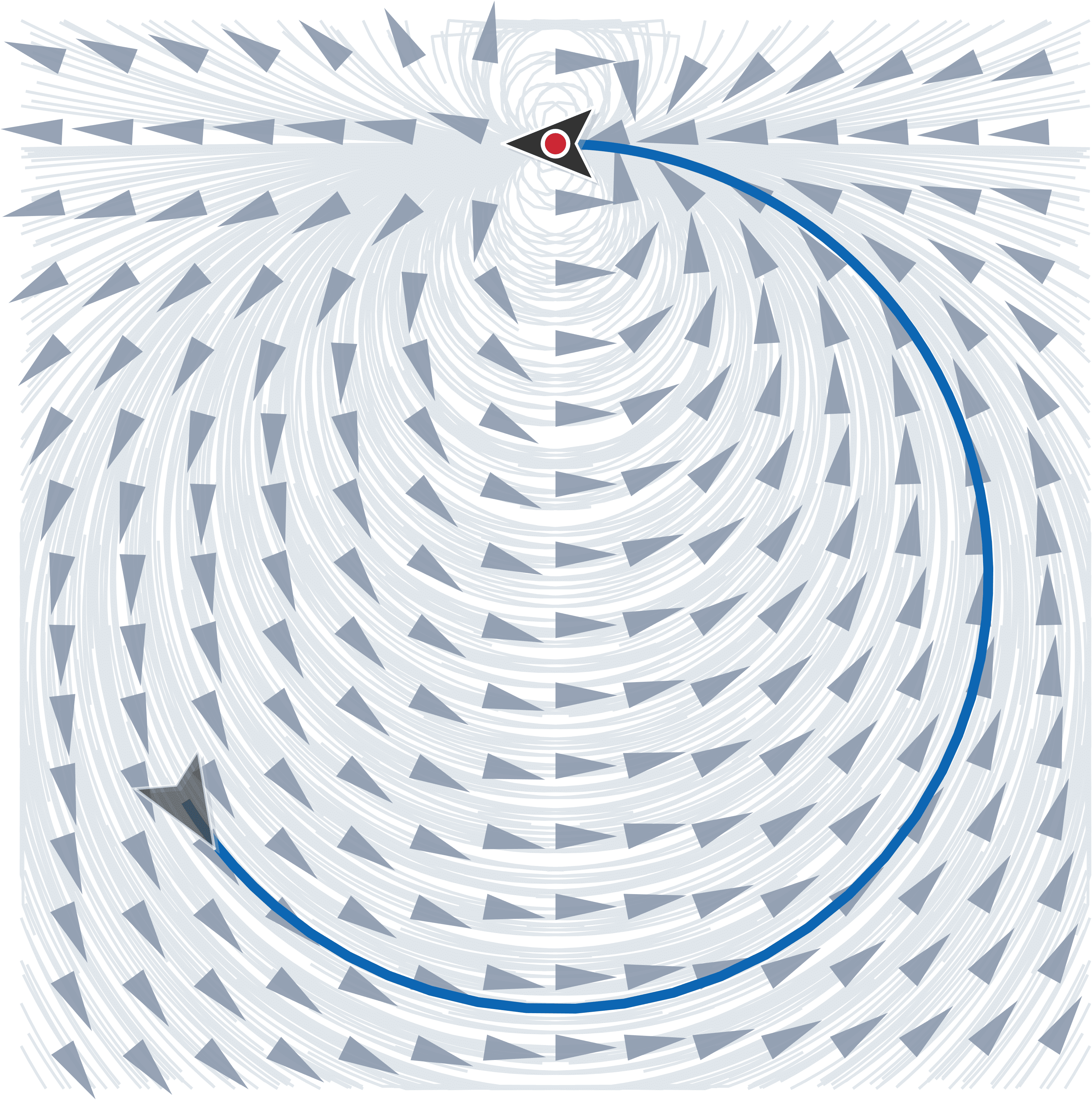}
        \label{fig:intro1}
    }
    \subfloat[]{
        \includegraphics[width=0.43\linewidth]{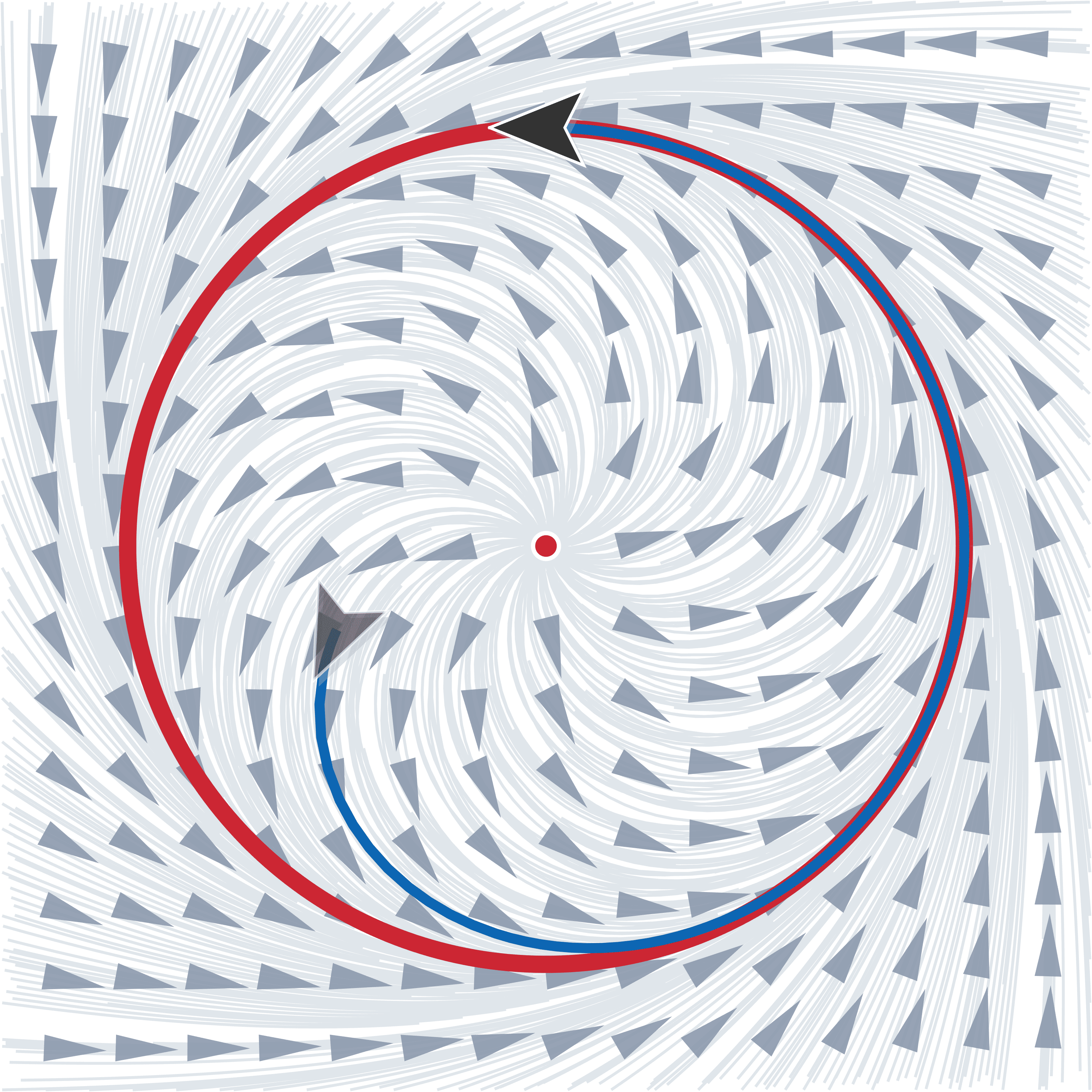}
        \label{fig:intro2}
    }
    \caption{Schematic of robot motion planning based on the dipole-like VF in (a) and the GVF in (b). Red points denote the singular points of the corresponding vector fields.}
    \label{fig:intro}
\end{figure}

Although VF-based methods are advantageous in handling nonholonomic constraints, the design of vector fields with explicit curvature guarantees remains relatively underexplored and presents substantial theoretical and practical challenges. For instance, in \cite{10886648}, a dipole-like VF is partitioned into two regions with separate control laws. However, this partitioning introduces curvature discontinuities at the boundaries, causing discontinuous control efforts. In practice, ensuring both a bounded trajectory curvature and continuous control inputs requires the integral curves of the vector field itself to have continuous and bounded curvature \cite{655130}. While the GVFs developed in \cite{2017Curvature,9476741} satisfy continuity and prescribed curvature bounds, the associated control laws fail to account for actuator limitations, causing tracking deviations that may still violate curvature constraints. Therefore, to guarantee the existence of feasible control laws under curvature constraints, the vector field and the control law must be \textit{co-designed}. However, existing co-design approaches still exhibit important limitations. For example, the method in \cite{11300826} constructs a bounded and continuous vector field and introduces a dynamic-gain feedback term to ensure a monotonic decrease in the orientation error. Nevertheless, it guarantees only asymptotic convergence, which can result in long convergence times in high-precision motion-planning tasks. Therefore, it remains an open problem to design a guiding vector field and an associated control law that enable a nonholonomic robot to reach a desired configuration in finite time while maintaining continuous and bounded trajectory curvature. Addressing this problem entails at least two key challenges.

The first challenge is the construction of a finite-time curvature-constrained vector field (FT-C2VF). Achieving finite-time convergence typically requires nonlinear terms that are non-Lipschitz near the desired manifold. Although these terms accelerate convergence, they may compromise the smoothness of the vector field and render its associated Jacobian singular (i.e., undefined or unbounded) in the vicinity of the manifold. Because the curvature of the vector field depends on the Jacobian matrix, these singularities make it difficult to guarantee continuous and bounded curvature.

The second challenge concerns the design of the control law based on the designed vector field, which is critical for vector field tracking and the maintenance of closed-loop stability. Many existing control laws for vector field tracking incorporate feedforward terms that correspond to directional changes of the vector field \cite{7942030,8619780,11300826,7484276,he2025novel,11127300,10540263}. The introduction of finite-time properties may induce abrupt changes in the control inputs due to Jacobian singularities. For example, frequent jumps in the angular velocity near the desired manifold have been reported in \cite{11246344}. Additionally, curvature constraints impose upper bounds on the angular velocities that depend on the linear velocities. The physical limitations of the actuators further restrict the angular velocity, which limits the ability of the robot to align with the FT-C2VF and potentially causes input saturation. Crucially, existing methods often employ input saturation to enforce curvature constraints \cite{11300826,9312173}, which may compromise theoretical stability guarantees and degrade closed-loop performance within the saturated region. Consequently, the design of a control law that guarantees almost global stability remains a critical challenge.

In this article, we address the two aforementioned challenges, namely the construction of finite-time curvature-constrained vector field (FT-C2VF) and the design of control laws that ensure almost global closed-loop stability. Specifically, based on a circle manifold, we refine the original guiding vector field proposed in \cite{7942030} by introducing a carefully designed pair of complementary gains to regulate its directional changes, guaranteeing that the integral curves of the vector field converge to the desired manifold in finite time, and that their curvature remains bounded and continuous. Building upon this FT-C2VF, we jointly design a saturation-free control law that eliminates the need for feedforward terms, thereby preventing input saturation and achieving almost global closed-loop stability. Furthermore, we propose a unified motion planning framework for heterogeneous nonholonomic robots, guiding them to either converge to the target configuration in finite time or periodically pass through it.

Here, we summarize the major contributions of our article.

First, we construct a novel guiding vector field based on a circle manifold by introducing a pair of complementary gains. By appropriately selecting the radius and the associated parameter, we show that this vector field addresses the first challenge. Namely, the integral curves of the vector field, serving as the reference trajectory for the robot, converge to the circle manifold in finite time and exhibit a continuous, bounded, and monotonically decreasing curvature (see Propositions \ref{propos:001} and \ref{propos:003}). Furthermore, an analytical expression for the curvature is provided to facilitate the co-design of the vector field and the control law (see Section \ref{sec:3}).

Second, in contrast to conventional controllers that rely on the Jacobian information (i.e., the change rate of  the vector field orientation) \cite{9312173,11300826,11246344,11127300,7942030,7484276,9969449}, we propose an almost globally $C^1$-smooth saturation-free control law without using Jacobian information. The proposed control law guarantees bounded trajectory curvature while preventing control input saturation (see Theorem \ref{thm:admissible_control}). Specifically, we first design a state-dependent feedback gain to generate a bounded commanded curvature, which directly dictates the actual curvature of the trajectory. In addition, a parallel-resistance structure is introduced to establish a dynamic upper bound on the linear velocity, allowing the controller to adaptively accommodate variations in trajectory curvature (see Section \ref{sec:4}).

Third, we rigorously establish that the closed-loop system admits an almost globally finite-time stable equilibrium under appropriate parameter design (see Theorem \ref{thm:finit_time}). This property holds for all initial configurations, excluding two sets of measure zero. The heading error between the robot and the vector field first vanishes in finite time, after which the robot reaches the circle manifold and converges exactly, rather than asymptotically, to the target configuration (see Theorem \ref{thm:velocity_convergence}).

Finally, we carry out numerical simulations to validate the theoretical results (see Sections \ref{sec:5.A} and \ref{sec:5.B}). In addition, Monte Carlo comparisons with existing VF-based motion-planning methods \cite{11300826,7484276,7942030,11246344} demonstrate the superior performance of the proposed algorithm in terms of saturation time and the total variation of angular velocity (see Section \ref{sec:5.C}). Moreover, we conduct hardware experiments using an Ackermann-steered vehicle, achieving fast, stable, and high-precision convergence to the target configuration (see Section \ref{sec:6}). 


The remainder of this article is organized as follows. Section~\ref{sec:2} introduces the curvature-constrained nonholonomic kinematic model and formalizes the finite-time generalized motion planning problem as a co-design problem involving a vector field and a control law. In Section~\ref{sec:3}, we present the structure of the FT-C2VF and rigorously prove its finite-time convergence while guaranteeing continuous and bounded curvature. To track the FT-C2VF, Section~\ref{sec:4} details the saturation-free control law and establishes almost-global finite-time convergence for the closed-loop system. In Section~\ref{sec:5}, we validate the theoretical results through numerical simulations and provide comparisons with existing VF-based methods. Section~\ref{sec:6} further validates the framework via hardware experiments on an Ackermann-steered vehicle. Finally, Section~\ref{sec:7} concludes the article and outlines future research. Before proceeding, some standard notations and basic concepts used throughout this article are introduced.

\textit{Notations:}  Given a positive integer $n$, we use boldface letters to denote vectors $\boldsymbol{v} \in \mathbb{R}^n$. The transpose and Euclidean norm of $\boldsymbol{v}$ are denoted by $\boldsymbol{v}^\top$ and $\|\boldsymbol{v}\|$, respectively. The symbol $:=$ denotes “is defined as”.  The distance between a point $\boldsymbol{p}_0 \in \mathbb{R}^n$ and a nonempty set $\mathcal{A} \subseteq \mathbb{R}^n$ is defined as $\mathrm{dist}(\boldsymbol{p}_0, \mathcal{A}) := \inf \{\|\boldsymbol{p}-\boldsymbol{p}_0\|:\boldsymbol{p} \in \mathcal{A}\}$. If $\boldsymbol{f}$ is a differentiable function of time $t$, then its time derivative is denoted by $\dot{\boldsymbol{f}}$, while the derivative of a univariate scalar function $f$ is denoted by $f'$. The notation $f(x)=O(g(x))$
as $x\to x_0$ means that there exist constants $C>0$ and $\delta>0$
such that $|f(x)|\le C|g(x)|$ for all $x$ satisfying
$0<|x-x_0|<\delta$. Let
$\mathbb{S}^1 := \mathbb R/(2\pi\mathbb{Z})$ denote the unit circle viewed as the quotient group. Whenever convenient, elements of $\mathbb{S}^1$ are represented by their principal values in $(-\pi,\pi]$. For a nonzero vector $\boldsymbol{v} = [v_x, v_y]^\top \in \mathbb{R}^2$, its orientation is defined as $\angle \boldsymbol{v} := \mathrm{atan2}(v_y, v_x) \in \mathbb{S}^1$, where $\mathrm{atan2}$ denotes the two-argument arctangent function.

\textit{Basic concepts:} A point $\boldsymbol{\xi}$ satisfying $\boldsymbol{\chi}(\boldsymbol{\xi}) = \mathbf{0}$ is a singular point of the vector field $\boldsymbol{\chi}:\mathbb{R}^n \to \mathbb{R}^n$ \cite[p.219]{lee2015smooth}, which corresponds to an equilibrium point of the corresponding autonomous ordinary differential equation $\dot{\boldsymbol{\xi}} = \boldsymbol{\chi}(\boldsymbol{\xi})$. A trajectory $\boldsymbol{\xi}: \left[0,+\infty \right) \to \mathbb{R}^n$  \textit{asymptotically converges} to a nonempty set $\mathcal{B}$ if for any $\epsilon > 0$, there exists $T > 0$ such that $\mathrm{dist}(\boldsymbol{\xi}(t),\mathcal{B}) < \epsilon$ for all $t > T$ \cite[Def.~4.1]{khalil2002nonlinear}. Given an initial condition $\boldsymbol{\xi}(0) = \boldsymbol{\xi}_0 \in \mathcal{N}$ within a domain $\mathcal{N} \subseteq \mathbb{R}^n$, the trajectory \textit{finite-time converges} to $\mathcal{B}$  if there exists a settling-time function $T:\mathcal{N} \to \left[0,+\infty \right)$ such that $\mathrm{dist}(\boldsymbol{\xi}(t),\mathcal{B}) = 0$ for all $t \ge T(\boldsymbol{\xi}_0)$ \cite[Def.~2.2]{doi:10.1137/S0363012997321358}.

\section{Problem Formulation}\label{sec:2}
\subsection{Kinematic Modeling and Curvature Constraints}
The configuration of a two-dimensional nonholonomic mobile robot is defined as $\boldsymbol{q} = [\boldsymbol{\xi}^\top, \theta]^\top \in \mathcal{C} = \mathbb{R}^2 \times \mathbb{S}^1$, where $\boldsymbol{\xi} = [x, y]^\top$ is the position, $\theta$ is the heading angle, and $\mathcal{C}$ denotes the configuration space. We define the heading and normal vectors as $\boldsymbol{d}(\theta) = [\cos\theta, \sin\theta]^\top$ and $\boldsymbol{n}(\theta) = [-\sin\theta, \cos\theta]^\top$, respectively. Assuming pure rolling without lateral slip, the robot is subject to the Pfaffian constraint $\boldsymbol{n}(\theta)^\top \dot{\boldsymbol{\xi}} = 0$, i.e., the lateral velocity is zero \cite[Ch.~13.3]{lynch2017modern}. Therefore, the nonholonomic kinematics can be expressed as
\begin{equation}\label{eq:kinematic_system}
    \dot{\boldsymbol{q}} = \begin{bmatrix} \boldsymbol{d}(\theta) \\ 0 \end{bmatrix} v + \begin{bmatrix} \mathbf{0}_{2 \times 1} \\ 1 \end{bmatrix} \omega,
\end{equation}
where the control input $\boldsymbol{u} = [v, \omega]^\top$ consists of the forward speed $v$ and the angular velocity $\omega$, with $v \in [v_{-}, v_{+}]$, $|\omega| \le \bar{\omega}$, $0 \le v_{-} \le v_{+}$, and $\bar{\omega} > 0$.

The trajectory curvature is mathematically defined as $\kappa = \omega /v$, where $\kappa \to \infty$ as $v \to 0$ for $\omega\neq 0$\cite[Ch.~11.5]{siciliano2009robotics}. To account for the aforementioned actuator limitations, we prescribe $\bar{\kappa} > 0$ as the maximum allowable curvature. Accordingly, the admissible control input set of the robot is defined as
\begin{equation}\label{eq:control_input_set}
\mathcal{U} = \left\{ [v,\omega]^\top \in \mathbb{R}^2:v \in [v_{-}, v_{+}],\ |\omega| \le \bar{\omega}, |\omega| \le \bar{\kappa}|v| \right\}.
\end{equation}
Notably, the admissible set $\mathcal{U}$ is highly versatile, providing a unified representation of the kinematic requirements of heterogeneous nonholonomic robots in the following two aspects.

\textit{1) Maximum Angular Velocity $\bar{\omega}$}: For differential drive robots, the maximum angular velocity $\bar{\omega}$ depends on the rotational speed difference between the two driving wheels. In this context, $\bar{\kappa}$ serves as a user-defined parameter satisfying $\bar{\kappa} v_{-} < \bar{\omega}$; otherwise,  $\mathcal{U}$ degenerates to $[v_{-}, v_{+}] \times [-\bar{\omega}, \bar{\omega}]$. In contrast, Ackermann-steered vehicles and fixed-wing aircraft are subject to a physical minimum turning radius $\rho$, which imposes a curvature bound $\bar{\kappa} = 1/\rho$. In such cases, $\bar{\omega} \ge \bar{\kappa}v_{-}$ serves as a predefined upper bound for the angular velocity; otherwise, the set $\mathcal{U}$ similarly degenerates to $[v_{-}, v_{+}] \times [-\bar{\omega}, \bar{\omega}]$.

\textit{2) Minimum Forward Speed  $v_-$:} Ground robots, such as differential drive robots and Ackermann-steered vehicles, can come to a complete stop, rendering $v_{-} = 0$ permissible. Conversely, fixed-wing aircraft must maintain forward flight to generate sufficient lift, requiring $v_{-} > 0$.

Thus, various nonholonomic robots with curvature constraints can be modeled by specifying $v_{-}, v_{+}, \bar{\omega}$ and $\bar{\kappa}$.

\subsection{Finite-Time Generalized Motion Planning}
Most VF-based motion-planning methods, such as those in \cite{11300826,10540263}, only guarantee asymptotic convergence to the target configuration without providing an explicit upper bound on the convergence time. To address heterogeneous kinematic constraints uniformly and achieve finite-time convergence, we formally define the finite-time generalized motion planning (FT-GMP) problem as follows.

\begin{problem}[FT-GMP problem]\label{pro:001}
    Given an initial configuration $\boldsymbol{q}_0 \in \mathcal{C}$ and a target configuration $\boldsymbol{q}_d \in \mathcal{C}$ with a desired terminal speed $v_d \in [v_-,v_+]$, the objective is to design control inputs $\boldsymbol{u} = [v, \omega]^\top \in \mathcal{U}$ such that the trajectory $\boldsymbol{q}(t)$ of system \eqref{eq:kinematic_system} satisfies the following two conditions.
    \begin{itemize}
    \item[1)] There exist a settling-time function $T_a: \mathcal{C} \to \left[0,+\infty \right)$ and a finite constant $\Delta T \ge 0$ such that for all $t \ge T_a(\boldsymbol{q}_0)$, there exists at least one time instant $t^* \in [t, t+\Delta T]$ satisfying $\boldsymbol{q}(t^*) = \boldsymbol{q}_d$.
    \item[2)] The speed matches the terminal speed whenever the trajectory $\boldsymbol{q}(t)$ reaches the target configuration $\boldsymbol{q}_d$, i.e., $v(t) = v_d$ for all $t \ge T_a(\boldsymbol{q}_0)$ such that $\boldsymbol{q}(t) = \boldsymbol{q}_d$.
    \end{itemize}
\end{problem}

This formulation unifies two representative robotic tasks by adjusting the terminal speed $v_d$ as follows.

\textit{1) Stationary Operation ($v_{d} = 0$):} The curvature constraint (i.e., $|\omega| \le \bar{\kappa} |v|$) strictly enforces $\omega = 0$ whenever $v = v_d = 0$. Consequently, it follows that  $\Delta T = 0$, $\boldsymbol{q}(t) \equiv \boldsymbol{q}_d$, and $v(t) \equiv 0$ for all $t \ge T_a(\boldsymbol{q}_0)$, meaning the robot remains stationary at the target configuration. 

\textit{2) Periodic Monitoring ($v_{d} > 0$):} The robot passes through $\boldsymbol{q}_d$ at speed $v_{d}$ and revisits this configuration periodically, with the recurrence period upper-bounded by $\Delta T$.

\subsection{Problem Reformulation Based on Vector Field}

To characterize curvature constraints on the integral curves of a vector field, we first introduce the following notion.
\begin{definition}\label{def:curvature_vf}
The curvature of the vector field at a nonsingular point is defined as the curvature of the integral curve of the vector field passing through that point. 
\end{definition}


\begin{lemma}\label{lem:curvature_vf}
Let $\boldsymbol{\chi}: \mathbb{R}^2 \to \mathbb{R}^2$ be continuously differentiable on $\mathbb{R}^2\setminus\mathcal{W}$, where $\mathcal{W}:=\{\boldsymbol{\xi}\in\mathbb{R}^2:\boldsymbol{\chi}(\boldsymbol{\xi})=\boldsymbol{0}\}$ is the singular set. Then, for any $\boldsymbol{\xi}\in\mathbb{R}^2\setminus\mathcal{W}$, the signed curvature of
$\boldsymbol{\chi}$ at $\boldsymbol{\xi}$ is given by
\begin{equation}\label{eq:curvature_vf}
    \kappa_{\chi}(\boldsymbol{\xi})
    =
    \frac{
    \left(\boldsymbol{E}\boldsymbol{\chi}(\boldsymbol{\xi})\right)^\top
    \boldsymbol{J}_{\chi}(\boldsymbol{\xi})
    \boldsymbol{\chi}(\boldsymbol{\xi})
    }{
    \|\boldsymbol{\chi}(\boldsymbol{\xi})\|^3
    },
\end{equation}
where $\boldsymbol{J}_{\chi}(\boldsymbol{\xi})
=\partial\boldsymbol{\chi}(\boldsymbol{\xi})/\partial\boldsymbol{\xi}$ and $\boldsymbol{E}= \left[\begin{smallmatrix} 0 & -1 \\ 1 & 0 \end{smallmatrix}\right]$.
\end{lemma}

\begin{IEEEproof}
  See Appendix \ref{app:curvature_vf}.
\end{IEEEproof}

For a given target configuration $\boldsymbol{q}_d = [\boldsymbol{\xi}_d^\top, \theta_d]^\top$, a circle manifold tangent to this configuration is constructed as
\begin{equation}\label{eq:M_r}
         \mathcal{M}_r = \left\{ \boldsymbol{\xi} \in \mathbb{R}^2 :\|\boldsymbol{\xi} - \boldsymbol{c} \|^2 = r^2,\boldsymbol{c} = \boldsymbol{\xi}_d \pm r \boldsymbol{n}(\theta_d) \right\},
\end{equation}
where $r$ is the radius. Assume that a vector field is designed to converge to $\mathcal{M}_r$ in finite time while ensuring that the curvature does not exceed $\bar{\kappa}$. If a tracking controller enables the trajectory of the robot to exactly follow the integral curve of this vector field, the trajectory naturally satisfies the curvature constraint, and the robot converges to the target configuration along $\mathcal{M}_r$. Such a vector field is defined as a finite-time curvature-constrained vector field (FT-C2VF), which is formally specified below.

\begin{definition}[FT-C2VF]\label{def1}
    Let $\mathcal{W} = \{\boldsymbol{\xi} \in \mathbb{R}^2: \boldsymbol{\chi}(\boldsymbol{\xi}) = \boldsymbol{0}\}$ denote the singular set of a vector field $\boldsymbol{\chi}: \mathbb{R}^2 \to \mathbb{R}^2$. $\boldsymbol{\chi}$ is defined as an FT-C2VF with respect to $\mathcal{M}_r$ if the differential equation $\dot{\boldsymbol{\xi}} = \boldsymbol{\chi}(\boldsymbol{\xi})$ fulfills the following properties.
        \begin{itemize}
        \item[1)] \textit{Curvature Constraint:} For all points $\boldsymbol{\xi} \in \mathbb{R}^2 \setminus \mathcal{W}$, the curvature of the vector field satisfies $|\kappa_\chi(\boldsymbol{\xi})| \le \bar{\kappa}$.
        \item[2)] \textit{Finite-time Convergence:} For all initial points $\boldsymbol{\xi}_0 \in \mathbb{R}^2 \setminus \mathcal{W}$, there exists a settling-time function $T_b: \mathbb{R}^2 \to [0, +\infty)$ such that the distance between the integral curve $\boldsymbol{\xi}(t)$ and the manifold $\mathcal{M}_r$ becomes zero in finite time; that is, $\mathrm{dist}(\boldsymbol{\xi}(t), \mathcal{M}_r) = 0$ for all $t \ge T_b(\boldsymbol{\xi}_0)$.
    \end{itemize}
\end{definition}

We now shift our objective toward the simultaneous design of the FT-C2VF and a tracking control law. Consequently, Problem \ref{pro:001} is recast as the co-design of the FT-C2VF and the control law (CL) as described below.

\begin{problem}[VF-CL Co-Design problem]\label{pro:002}
    Given an initial configuration $\boldsymbol{q}_0 = [\boldsymbol{\xi}_0^\top,\theta_0]^\top \in \mathcal{C}$ and a target configuration $\boldsymbol{q}_d = [\boldsymbol{\xi}_d^\top,\theta_d]^\top \in \mathcal{C}$ with a desired terminal speed $v_d \in [v_-,v_+]$, the VF-CL Co-Design problem is to jointly design an FT-C2VF $\boldsymbol{\chi}$ with $\mathcal{M}_r$ satisfying \eqref{eq:M_r} and $\theta_d = \angle \boldsymbol{\chi}(\boldsymbol{\xi}_d)$, and a control law $\boldsymbol{u} = [v,\omega]^\top \in \mathcal{U}$ for the system \eqref{eq:kinematic_system} such that the following two conditions are satisfied.
     \begin{itemize}
     \item[C1] \textit{(Heading Alignment):} There exists a settling-time function $T_c:\mathcal{C} \to \left[0,+\infty \right)$ such that the heading of the robot aligns with the direction of the vector field; that is, $\theta(t) = \angle\boldsymbol{\chi}(\boldsymbol{\xi}(t))$ for all $t \ge T_c(\boldsymbol{q}_0)$. 
     \item[C2] \textit{(Speed Matching):} There exists a settling-time function $T_d:\mathcal{C} \to \left[0,+\infty \right)$ such that the speed matches the terminal speed whenever the trajectory $\boldsymbol{q}(t)$ reaches the target configuration $\boldsymbol{q}_d$ after $T_d(\boldsymbol{q}_0)$, and maintains a positive speed otherwise; that is, $v(t) = v_d$ for all $t \ge T_d(\boldsymbol{q}_0)$ when $\boldsymbol{q}(t) = \boldsymbol{q}_d$, and $v(t) > 0$ for all $t \ge 0$ whenever $\boldsymbol{q}(t) \neq \boldsymbol{q}_d$.
    \end{itemize}
\end{problem}

\begin{remark}
Problem \ref{pro:002} serves as a constructive sufficient condition for solving Problem \ref{pro:001}. This reformulation decouples the complex FT-GMP problem into sequential and tractable subtasks. C1 ensures that the heading of the robot aligns with the FT-C2VF in finite time. Upon alignment, the trajectory of the robot is governed by the vector field, which naturally guides the robot to converge to the circle manifold $\mathcal{M}_r$. Meanwhile, C2 guarantees a positive speed before reaching the target configuration to prevent premature stalling. Furthermore, because $\mathcal{M}_r$ is a closed manifold, this continuous motion with a positive speed ensures the robot inevitably reaches $\boldsymbol{q}_d$, remaining stationary if $v_d = 0$, or repeatedly passing through $\boldsymbol{q}_d$ if $v_d > 0$. By default, it is assumed that $v_d = v_-$; otherwise, a new lower bound of the linear velocity can simply be defined as $v'_- = v_d$. \hfill $\blacktriangleleft$
\end{remark}

\section{Design of Finite-Time Curvature-Constrained Vector Field}\label{sec:3}
In this section, we construct a guiding vector field based on the circle manifold in \eqref{eq:M_r} and show that it is an FT-C2VF by appropriately selecting its parameters and radius.

Without loss of generality, the coordinate frame can be translated such that the center of $\mathcal{M}_r$ coincides with the origin. Subsequently, the manifold can be expressed as $\mathcal{M}_r = \{\boldsymbol{\xi} \in \mathbb{R}^2: \phi(\boldsymbol{\xi}) = 0\}$, where the level set function $\phi:\mathbb{R}^2\to\mathbb{R}$ is defined as $\phi(\boldsymbol{\xi}) = \|\boldsymbol{\xi}\|^2 - r^2$. The guiding vector field $\boldsymbol{\chi}:\mathbb{R}^2\to\mathbb{R}^2$ is designed as follows:
\begin{equation}\label{eq:FT-C2VF}
\boldsymbol{\chi}(\boldsymbol{\xi})=\psi_{\tau}(\eta(\boldsymbol{\xi}))\boldsymbol{E}\nabla\phi(\boldsymbol{\xi})+\psi_{\nu}(\eta(\boldsymbol{\xi}))\nabla\phi(\boldsymbol{\xi}),
\end{equation}
where $\boldsymbol{E}\in SO(2)$ is the $90^\circ$ rotation matrix $\left[\begin{smallmatrix} 0 & -1 \\ 1 & 0 \end{smallmatrix}\right]$, and $\nabla\phi \in \mathbb{R}^{2}$ is the gradient of $\phi$ with respect to $\boldsymbol{\xi}$. The \emph{radial ratio function} $\eta:\mathbb{R}^2 \to \left[0,+\infty\right)$ is defined as $\eta(\boldsymbol{\xi})=\|\boldsymbol{\xi}\|/r$, which ensures that $\mathrm{dist}(\boldsymbol{\xi}(t),\mathcal{M}_r) = 0$ when $\eta(\boldsymbol{\xi}) = 1$. Furthermore, $\psi_{\tau}:[0,+\infty)\to[0,1]$ is a non-negative \emph{tangential gain function} satisfying $\psi_{\tau}(1)>0$, and $\psi_{\nu}:[0,+\infty)\to[-1,1]$ is a decreasing \emph{normal gain function} satisfying $\psi_{\nu}(1)=0$. For all $\eta \in \left[0,+\infty\right)$, these two functions are designed to satisfy
\begin{equation}\label{eq:equal_psi}
     \psi_{\tau}^2(\eta)+\psi_{\nu}^2(\eta)=1. 
\end{equation}
It follows that 
\begin{equation}\label{eq:daxiao}
    \|\boldsymbol{\chi}(\boldsymbol{\xi})\| =\sqrt{\psi^2_\tau+\psi^2_\nu} \|\nabla \phi(\boldsymbol{\xi})\|= 2\|\boldsymbol{\xi}\|,
\end{equation}
 which indicates that $\boldsymbol{\xi}=\boldsymbol{0}$ is the unique singular point; that is, the singular set in this translated frame is $\mathcal{W}=\{\boldsymbol{0}\}$. Similar to classical guiding vector fields \cite[eq.~(9)]{7942030}, \cite[eq.~(2)]{9312173}, \cite[eq.~(13)]{7172278}, the first term of the proposed vector field is tangent to the desired manifold $\mathcal{M}_r$, thereby enabling the robot to move along it. The second term is perpendicular to the first, driving the robot toward the desired manifold. Consequently, the combined vector field simultaneously forces the robot to move toward and along the desired manifold. Unlike existing methods, the gain functions in our approach are specifically designed to yield a tractable analytical expression for the curvature of the vector field, as shown in Lemma \ref{lem:001}.

\begin{lemma}\label{lem:001}
     For any point $\boldsymbol{\xi}\in\mathbb{R}^2 \setminus \mathcal{W}$,
    let $\eta=\eta(\boldsymbol{\xi})$. The signed curvature of
    $\boldsymbol{\chi}$ depends only on the radial ratio $\eta$ and is
    therefore well defined as a scalar function
    $\kappa_\chi:(0,+\infty)\to\mathbb R$. Moreover, it is given by
    \begin{equation}\label{eq:lemma1}
    \kappa_\chi(\eta) = \frac{1}{r} \left( \frac{\psi_{\tau}(\eta)}{\eta} + \psi_{\tau}'(\eta) \right).
    \end{equation}
\end{lemma}

\begin{IEEEproof}
    See Appendix \ref{app:001}.
\end{IEEEproof}

While the geometric curvature $\kappa_\chi$ naturally describes the bending of the integral curves along the tangential direction (e.g., $\kappa_\chi > 0$ and $\kappa_\chi < 0$ indicate left and right turns, respectively), a comprehensive geometric characterization of the vector field also requires quantifying its spatial variation in the transversal direction. To this end, we introduce the \textit{orthogonal curvature} $\kappa_\perp$, which is defined as the curvature of the new vector field $\boldsymbol{E\chi}$ obtained by rotating each vector by $90^\circ$ of the original vector field $\boldsymbol{\chi}$, and it characterizes the spatial variation of the original vector field $\boldsymbol{\chi}$ along the “orthogonal direction”. Benefiting from the structural consistency between the tangential and normal components in the proposed FT-C2VF, $\kappa_\perp$ admits a dual formulation with the same structure as that of $\kappa_\chi$. Namely, they satisfy a fundamental geometric identity, as established in the following proposition.

\begin{proposition}\label{cor:001}
The orthogonal curvature $\kappa_\perp:[0, +\infty) \to \mathbb{R}$ is given by
\begin{equation}\label{eq:kappa_prep}
\kappa_\perp(\eta) = \frac{1}{r} \left( \frac{\psi_{\nu}(\eta)}{\eta} + \psi_{\nu}'(\eta) \right).
\end{equation}
Moreover, the following identity holds for all $\eta \in (0,+\infty)$:
\begin{equation}\label{eq:identity}
    \kappa_\chi\psi_\tau+\kappa_\perp\psi_\nu = \frac{1}{r\eta}.
\end{equation}
\end{proposition}
\begin{IEEEproof}
    See Appendix \ref{app:Corollary1}.
\end{IEEEproof}

Lemma \ref{lem:001} demonstrates that the curvature of $\boldsymbol{\chi}$ depends exclusively on the tangential gain function $\psi_{\tau}$ and the radius $r$. To bound the curvature, we design the piecewise continuous tangential gain function as follows:
\begin{subequations}\label{eq:kt_def}
\begin{numcases}{\psi_{\tau}(\eta) =}1 - (1 - \eta)^{1+\gamma}, & $\eta \in [0,1]$ \label{eq:kt_def_a} \\ 
1 - (1 - \eta^{-1})^{1+\gamma}, & $\eta \in (1,+\infty)$ \label{eq:kt_def_b}
\end{numcases}
\end{subequations}
where $\gamma > 0$ is a constant gain. In this way, the normal gain function is uniquely determined as
\begin{subequations}\label{eq:kn_def}
\begin{numcases}{\psi_{\nu}(\eta) =}\sqrt{1 - \psi_{\tau}^2(\eta)}, & $\eta \in [0,1]$ \label{eq:kn_def_a} \\
-\sqrt{1 - \psi_{\tau}^2(\eta)}. & $\eta \in (1,+\infty)$ \label{eq:kn_def_b}
\end{numcases}
\end{subequations}

Note that the curvature of the vector field is undefined at singular points. Hence, the curvature constraint is considered only on the domain $\eta \in (0,+\infty)$. In this domain, the tangential gain function $\psi_{\tau}$ is continuously differentiable\footnote{Given an integer $k$, a vector field (or function) $\boldsymbol{\chi}: \mathbb{R}^n \to \mathbb{R}^n$ is said to be of class $C^k$ if all its component functions have continuous partial derivatives of order up to $k$ \cite[p.~644]{lee2015smooth}. In this article, we denote this by $\boldsymbol{\chi} \in C^k$.} (i.e., $\psi_{\tau} \in C^1$), thereby guaranteeing the continuity of the curvature $\kappa_\chi$ given in \eqref{eq:lemma1} (see Fig.~\ref{fig:kappa}).  Further details on the curvature of $\boldsymbol{\chi}$ are given in the following proposition.

\begin{proposition}[Curvature constraint]\label{propos:001}
    The curvature $\kappa_\chi:(0, +\infty) \to \mathbb{R}$ is a continuous and strictly decreasing function for $\eta \in (0,+\infty)$, which can be analytically expressed as
    \begin{equation}\label{eq:kappa_piecewise}
    \kappa_\chi(\eta)=
    \begin{cases}
    \displaystyle\frac{1}{r}\left[\frac{1-(1-\eta)^\gamma}{\eta}+(2+\gamma)(1-\eta)^\gamma\right], \eta \in (0,1], \\[2.5ex]
    \displaystyle \frac{1}{r\eta} \left[ 1 - \left(1 + \frac{\gamma}{\eta} \right) \left(1 - \frac{1}{\eta} \right)^\gamma \right], \eta \in (1, +\infty).
    \end{cases}
    \end{equation}
 Moreover, it satisfies
    \begin{equation}\label{eq:kappa_max_min}
    \lim_{\eta\to +\infty}\kappa_\chi(\eta) = 0
    < \kappa_\chi(\eta)
    < \lim_{\eta\to 0^+}\kappa_\chi(\eta)
    = \frac{2(1+\gamma)}{r}.
    \end{equation}
\end{proposition}

\begin{IEEEproof}
    See Appendix \ref{app:002}.
\end{IEEEproof}

To study more properties of the vector field, we investigate the trajectory of the following autonomous ordinary differential equation:
\begin{equation}\label{eq:ODE}
    \dot{\boldsymbol{\xi}}=\boldsymbol{\chi}(\boldsymbol{\xi})
\end{equation}
To ensure the integral curves of $\boldsymbol{\chi}$ reach and traverse $\mathcal{M}_r$ in finite time, we set $\gamma\in(0,1)$. This guarantees the existence and uniqueness of the forward solution to \eqref{eq:ODE}, as detailed in Appendix~\ref{app:003}. Given $\gamma \in (0,1)$, an explicit upper bound on the convergence time is provided as follows.

\begin{proposition}[Finite-time convergence]\label{propos:003}
     For any initial point $\boldsymbol{\xi}_0 \in \mathbb{R}^2 \setminus \mathcal{W}$, there exists a settling-time function $T_\chi: \mathbb{R}^2 \to [0,+\infty)$ such that the trajectory $\boldsymbol{\xi}(t)$ converges to the circle manifold $\mathcal{M}_r$ at $t = T_\chi(\boldsymbol{\xi}_0)$. Moreover, we have
    \begin{equation}\label{eq:T_chi}
    T_\chi(\boldsymbol{\xi}_0) \le
    \begin{cases}
    \dfrac{(1-\eta(\boldsymbol{\xi}_0))^{\frac{1-\gamma}{2}}}{\eta(\boldsymbol{\xi}_0)(1-\gamma)}, & \eta(\boldsymbol{\xi}_0) \in (0,1] \\
    \dfrac{(\eta(\boldsymbol{\xi}_0)-1)^{\frac{1-\gamma}{2}}}{1-\gamma}. & \eta(\boldsymbol{\xi}_0) \in (1,+\infty)
    \end{cases}
    \end{equation}
\end{proposition}

\begin{IEEEproof}
    See Appendix \ref{app:004}.
\end{IEEEproof}

Notably, finite-time convergence does not necessarily imply a shorter entrance time into every prescribed nonzero error neighborhood. Its main advantage is a finite exact settling time to the target set. Equivalently, for a fixed initial error, the time to enter an $\epsilon$-neighborhood of the target admits an upper bound independent of $\epsilon$ as $\epsilon\to0^+$. The following scalar example illustrates this point.

\begin{example}\label{example:001}
We compare two first-order scalar autonomous systems describing the evolution of a nonnegative error magnitude $e(t)\in\mathbb{R}_{\ge0}$: the exponentially convergent system $\dot e=-\lambda e$ and the finite-time convergent system $\dot e=-\lambda e^u$, where $\lambda>0$ and $0<u<1$. The target is the origin, where $e(t)=0$ represents exact convergence. For a tolerance $\epsilon>0$, define the error neighborhood $\mathcal{B}_{\epsilon}:=\{e\in\mathbb{R}_{\ge0}:e\le \epsilon\}$. Given $e(0)=e_0>\epsilon$, define the entrance time into $\mathcal{B}_{\epsilon}$ as $T_{\epsilon}:=\inf\{t\ge0:e(s)\le\epsilon, \ \forall s\ge t\}$. For the exponentially convergent system $\dot e=-\lambda e$ with $\lambda>0$, the solution is $e(t)=e_0e^{-\lambda t}$, yielding $T_{\epsilon}^{\rm exp}=\frac{1}{\lambda}\ln\frac{e_0}{\epsilon}$. Thus, $T_{\epsilon}^{\rm exp}\to+\infty$ as $\epsilon\to0^+$, meaning exact convergence occurs only asymptotically. Conversely, for the finite-time system $\dot e=-\lambda e^u$ with $\lambda>0$ and $0<u<1$, the solution exactly reaches the origin at $T_0^{\rm ft}=\frac{e_0^{1-u}}{\lambda(1-u)}$. For any $0<\epsilon<e_0$, the entrance time into $\mathcal{B}_{\epsilon}$ is $T_{\epsilon}^{\rm ft}=\frac{e_0^{1-u}-\epsilon^{1-u}}{\lambda(1-u)}\le T_0^{\rm ft}$, uniformly bounded independently of $\epsilon$. For example, let $\lambda=2$, $e_0=10$, and $u=0.5$. For $\epsilon=0.1$, $T_{0.1}^{\rm exp}\approx2.30$ and $T_{0.1}^{\rm ft}\approx2.85$, indicating the exponential system enters this neighborhood earlier. However, for a tighter tolerance $\epsilon=10^{-4}$, $T_{10^{-4}}^{\rm exp}\approx5.76$ and $T_{10^{-4}}^{\rm ft}\approx3.15$. Moreover, the finite-time system exactly reaches the origin at $T_0^{\rm ft}\approx3.16$, whereas the exponential system never reaches it in finite time.
    \hfill $\blacktriangleleft$
\end{example}

The preceding example can be extended to practical robotic tasks. By defining the target configuration set as $\mathcal{Q}_\epsilon = \{\boldsymbol{q}\in\mathcal{C}:\|\boldsymbol{q} - \boldsymbol{q}_d\| \le \epsilon \}$, the convergence time of the robot to $\mathcal{Q}_\epsilon$ is guaranteed to be upper bounded by some constant value, regardless of how small the tolerance $\epsilon$ is chosen. Building upon Propositions \ref{propos:001} and \ref{propos:003}, we establish the following sufficient condition for $\boldsymbol{\chi}$ to be an FT-C2VF.

\begin{theorem}\label{thm:FT-C2VF}
     The vector field $\boldsymbol{\chi}$ in \eqref{eq:FT-C2VF} is an FT-C2VF (i.e., it satisfies both requirements in Definition \ref{def1} if  $\gamma \in (0,1)$ and $r\ge \frac{2(1+\gamma)}{\bar{\kappa}}$.
\end{theorem}

\begin{IEEEproof}
    Proposition \ref{propos:001} ensures that any point $\boldsymbol{\xi}\in\mathbb{R}^2 \setminus \mathcal{W}$ satisfies the curvature constraint $\kappa_\chi(\eta(\boldsymbol{\xi}))<\frac{2(1+\gamma)}{r}\le \bar{\kappa}$. Furthermore, by setting $T_b = T_\chi$ in Definition \ref{def1}, finite-time convergence is guaranteed according to Proposition \ref{propos:003}.
\end{IEEEproof}

\section{Design of Saturation-Free Control Law and Convergence Results}\label{sec:4}
In this section, we develop a saturation-free control law for the linear and angular velocities based on the FT-C2VF to solve Problem \ref{pro:002}. To this end, we introduce the closed-loop autonomous system under this control strategy and derive related convergence results.

For a robot with configuration $\boldsymbol{q} = [\boldsymbol{\xi}^\top, \theta]^\top$, where $\boldsymbol{\xi} \notin \mathcal{W}$, the heading error with respect to the FT-C2VF is defined as $\theta_e = \theta - \angle\boldsymbol{\chi}(\boldsymbol{\xi})$. This error satisfies $\cos\theta_e = \boldsymbol{h}^\top\hat{\boldsymbol{\chi}}$ and $\sin\theta_e = \boldsymbol{h}^\top \boldsymbol{E} \hat{\boldsymbol{\chi}}$, where $\hat{\boldsymbol{\chi}} = \boldsymbol{\chi}/\|\boldsymbol{\chi}\|$ and $\boldsymbol{h} = [\cos\theta,\sin\theta]^\top$. We define a configuration error function $\Xi : \mathcal{C} \to \mathbb{R}_{\ge 0}$ as follows:
\begin{equation}\label{eq:Xi}
\Xi(\boldsymbol{q}) = \|\boldsymbol{\xi} - \boldsymbol{\xi}_d\|^2 + \mu \big( 1 - \cos\theta_e \big),
\end{equation}
where $\mu > 0$ is a constant gain balancing the position and heading weights. Clearly, $\Xi(\boldsymbol{q}) = 0$ if and only if $\boldsymbol{q} = \boldsymbol{q}_d$; i.e., the robot configuration $\boldsymbol{q}$ coincides with the desired configuration $\boldsymbol{q}_d$.

To ensure that the trajectory curvature of the robot under the control input is bounded by $\bar{\kappa}$, the \textit{commanded curvature} is formulated as follows:
\begin{equation}\label{command_kappa}
\kappa_c = \kappa_\chi - k_\omega\lceil \sin\theta_e\rfloor^\alpha,
\end{equation}
where $\kappa_\chi$ is the curvature of the FT-C2VF in \eqref{eq:kappa_piecewise}, and $k_\omega= \bar{\kappa} - \kappa_\chi > 0$ is a state-dependent gain. The operator $\lceil x \rfloor^\alpha := |x|^\alpha \operatorname{sgn}(x)$ denotes the generalized fractional-power sign function with $\alpha >0$, where $\operatorname{sgn}(\cdot)$ is the signum function (i.e., $\operatorname{sgn}(x)=1$ for $x>0$, $0$ for $x=0$, and $-1$ for $x<0$). 

Then, we have the following inequalities
\begin{align}\label{eq:kappa_c_max}
|\kappa_c| &\le \kappa_\chi + \left|k_\omega\lceil \sin\theta_e\rfloor^\alpha\right|  \notag\\
&\le \kappa_\chi + (\bar{\kappa} - \kappa_\chi)|\sin\theta_e|^\alpha\le \bar{\kappa}.
\end{align}
In \eqref{command_kappa}, the first term ensures trajectory alignment with the integral curve of the vector field when $\theta_e = 0$, causing the robot to move along it, while the second term provides feedback correction for heading errors. Once the commanded curvature $\kappa_c$ and the speed $v$ are determined, the angular velocity is given by $\omega = v \kappa_c$. This design decouples the geometric and temporal kinematics: the angular velocity is adaptively scaled by the speed $v$, ensuring consistent trajectory generation across varying speed profiles.

\begin{figure}[!th]
\centering
\includegraphics[width=2in]{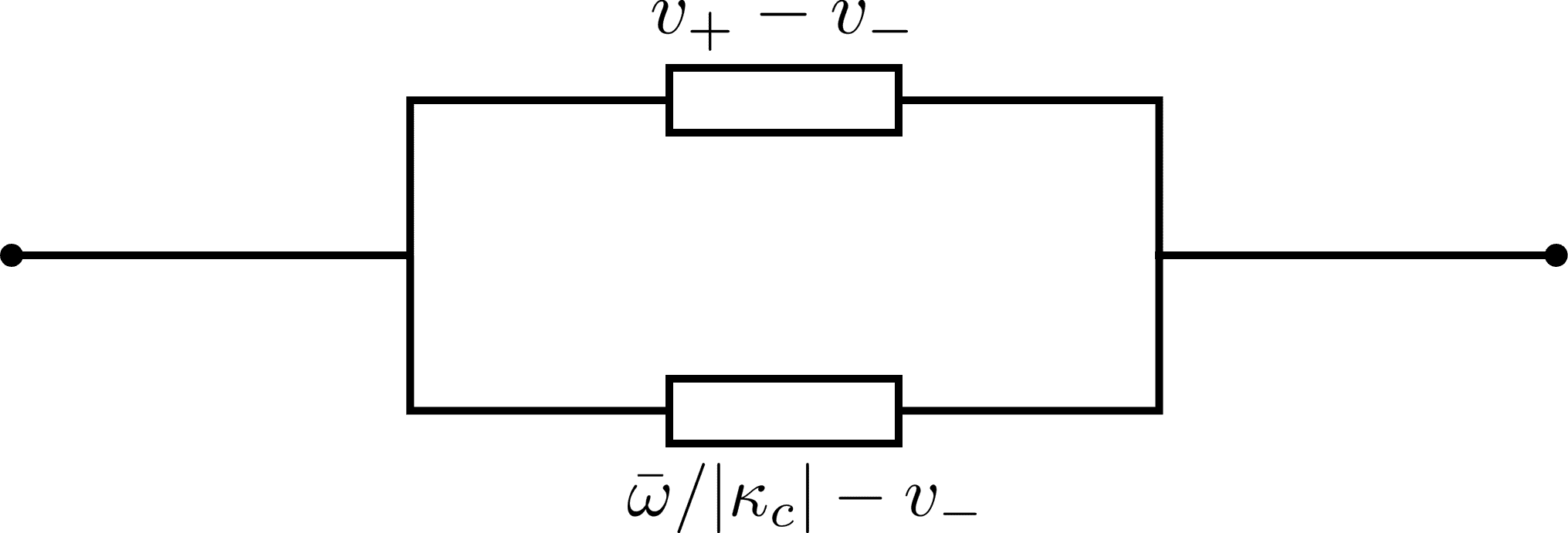}
\caption{Parallel-resistance-like structure for the adaptive speed upper bound.}
\label{fig:parallel}
\end{figure}

Given $\kappa_c$, the speed must satisfy $v \le \bar{\omega}/|\kappa_c|$ to avoid violation of the angular velocity constraint $|\omega| \le \bar{\omega}$ when $|\kappa_c|$ is large. Conversely, when $|\kappa_c|$ is small, the speed is primarily restricted by its nominal upper bound $v_+$. If the kinematic velocity bounds satisfy $v_+ > v_-$, an adaptive velocity upper bound is designed as
\begin{equation}\label{eq:cmv}
v_{r} = v_{-} + \left( \frac{1}{v_{+} - v_{-}} + \frac{1}{\bar{\omega}/|\kappa_c| - v_{-}} \right)^{-1}.
\end{equation}
Under the premise that $\bar{\kappa} < \bar{\omega}/v_-$, this adaptive bound gracefully accounts for variations in trajectory curvature, ensuring deceleration during sharp turns and higher speeds along gentle curves. This structure can be interpreted by analogy with parallel resistance in circuit theory~\cite[Ch.~2.6]{alexander2021fundamentals}, as illustrated in \cref{fig:parallel}. Specifically, the nominal maximum velocity $v_+$ and the maximum feasible turning velocity $\bar{\omega}/|\kappa_c|$ jointly dictate the upper bound, ensuring that $v_r$ is less than the minimum of $v_+$ and $\bar{\omega}/|\kappa_c|$. To achieve the motion planning objectives, we synthesize the linear and angular velocity control laws as follows:
\begin{subequations}\label{eq:co_design_law}
\begin{align}
    v &= v_{-} + \big( v_{r} - v_{-} \big) \left( 1 - e^{-k_v \Xi^\beta} \right), \label{eq:vlaw} \\
   \omega &= v\kappa_c, \label{eq:omegalaw}
\end{align}
\end{subequations}
where $k_v,\beta > 0$ are constant gains. The design of the speed controller in \eqref{eq:co_design_law} ensures that if the robot is far from the target configuration (i.e., $\Xi \gg 0$), $v \to v_r$, while upon reaching the target configuration (i.e., $\Xi = 0$), $v = v_{-}$.


\begin{theorem}[Saturation-free admissible control]\label{thm:admissible_control}
The control input $\boldsymbol{u} = [v,\omega]^\top$ generated by \eqref{eq:co_design_law} is saturation-free; that is, the forward speed $v$ and angular velocity $\omega$ remain strictly below their respective prescribed bounds. Moreover, the generated input is admissible, i.e., $\boldsymbol{u}\in\mathcal{U}$, and the control law is almost globally $C^1$-smooth on  $\mathcal{C}=\mathbb{R}^2\times\mathbb{S}^1$.
\end{theorem}

\begin{IEEEproof}
    See Appendix \ref{app:005}.
\end{IEEEproof}

\begin{remark}
It is worth noting that for nonholonomic mobile robots, especially Dubins car-like models described by \eqref{eq:kinematic_system}, the underlying principle of \eqref{eq:co_design_law} differs from most existing VF-based approaches.  For example, in \cite[eq.~(24)]{9312173}, \cite[eq.~(20)]{11300826}, \cite[eq.~(13)]{11246344}, \cite[eq.~(13)]{11127300}, \cite[eq.~(23)]{7942030}, \cite[eq.~(18)]{7484276}, and \cite[eq.~(31)]{9969449}, the angular velocity control law is deliberately designed as the sum of a feedforward and a feedback term. The feedforward term provides the angular velocity required to follow the vector field and is usually formulated as the time derivative of the vector field direction, e.g., $\dot{\angle \boldsymbol{\chi}} = \frac{-1}{\|\boldsymbol{\chi}\|}\hat{\boldsymbol{\chi}}^\top\boldsymbol{E}\boldsymbol{J}(\boldsymbol{\chi})\dot{\boldsymbol{\xi}}$, where $\boldsymbol{J}(\boldsymbol{\chi})$ is the Jacobian of $\boldsymbol{\chi}$ with respect to the position $\boldsymbol{\xi}$. The feedback term compensates for the heading error $\theta_e$. Although such control laws enable vector field tracking, they typically do not explicitly account for curvature constraints, and thus must employ input saturation to enforce feasibility. Moreover, under finite-time convergence requirements, the Jacobian $\boldsymbol{J}(\boldsymbol{\chi})$ may become unbounded near the desired manifold, which may render the control law infeasible. In contrast, \eqref{eq:co_design_law} satisfies Theorem \ref{thm:admissible_control}, it relies only on the curvature of the FT-C2VF and does not involve the Jacobian of the vector field.  
\hfill $\blacktriangleleft$
\end{remark}

To analyze the closed-loop error system, we study the evolution of the heading error $\theta_e$ and the radial ratio $\eta$. As detailed in Appendix \ref{app:derivation}, the dynamics is given by
\begin{subequations}\label{eq:yamhua}
\begin{align}
\dot{\theta}_e &= \omega - v(\kappa_\chi(\eta) \cos\theta_e + \kappa_\perp(\eta) \sin\theta_e),\label{eq:theta_e_1} \\
\dot{\eta} &=  \frac{v}{r} \big( \psi_{\nu} \cos\theta_e - \psi_{\tau} \sin\theta_e \big).\label{eq:eta_1}
\end{align}
\end{subequations}
Before the robot converges to the target configuration, we have $v>0$. By applying orbital equivalence in \cite[Def.~2.4]{Kuznetsov2023ElementsOA}, we introduce the time reparameterization $\mathrm{d}\tau = v\mathrm{d}t$. On the time scale $\tau$, \eqref{eq:yamhua} can be rewritten as the following closed-loop autonomous system on the cylindrical manifold $\mathcal{H} = \mathbb{S}^1 \times \mathbb{R}_{\ge 0}$:
\begin{equation}\label{eq:bihuan}
\dot{\boldsymbol{z}}=\boldsymbol{F}(\boldsymbol{z})=\begin{bmatrix}
\dot{\theta}_e \\
\dot{\eta}
\end{bmatrix} =
\begin{bmatrix}
\kappa_c(\eta,\theta_e) - \kappa_\chi(\eta) \cos\theta_e - \kappa_\perp(\eta) \sin\theta_e \\
\frac{1}{r} \big( \psi_{\nu} \cos\theta_e - \psi_{\tau} \sin\theta_e \big)
\end{bmatrix},
\end{equation}
where $\kappa_c(\eta,\theta_e)=\kappa_\chi(\eta)-k_\omega(\eta)\left\lceil\sin\theta_e\right\rfloor^\alpha$ and $k_\omega(\eta)=\bar{\kappa}-\kappa_\chi(\eta).$ We first exclude $\eta = 0 $ from the state space because the direction of the FT-C2VF, and hence the heading error $\theta_e$, is not defined at $\boldsymbol{\xi} = 0$. Then, we analyze system \eqref{eq:bihuan}. The desired equilibrium point is $\boldsymbol{z}_{d}=(0,1)$, corresponding to $\theta_e=0$ and $\eta=1$. Define  $\mathcal{E}=\{\boldsymbol{z} \in \mathcal{H}:\boldsymbol{F}(\boldsymbol{z}) =\boldsymbol{0}\}$. For a given equilibrium point $\boldsymbol{z}^*=(\theta_e^*,\eta^*)\in \mathcal{E}$, from $\dot{\eta}^* = 0$ and \eqref{eq:equal_psi}, we obtain
\begin{equation}\label{eq:geometric_identity}
\tan\theta_e^* = \frac{\psi_{\nu}(\eta^*)}{\psi_{\tau}(\eta^*)}
\implies
\begin{bmatrix}
\sin\theta_e^* \\
\cos\theta_e^*
\end{bmatrix} = \pm\begin{bmatrix}\psi_{\nu}(\eta^*) \\
\psi_{\tau}(\eta^*)\end{bmatrix}.
\end{equation}
Substituting \eqref{eq:identity}, \eqref{command_kappa}, and \eqref{eq:geometric_identity} into the condition $\dot{\theta}_e^* = 0$ gives two branch equations:
\begin{equation}\label{eq:trap_eq}
\Upsilon_{\pm}(\eta^*) := \kappa_\chi(\eta^*) - k_\omega(\eta^*)\lceil \pm\psi_{\nu}(\eta^{*})\rfloor^\alpha  \mp \frac{1}{r\eta^*} = 0.
\end{equation}
Here, $\eta^{*}$ is a root of the function $\Upsilon_\pm$, where the $\pm$ branches correspond to counterclockwise and clockwise orbital motions, respectively. We present the final stability result.

\begin{theorem}\label{thm:finit_time}
Suppose that $\gamma \in (0,1)$ and
\begin{subequations}\label{eq:param_choose}
\begin{align}
0 < \alpha < \frac{2\gamma}{1+\gamma}, \label{eq:alpha_choose} \\
r \ge \frac{3 + 2\gamma}{\bar{\kappa}}. \label{eq:r_choose}
\end{align}
\end{subequations}
Then system \eqref{eq:bihuan} has the following properties:
\begin{itemize}
    \item[1)] The system possesses exactly two equilibria, namely the desired equilibrium $\boldsymbol{z}_{d}$ and an equilibrium $\boldsymbol{z}_s = (\theta_{es}, \eta_s)$ on the negative branch, where $\eta_s$ is the unique root of $\Upsilon_-(\eta_s) = 0$ in $(1,+\infty)$ and $\theta_{es} \in (\pi/2,\pi)$ is uniquely determined by $\sin\theta_{es} = - \psi_\nu(\eta_s),\cos\theta_{es} = -\psi_\tau(\eta_s)$. 
    \item[2)] The equilibrium $\boldsymbol{z}_s$ is unstable.
    \item[3)] System \eqref{eq:bihuan} admits no nontrivial periodic orbits; in particular, it admits no limit cycles.
    \item[4)]The desired equilibrium $\boldsymbol{z}_{d}$ is almost globally finite-time stable on $\mathcal{H}$.
\end{itemize}
\end{theorem}

\begin{IEEEproof}
    See Appendix \ref{app:007}.    
\end{IEEEproof}

Building upon the finite-time stability of the closed-loop error system established in Theorem \ref{thm:finit_time}, we now extend our analysis to the full kinematic response of the nonholonomic robot. The following theorem formalizes the almost-global finite-time convergence of the robot to the target configuration under the proposed control laws, thereby fulfilling the two objectives (C1 and C2) in Problem~\ref{pro:002}.

\begin{theorem}\label{thm:velocity_convergence}
Under the control laws \eqref{eq:vlaw} with $\beta \in (0, 0.5)$ and \eqref{eq:omegalaw}, the nonholonomic robot \eqref{eq:kinematic_system} subject to the constraints \eqref{eq:control_input_set} converges to the target configuration $\boldsymbol{q}_d$ in finite time, except for initial configurations belonging to the measure-zero non-converging set
$\mathcal{Q}_c=\{\boldsymbol{q}=[\boldsymbol{\xi}^\top,\theta]^\top: \boldsymbol{\xi} \in \mathcal{W}_\eta\} \cup  \{\boldsymbol{q}=[\boldsymbol{\xi}^\top,\theta]^\top: (\theta-\angle \boldsymbol{\chi}(\boldsymbol{\xi}), \|\boldsymbol{\xi}\|/r) = \boldsymbol{z}_s\}$.
\end{theorem}

\begin{IEEEproof}
See Appendix \ref{app:008}. 
\end{IEEEproof}

\begin{remark}
    Notably, if the condition on $\gamma$ in Theorem \ref{thm:finit_time} is changed to $\gamma >  1$ (while other conditions of $\alpha$ and $r$ remain), the original finite-time convergence of \eqref{eq:bihuan} becomes to only almost global asymptotic stability at $\boldsymbol{z}_{d}$. In the existing literature, most VF-based motion-planning algorithms for nonholonomic robots rely heavily on the Jacobian matrix of the vector field (i.e., the change rate of  the vector field orientation) to derive the control law, regardless of whether curvature constraints are enforced \cite{11300826,7484276,11246344,9312173,11127300}. It remains largely unexplored whether reliable VF tracking can be achieved using solely geometric curvature feedforward and heading feedback. This work has answered this question by establishing a novel Jacobian-free control law with strict convergence guarantees while satisfying explicit curvature constraints.
\end{remark}

\section{Numerical Validation and Method Comparison}\label{sec:5}
In this section, we verify the instability of the non-converging set and the effectiveness of the proposed algorithm using the numerical simulation groups detailed in Table \ref{tab:initial_conditions}. Moreover, Monte Carlo experiments demonstrate the advantages of the proposed approach over existing VF-based motion-planning methods. For all simulations, the kinematic parameters are configured as follows: $v_{+} = 1$, $\bar{\omega} = 1$, $\rho = 1$ and $\bar{\kappa} = 1$. The parameters of the FT-C2VF are set to $\gamma = 0.5$ and $r = 4$, while the control law parameters are chosen as $\alpha = 0.6$, $\beta = 0.4$, $k_v = 1$, and $\mu = 4$.

\begin{table}[htbp]
  \centering
  \footnotesize
  \setlength{\tabcolsep}{3pt}
  \caption{Initial and Target Configurations of Simulations}\label{tab:001}
  \label{tab:initial_conditions}
  \begin{tabular}{@{} l l c c c @{}} 
    \toprule
    \multicolumn{2}{c}{\textbf{Setting}} & $\boldsymbol{q}_0^\top = [x_0, y_0, \theta_0]$ & $\boldsymbol{q}_d^\top = [x_d, y_d, \theta_d]$ & $v_d$ \\
    \midrule
    \multirow{4}{*}{\begin{tabular}[c]{@{}l@{}}\textbf{Group 1}\\ $(\epsilon=10^{-4})$\end{tabular}} 
    & \textbf{Rob.~1} & $[0, 0, 0]$ & $[0, 4, \pi]$ & $0$ \\
    & \textbf{Rob.~2} & $[\epsilon, \epsilon, \epsilon]$ & $[0, 4, \pi]$ & $0$ \\
    & \textbf{Rob.~3} & $[r\eta_s, 0, \theta_s]$ & $[0, -4, 0]$ & $0$ \\
    & \textbf{Rob.~4} & $[r\eta_s+\epsilon, \epsilon, \theta_s+\epsilon]$ & $[0, -4, 0]$ & $0$ \\
    \midrule
    \multirow{4}{*}{\textbf{Group 2}} 
    & \textbf{Rob.~5} & $[0.8, 0.7, -5\pi/6]$ & $[-2\sqrt{2}, 2\sqrt{2}, -3\pi/4]$ & $0$ \\
    & \textbf{Rob.~6} & $[-0.5, 0.7, -\pi/5]$ & $[-2\sqrt{2}, -2\sqrt{2}, -\pi/4]$ & $0$ \\
    & \textbf{Rob.~7} & $[-1.2, -0.9, \pi/4]$ & $[2\sqrt{2}, -2\sqrt{2}, \pi/4]$ & $0$ \\
    & \textbf{Rob.~8} & $[0.4, -0.7, 5\pi/8]$ & $[2\sqrt{2}, 2\sqrt{2}, 3\pi/4]$ & $0$ \\
    \midrule
    \multirow{4}{*}{\textbf{Group 3}} 
    & \textbf{Rob.~9} & $[4.9, 4.8, -3\pi/5]$ & $[-4, 0, -\pi/2]$ & $0.2$ \\
    & \textbf{Rob.~10} & $[-3.9, 3.8, 2\pi/3]$ & $[0, -4, 0]$ & $0.4$ \\
    & \textbf{Rob.~11} & $[-4.2, -4.5, -2\pi/5]$ & $[4, 0, \pi/2]$ & $0.6$ \\
    & \textbf{Rob.~12} & $[5.4, -4.2, -\pi/2]$ & $[0, 4, \pi]$ & $0.8$ \\
    \bottomrule
  \end{tabular}
\end{table}

\begin{figure*}[!t]
    \centering
    \subfloat[]{\includegraphics[width=0.3775\linewidth]{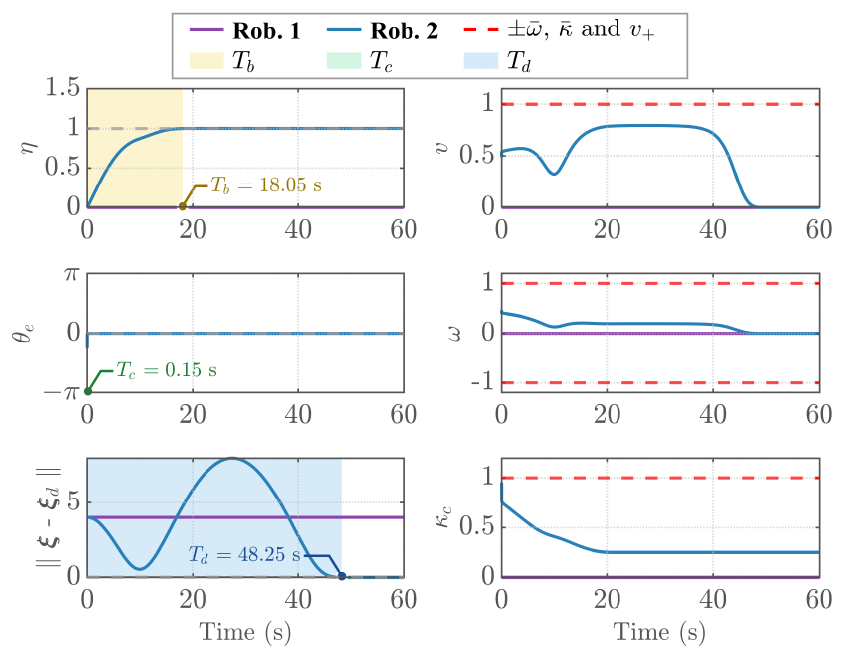}%
    \label{fig:data_group1}}
    \subfloat[]{\includegraphics[width=0.235\linewidth]{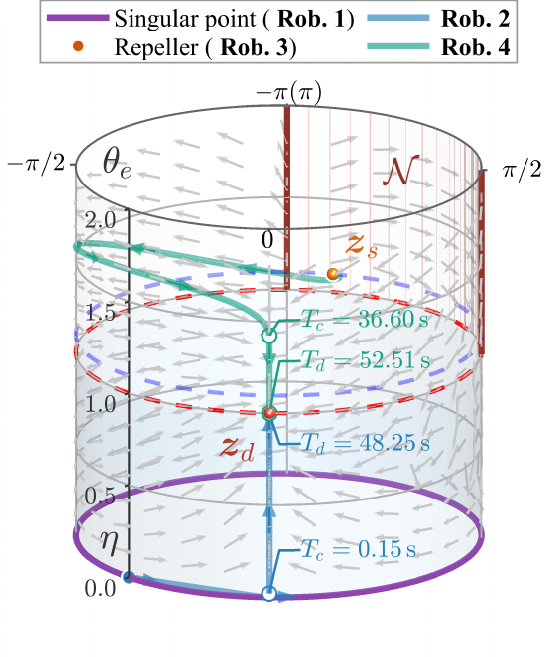}%
    \label{fig:limit_cycle}}
    \subfloat[]{\includegraphics[width=0.3775\linewidth]{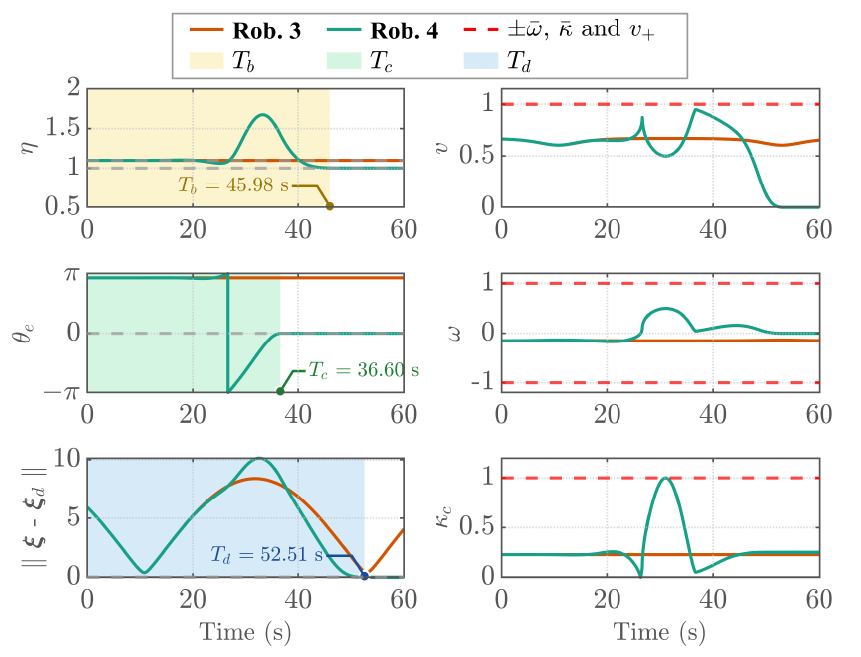}%
    \label{fig:data_group2}}
    
    \subfloat[]{\includegraphics[width=0.245\linewidth]{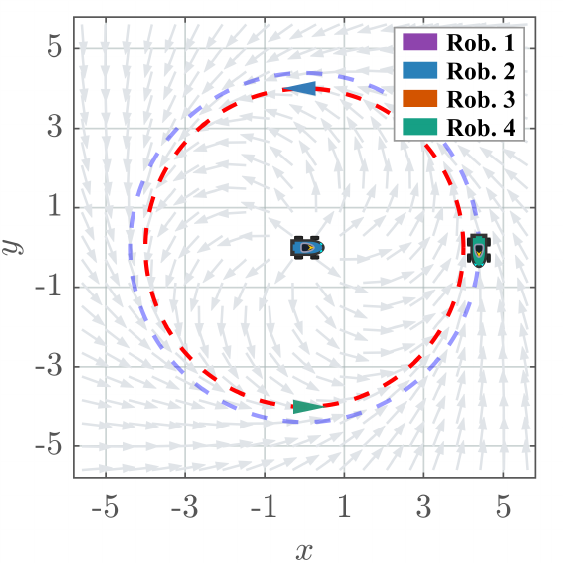}%
    \label{fig:snap1}}
    \subfloat[]{\includegraphics[width=0.245\linewidth]{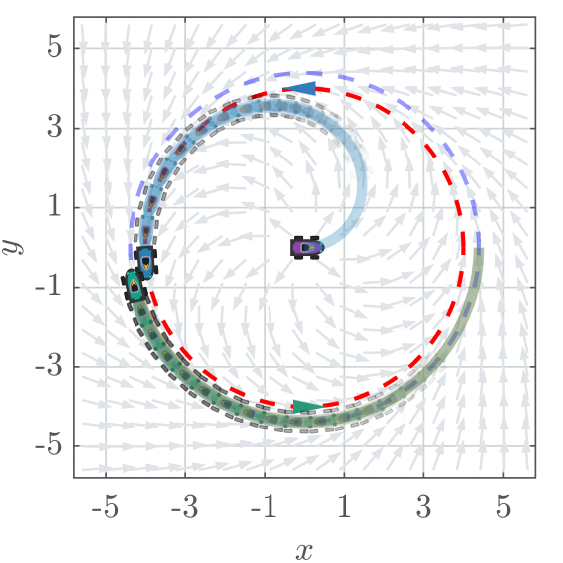}%
    \label{fig:snap2}}
    \subfloat[]{\includegraphics[width=0.245\linewidth]{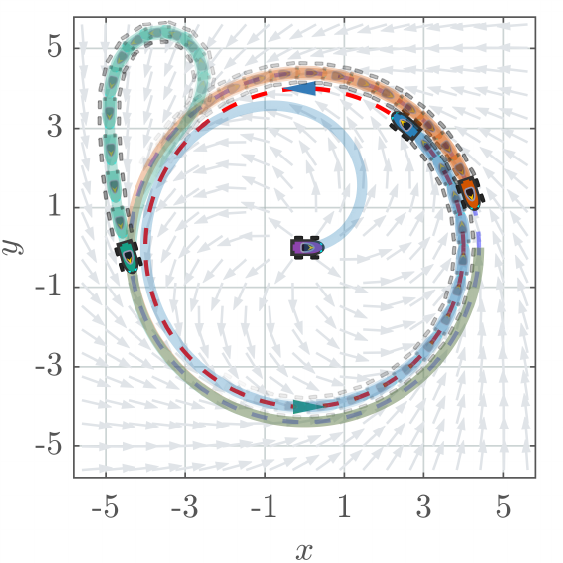}%
    \label{fig:snap3}}
    \subfloat[]{\includegraphics[width=0.245\linewidth]{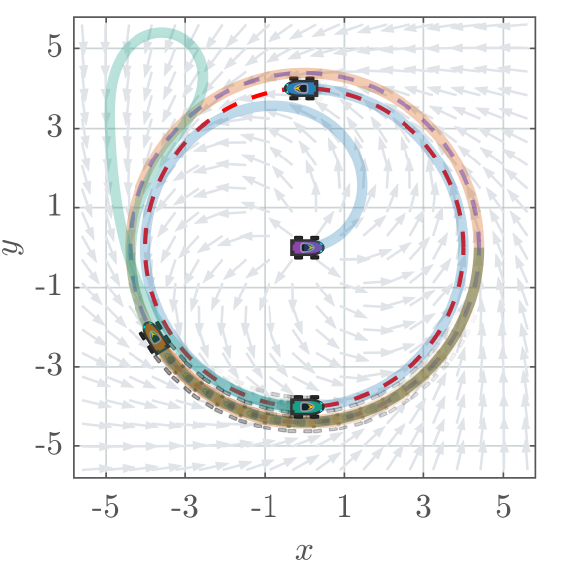}%
    \label{fig:snap4}}
    
    \caption{Simulation verification of the instability of the non-converging set. Under the proposed control law, (a) and (c) present the temporal evolution of the states for different robots, where the state errors ($\eta$, $\theta_e$, and $\|\boldsymbol{\xi}-\boldsymbol{\xi}_d\|$) converge to zero in finite time ($T_b$, $T_c$, and $T_d$ denote the convergence times for the respective stages). (b) illustrates the trajectory evolution of the four robots governed by the closed-loop system. (d)–(g) display snapshots of $\textrm{Rob.~1-4}$ at various time instants: (d) $t = 0\mathrm{s}$; (e) $t = 20\mathrm{s}$; (f) $t = 40\mathrm{s}$; and (g) $t = 60\mathrm{s}$.}
    \label{fig:simulation1}
\end{figure*}

\subsection{Instability of the Non-converging Set}\label{sec:5.A}
In practical physical environments, disturbances such as sensor noise are inevitable. Even if the initial configuration of the robot satisfies $\boldsymbol{q}_0 \in \mathcal{Q}_c$, minor perturbations force the robot out of this set, triggering the finite-time convergence property. The instability of $\mathcal{Q}_c$ is verified through Group 1. As illustrated in Figs.~\ref{fig:data_group1} and \ref{fig:data_group2}, under ideal conditions without perturbations, Rob.~1, positioned at the singularity, remains stationary, whereas Rob.~3, initialized at the repeller, orbits with a radius of $R = r\eta_s$ and a heading satisfying $\theta = \theta_{es} + \angle\boldsymbol{\chi}(\boldsymbol{\xi})$. However, when subject to minor perturbations, the corresponding robots (Rob.~2 and Rob.~4) escape these unstable states and successfully converge to the target configuration in finite time. Furthermore, Fig.~\ref{fig:limit_cycle} depicts the phase-plane trajectories governed by the closed-loop system. Theorem \ref{thm:finit_time} confirms that the perturbed trajectories first converge to the forward-invariant manifold $\mathcal{P}$ at time $T_c$, and subsequently reach the desired equilibrium $\boldsymbol{z}_{d}$ at time $T_d$. Snapshots detailing the dynamic evolution of the trajectories are provided in Figs.~\ref{fig:snap1}--\ref{fig:snap4}.

\begin{figure*}[!t]
    \centering
    \subfloat[]{
        \includegraphics[width=0.96\linewidth]{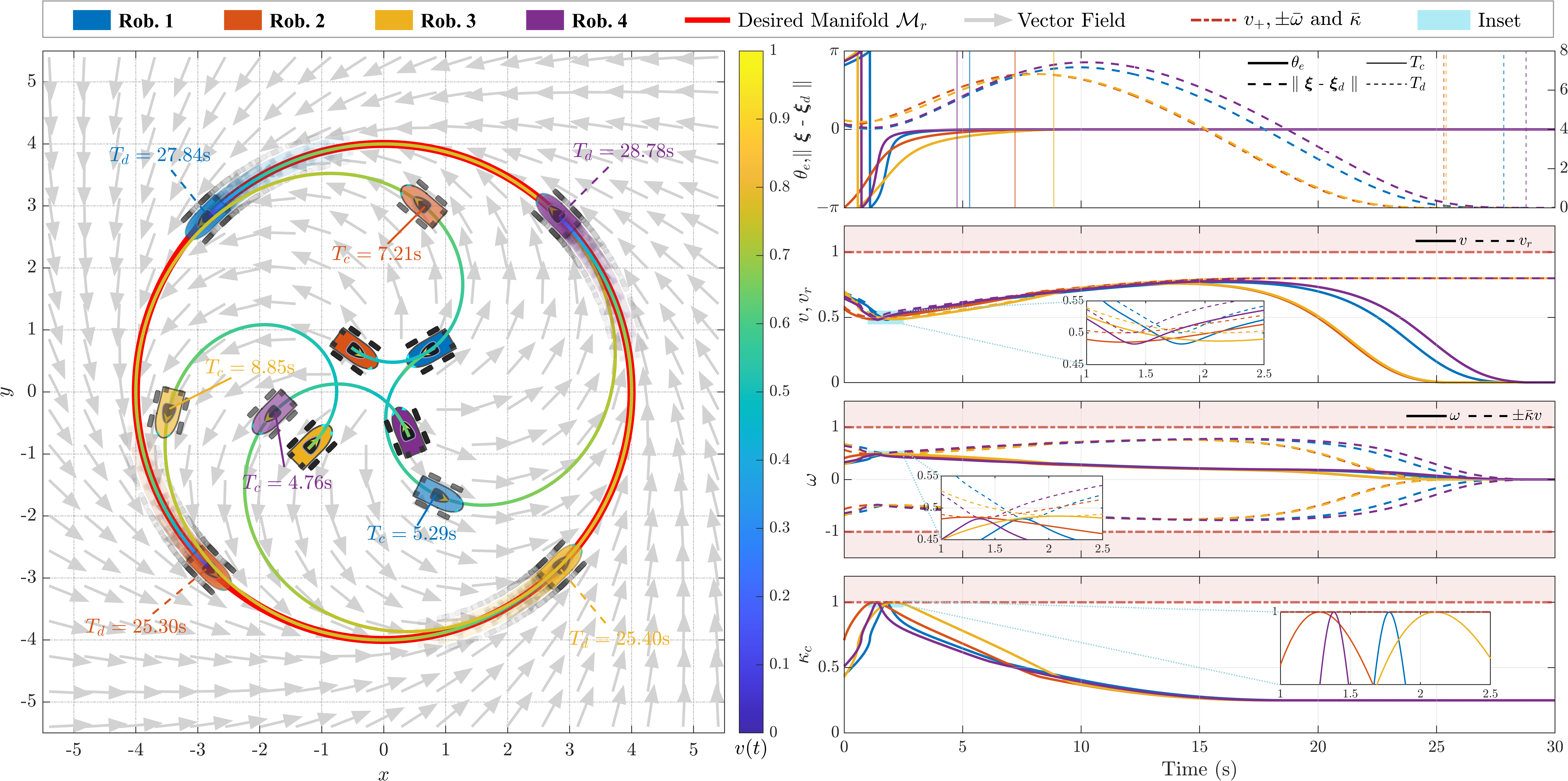}
        \label{fig:sub1}
    }%
    \hspace{-2mm}%
    \subfloat[]{
        \includegraphics[width=0.96\linewidth]{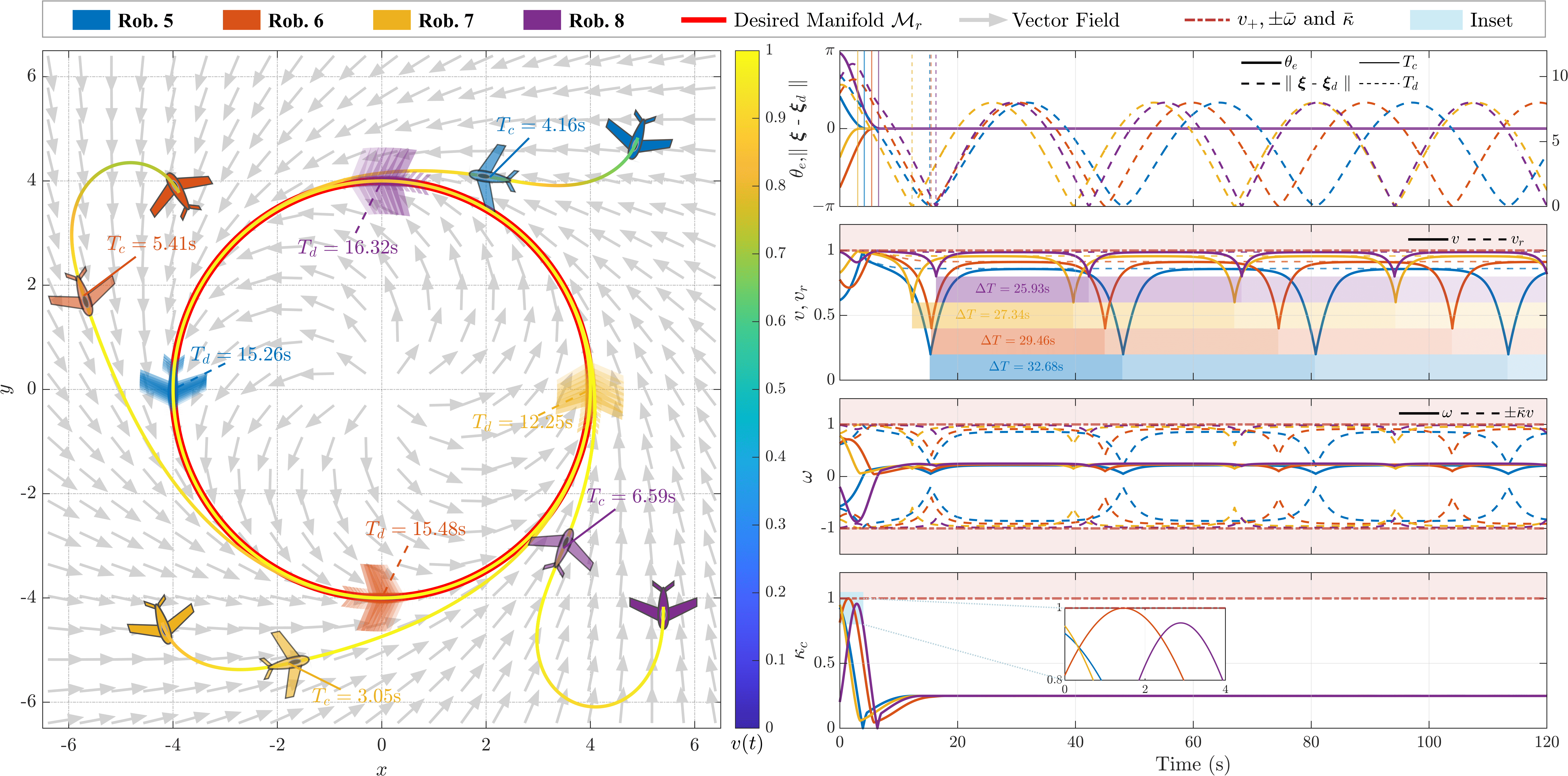}
        \label{fig:sub2}
    }
    \caption{Simulation verification of the effectiveness of the proposed algorithm. Under the proposed control law, the left panels illustrate the motion trajectories of the robots and their corresponding velocity distributions. The right panels of (a) demonstrate that each robot in Group 2 converges to the target configuration $\boldsymbol{q}_d$. The right panels of (b) show that each robot in Group 3 periodically passes through the target configuration $\boldsymbol{q}_d$ with a desired velocity $v_d$ and a recurrence period of $\Delta T$. Moreover, the heading errors $\theta_e$ and the position errors $\|\boldsymbol{\xi}-\boldsymbol{\xi}_d\|$ of all robots converge to zero in finite time, at $T_c$ and $T_d$, respectively.}
    \label{fig:simulation2}
\end{figure*}

\subsection{Effectiveness of the Proposed Algorithm}\label{sec:5.B}
Group 2 and Group 3 are utilized to validate the efficacy of the proposed approach. Under the proposed motion planning algorithm, as shown in Figs.~\ref{fig:sub1} and \ref{fig:sub2}, the configuration errors of the nonholonomic robots converge to zero, while the control inputs and trajectory curvatures remain confined within their permissible physical limits, i.e., $|\omega| \le \min\{\bar{\omega}, \bar{\kappa}v\}$ and $|\kappa|\le \bar{\kappa}$. These results further validate Theorems \ref{thm:admissible_control}--\ref{thm:velocity_convergence} established in Section \ref{sec:4}.

\subsection{Comparison with other VF-based Methods}\label{sec:5.C}

\begin{table*}[htbp]
    \centering
    \caption{Vector Fields and Control Laws in Monte Carlo Experiments}
    \label{tab:parameter_setup}
    \renewcommand{\arraystretch}{1.4} 
    \footnotesize
    
    \begin{tabular}{l c c c >{\centering\arraybackslash}p{4.5cm}}
        \toprule[1.5pt] 
        \textbf{Algorithm} & \textbf{Vector Field} & \textbf{Forward Speed} & \textbf{Angular Velocity} & \textbf{Parameter} \\
        \midrule
        \textbf{FT-C2VF} (Ours) & 
        Eq.~\eqref{eq:FT-C2VF} & 
        Eq.~\eqref{eq:vlaw} & 
        Eq.~\eqref{eq:omegalaw} & 
       Same as in Section \ref{sec:5.B} \\
        \textbf{CVF} \cite{11300826} & 
        \cite[eq.~(15)]{11300826} & 
        \cite[eq.~(20)]{11300826} & 
        \cite[eq.~(17)]{11300826} & 
        $c_{\boldsymbol{p}}=9, c_{\omega}=\pi$, \cite[eq.~(26)]{11300826} \\
        \textbf{AVF} \cite{7484276} & 
        \cite[eq.~(8)]{7484276} & 
        \cite[eq.~18(a)]{7484276} & 
        \cite[eq.~18(b)]{7484276} & 
        $k_u=1, k_\omega=1$ \\
        \textbf{GVF} \cite{7942030} & 
        \cite[eq.~(9)]{7942030} & 
        Constant & 
        \cite[eq.~(23)]{7942030} & 
        $k_n=0.18,k_\delta = 1$ \\
        \textbf{FT-GVF} \cite{11246344} & 
        \cite[eq.~(3)]{11246344} & 
        Constant & 
        \cite[eq.~(13)]{11246344} & 
        $k=0.4,\alpha=0.6,k_\theta = 1$ \\   
        \bottomrule[1.5pt] 
        \multicolumn{5}{l}{\footnotesize The parameter notations are adopted from their respective original references.} \\
    \end{tabular}
\end{table*}

\begin{figure*}[!t]
\centering
\noindent\makebox[\linewidth][c]{%
\subfloat[]{
    \includegraphics[width=0.61\linewidth]{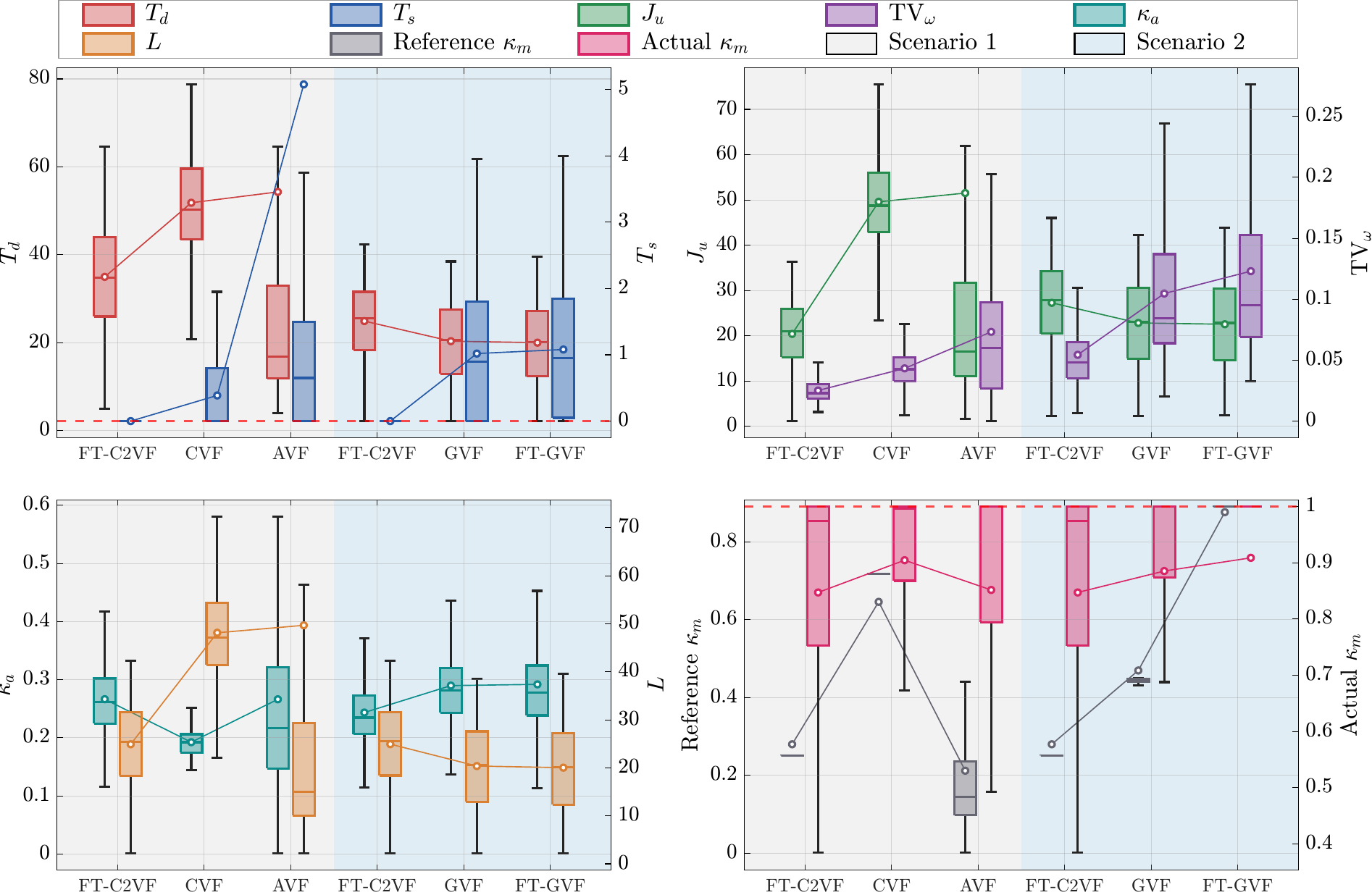}
    \label{fig3:sub_a}
} 
\subfloat[FT-C2VF]{
    \includegraphics[width=0.36\linewidth]{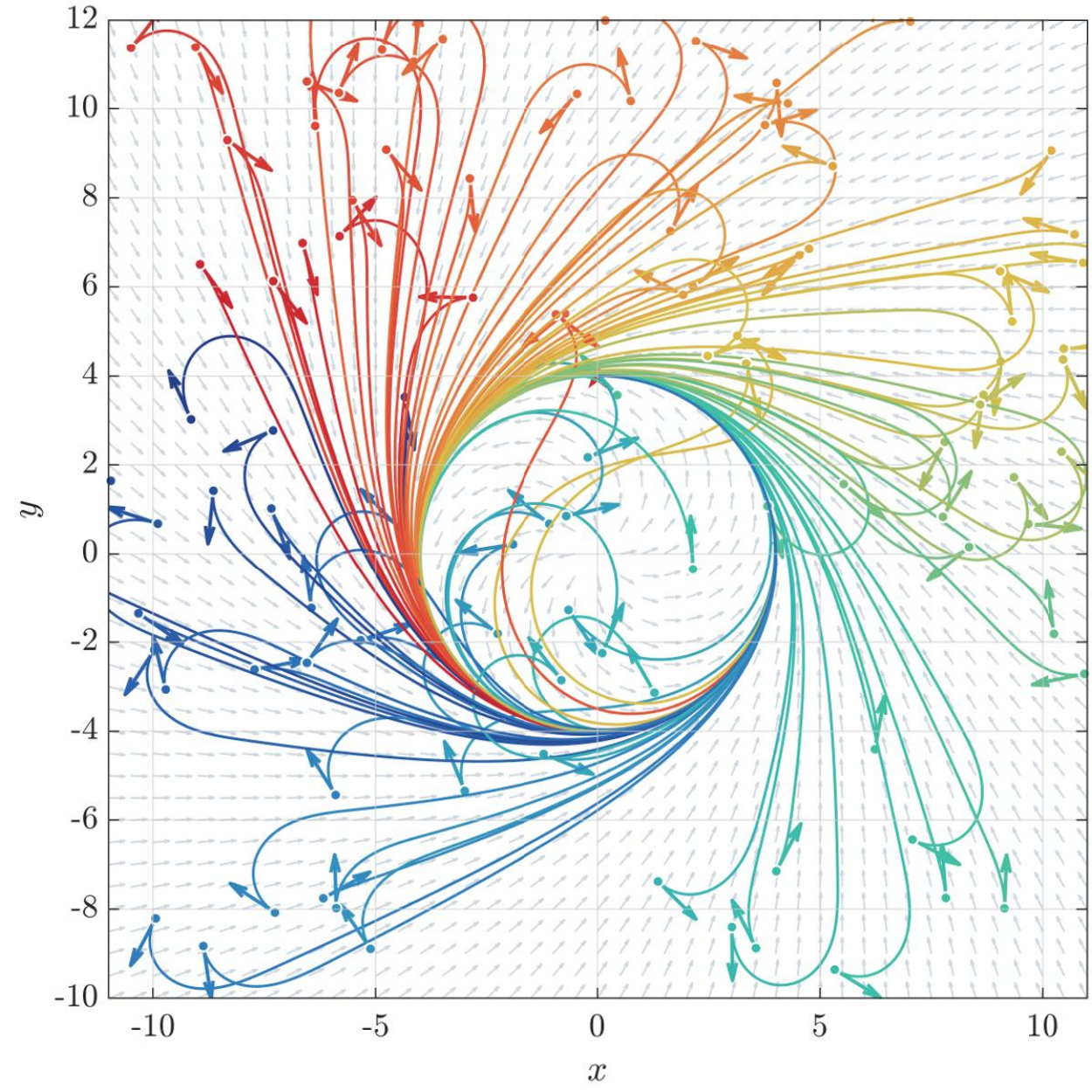}
    \label{fig3:sub_b}
}
}%
\\[-2mm]
\noindent\makebox[\linewidth][c]{%
\subfloat[CVF]{
    \includegraphics[width=0.24\linewidth]{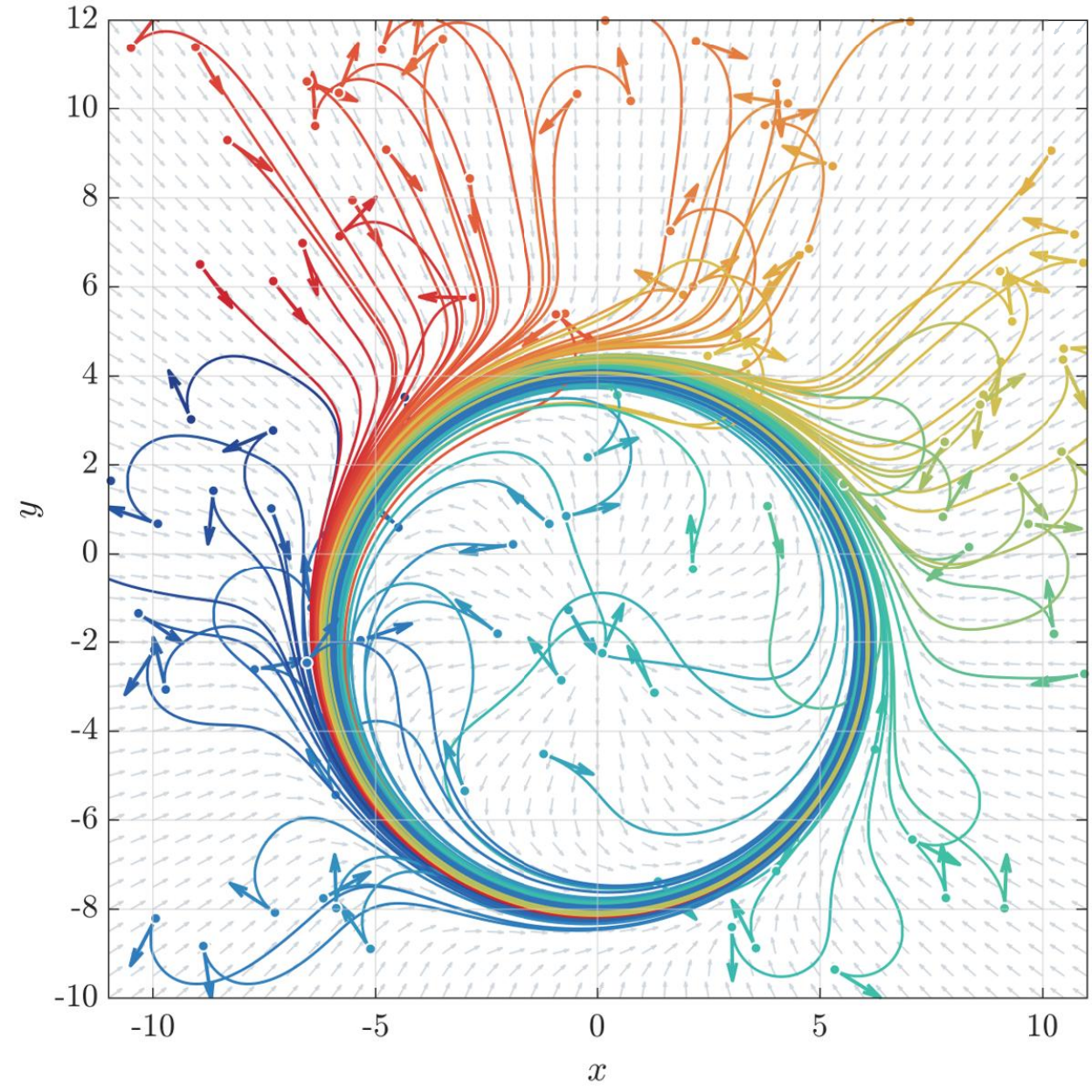}
    \label{fig3:sub_c}
}%
\hspace{-2mm}%
\subfloat[AVF]{
    \includegraphics[width=0.24\linewidth]{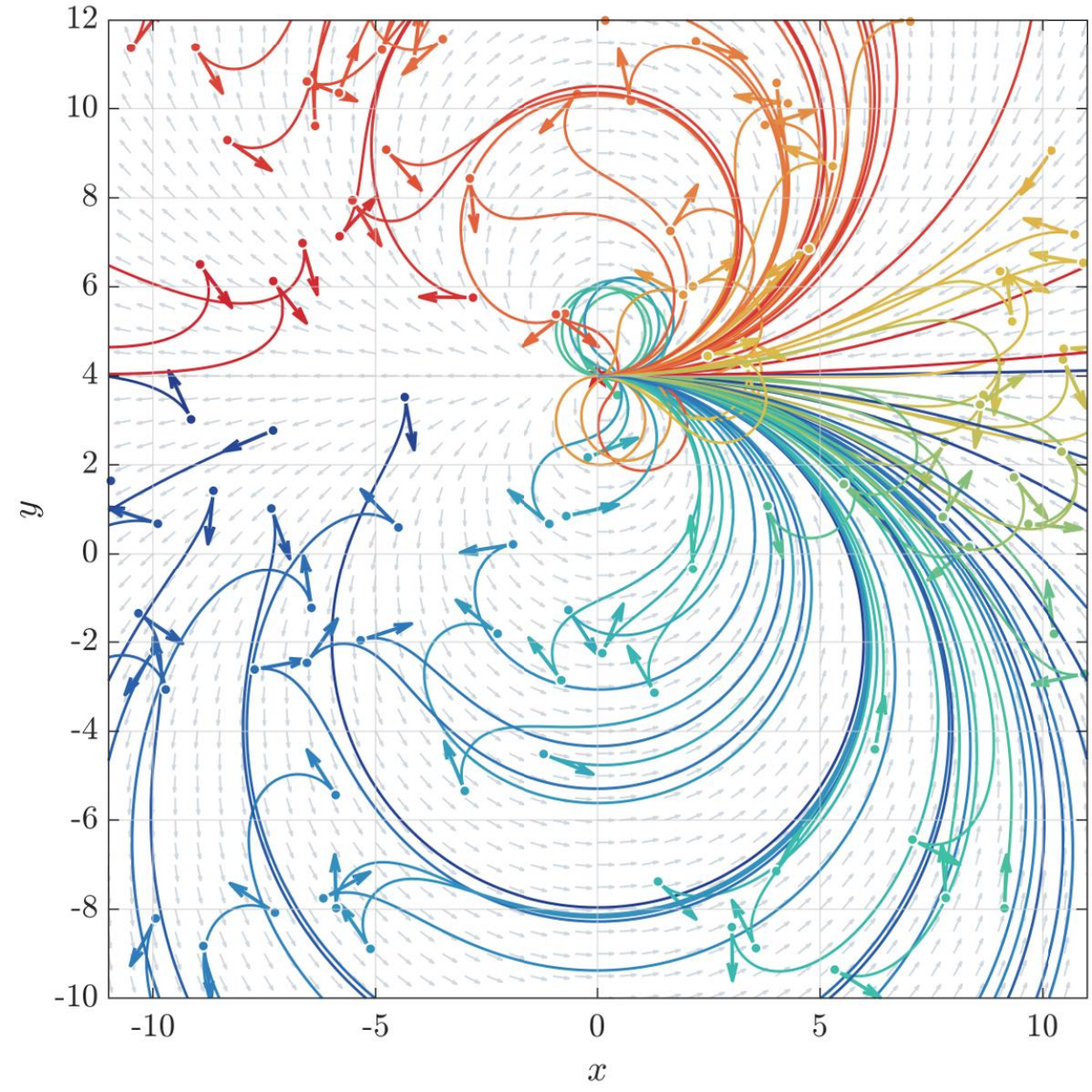}
    \label{fig3:sub_d}
}%
\hspace{-2mm}%
\subfloat[GVF]{
    \includegraphics[width=0.24\linewidth]{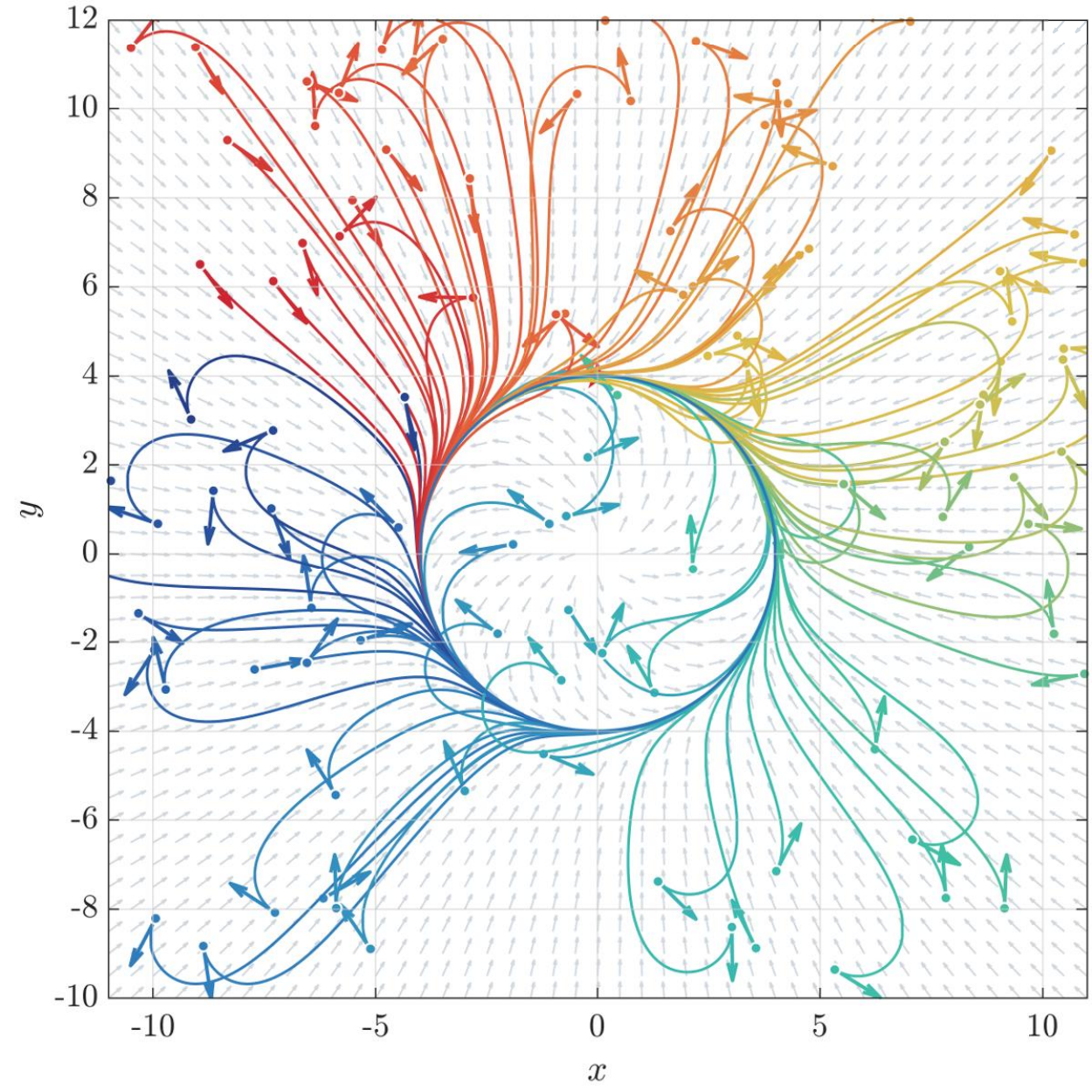}
    \label{fig3:sub_e}
}%
\hspace{-2mm}%
\subfloat[FT-GVF]{
    \includegraphics[width=0.24\linewidth]{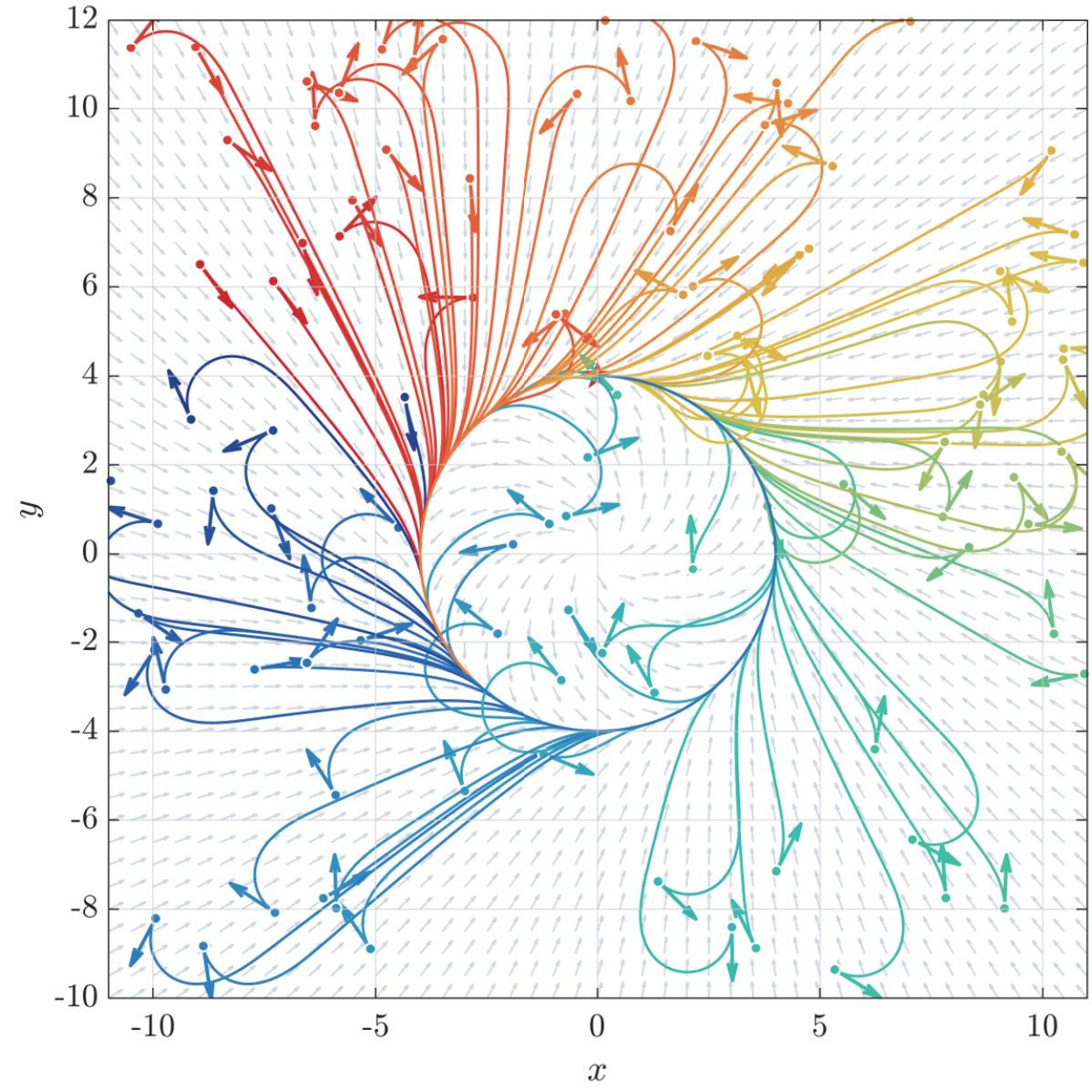}
    \label{fig3:sub_f}
}%
}
\caption{Monte Carlo experimental results of $\epsilon = 0.1$. (a) Statistical distributions of the metrics for different algorithms across 1000 trials. The lower and upper boundaries of the box plots correspond to the 25th ($Q_1$) and 75th ($Q_3$) percentiles of the data, respectively, while the internal horizontal lines denote the medians. The whiskers extend to the most extreme actual data points bounded within $[Q_1 - 1.5\mathrm{IQR}, Q_3 + 1.5\mathrm{IQR}]$, where $\mathrm{IQR} = Q_3 - Q_1$ represents the interquartile range. Data points outside this range are defined as outliers; these are omitted from the visualization to emphasize the primary distribution characteristics. The hollow circles indicate the mean values of the respective metrics. The red horizontal dashed lines demarcate $T_s = 0$ and $\bar{\kappa} = 1$, respectively. (b)-(f) illustrate the initial configurations and the resulting motion trajectories of the evaluated algorithms in the first 100 trials.}
\label{fig:simulation3}
\end{figure*}

\textit{1) Setup:} We conduct Monte Carlo simulations comparing the proposed method with four existing VF-based motion-planning algorithms. By design, only the proposed FT-C2VF and the CVF explicitly incorporate curvature constraints; the baselines (AVF, GVF, and FT-GVF) assume unconstrained kinematics. As summarized in Table \ref{tab:parameter_setup}, the parameters of GVF and FT-GVF are manually tuned to ensure a fair comparison, with their vector fields also satisfying the curvature bound $\bar{\kappa}$. However, AVF remains untuned, as its formulation does not permit such curvature constraint enforcement. The evaluations are divided into two distinct scenarios: a variable-speed scenario evaluates the CVF, AVF, and FT-C2VF using an Ackermann-steered model with $v_+ = 1$ and $v_- = 0$; in contrast, a constant-speed scenario assesses the GVF, FT-GVF, and FT-C2VF using a unicycle model with a fixed speed of $v= 1$. To establish a standardized baseline, the target configuration is fixed at $\boldsymbol{q}_d = [0, r, \pi]^\top$, while the initial configuration $\boldsymbol{q}_0 = [x_0, y_0, \theta_0]^\top$ is sampled uniformly at random from the configuration space $\mathcal{Q}_0 = [-11, 11] \times [-10, 12] \times (-\pi, \pi]$ across $1000$ independent trials.

\textit{2) Specifications:} To ensure a rigorous and fair comparison, three alignment protocols are enforced across the baselines. First, to simultaneously enforce the curvature constraint, the angular velocity is limited by the saturation function $\operatorname{Sat}_{a}^{b}: \mathbb{R} \to \mathbb{R}$, defined as $\operatorname{Sat}_{a}^{b}(x) = x$ for $x \in [a, b]$, $\operatorname{Sat}_{a}^{b}(x) = a$ for $x \in (-\infty, a)$, and $\operatorname{Sat}_{a}^{b}(x) = b$ for $x \in (b, \infty)$, where $a < b$ are constants. Second, to accommodate the structural constraints of the CVF \cite[eq.~(14)]{11300826}, the circle manifolds $\mathcal{M}_r$ are designed as
\begin{equation*}
\mathcal{M}_r = 
\begin{cases}
\{\boldsymbol{\xi} \in \mathbb{R}^2 : \|\boldsymbol{\xi} - [0, -2\rho]^\top\| = 6\rho\}, & \text{CVF}, \\ \{\boldsymbol{\xi} \in \mathbb{R}^2 : \|\boldsymbol{\xi}\| = r\}, & \text{Others}.
\end{cases}
\end{equation*}
The manifold-free AVF is excluded from this specific formulation. Consequently, the actual control input is given by $\omega_r(t) = \operatorname{Sat}_{-\tilde{\omega}(t)}^{\tilde{\omega}(t)}\big(\omega(t)\big)$. Here, $\tilde{\omega}(t) = \bar{\kappa}v(t)$ in the variable-speed scenario, and $\tilde{\omega}(t) = \bar{\omega}$ in the constant-speed scenario. Finally, the task is considered complete if the robot configuration enters this set: $\mathcal{Q}_\epsilon = \big\{ \boldsymbol{q} \in \mathcal{C} : \|\boldsymbol{\xi} - \boldsymbol{\xi}_d\| \le \epsilon, |\theta - \theta_d| \le 0.05 \big\}$. The convergence time is then defined as $T_d(\epsilon) = \inf \big\{ t \ge 0: \boldsymbol{q}(t) \in \mathcal{Q}_\epsilon \big\}$.

\textit{3) Metrics:} The evaluations focus on three primary aspects: convergence performance, trajectory quality, and control performance. Convergence performance is quantified by the convergence time $T_d(\epsilon)$ given a specified error tolerance $\epsilon$. Trajectory quality evaluates the path lengths $L$, mean curvatures $\kappa_a$, and maximum curvatures $\kappa_m$ of both reference and actual trajectories. Furthermore, control performance is assessed through the saturation time $T_s = \int_{0}^{T_a} \mathbb{I}\Big(|\omega(t)| \ge \min(\bar{\omega}, \bar{\kappa}v(t))\Big) \mathrm{d}t$, the control energy $J_u = \int_{0}^{T_a} \big(v^2(t) + \omega^2(t)\big) \mathrm{d}t$, and the total variation of the angular velocity $\mathrm{TV}_\omega = \frac{1}{T_a} \int_{0}^{T_a} |\dot{\omega}(t)| \mathrm{d}t$, where $\mathbb{I}(\cdot)$ denotes the indicator function, yielding $1$ if the enclosed proposition is true and $0$ otherwise.

\begin{table}[htbp]
    \centering
    \footnotesize
    \setlength{\tabcolsep}{2.8pt}
    \renewcommand{\arraystretch}{1.13}
    \caption{Average Performance Metrics of Different Algorithms}
    \label{tab:metrics}
    \begin{tabular}{@{} l l c c c c c c c c @{}}
        \toprule
        \multicolumn{2}{c}{\textbf{Algorithm}} 
        & $T_d$ 
        & $T_s$ 
        & $J_u$ 
        & $\mathrm{TV}_\omega$ 
        & $\kappa_a$ 
        & $L$ 
        & $\kappa^{\mathrm{ref}}_m$ 
        & $\kappa^{\mathrm{act}}_m$ \\
        \midrule

        \multirow{3}{*}{\textbf{Scenario 1}} 
        & \textbf{FT-C2VF} & \textbf{35.00} & \textbf{0.00} & \textbf{20.41} & \textbf{0.03} & 0.27 & \textbf{24.90} & 0.28 & \textbf{0.84} \\
        & \textbf{CVF}    & 51.84 & 0.39 & 49.58 & 0.04 & \textbf{0.19} & 48.14 & 0.65 & 0.90 \\
        & \textbf{AVF}    & 54.31 & 5.08 & 51.51 & 0.07 & 0.26 & 49.69 & \textbf{0.21} & 0.85 \\
        \midrule

        \multirow{3}{*}{\textbf{Scenario 2}} 
        & \textbf{FT-C2VF} & 24.94 & \textbf{0.00} & 27.28 & \textbf{0.05} & \textbf{0.24} & 24.94 & \textbf{0.28} & \textbf{0.85} \\
        & \textbf{GVF}    & 20.37 & 1.02 & 22.83 & 0.10 & 0.29 & 20.38 & 0.47 & 0.88 \\
        & \textbf{FT-GVF} & \textbf{20.03} & 1.08 & \textbf{22.55} & 0.12 & 0.29 & \textbf{20.04} & 0.88 & 0.90 \\
        \bottomrule
        \multicolumn{10}{l}{\scriptsize The bold values represent the best performance among these algorithms.}
    \end{tabular}
\end{table}

\begin{figure}[!h]
\centering
\includegraphics[width=3.2in]{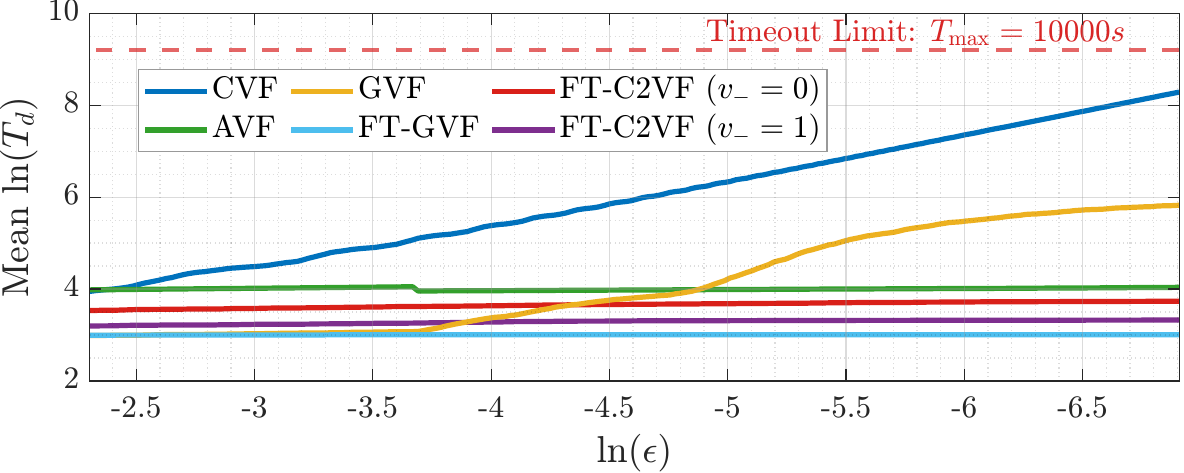}
\caption{The evolution curves of the logarithmic mean convergence time $\ln(T_d)$ versus the logarithmic tolerance $\ln(\epsilon)$ across all experiments. The horizontal axis is plotted in descending order, and the red dashed line indicates the timeout limit $T_{\max} = 10000s$.}
\label{fig:rongcha}
\end{figure}

\textit{4) Results:} The statistical results for $\epsilon=0.1$ in \cref{fig3:sub_a} and Table \ref{tab:metrics} reveal three notable properties of FT-C2VF. First, it outperforms the CVF method across all metrics except the mean actual curvature. Second, whereas the baseline methods enforce the curvature bound $\kappa_m \le \bar{\kappa}$ via input saturation, they do so at the cost of control saturation (i.e., a high saturation time $T_s$). By contrast, FT-C2VF inherently satisfies the curvature constraints, maintaining $T_s \equiv 0$ across all trials. Third, FT-C2VF achieves the smallest median and interquartile range of the angular-velocity variation $\mathrm{TV}_{\omega}$, indicating a smoother control action. This smoothness contrasts with the pronounced input chattering observed in FT-GVF, primarily due to the singularity of the Jacobian near the desired manifold $\mathcal{M}_r$. However, curvature constraints limit the maneuverability of the robot, thereby increasing the trajectory length, which consequently leads to longer convergence times and higher control energy in the constant-speed scenario. This represents a reasonable kinematic trade-off. The resultant spatial trajectories generated by each algorithm are depicted in \crefrange{fig3:sub_b}{fig3:sub_f}.  As illustrated in \cref{fig:rongcha}, the convergence time $T_d$ of FT-C2VF, AVF, and FT-GVF remains essentially insensitive to the error tolerance $\epsilon$, whereas that of CVF and GVF increases as $\epsilon$ decreases, reflecting their asymptotic convergence behavior.


\section{Hardware Experiment Results}\label{sec:6}
In this section, we conduct outdoor experiments using an Ackermann-steered vehicle to validate the effectiveness and robustness of the proposed closed-loop motion planning algorithm across different robotic platforms.

\begin{figure}[!t]
\centering
\includegraphics[width=0.95\linewidth]{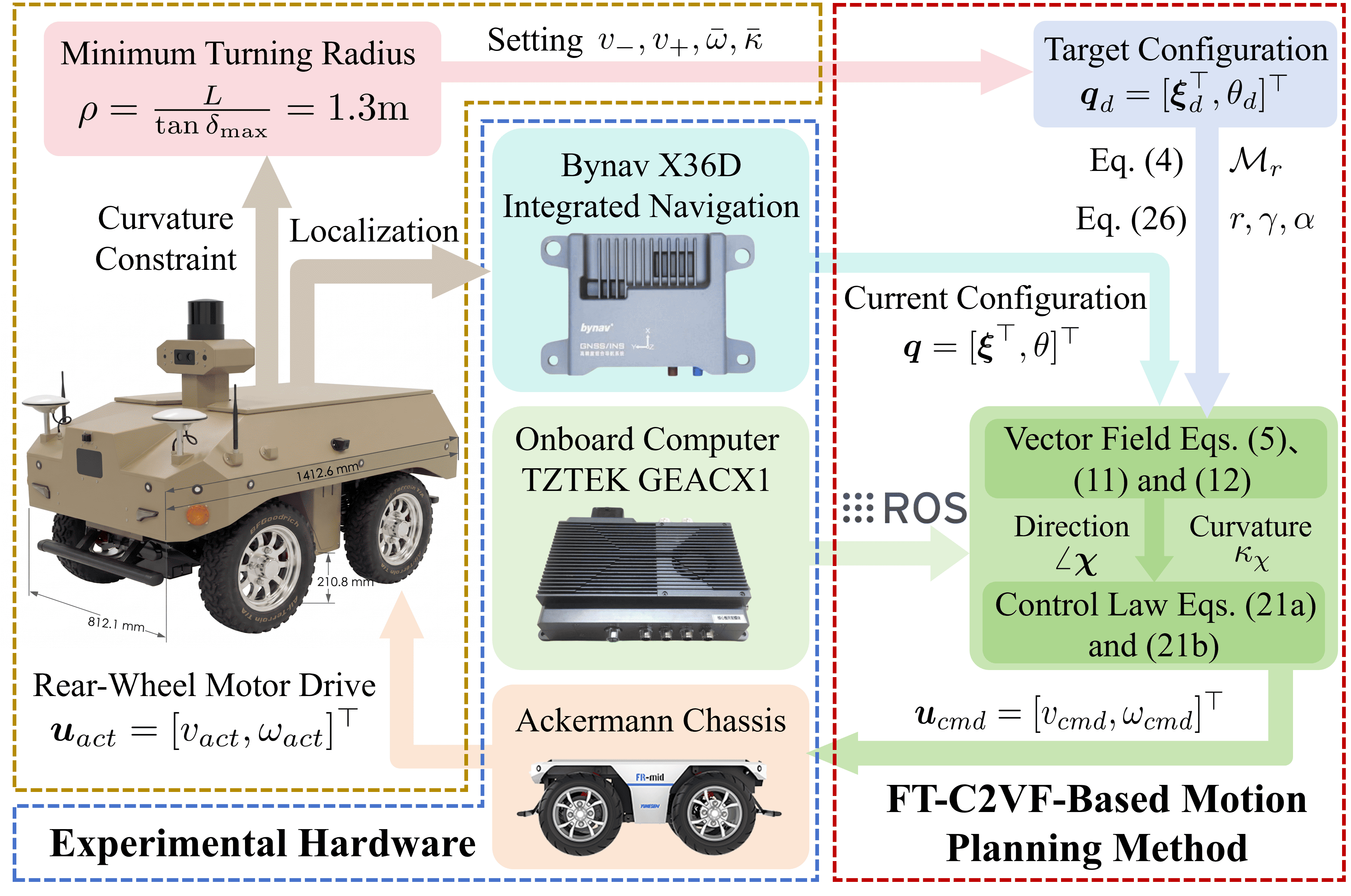}
\caption{Experimental architecture for the Ackermann-steered vehicle.}
\label{fig:exper01_set}
\end{figure}

\begin{figure*}[!t]
\centering

\hspace{-2.5mm}%
\noindent\makebox[\linewidth][c]{%
\begin{minipage}{0.34\linewidth}
    \centering
    \subfloat[]{
        \includegraphics[width=\linewidth]{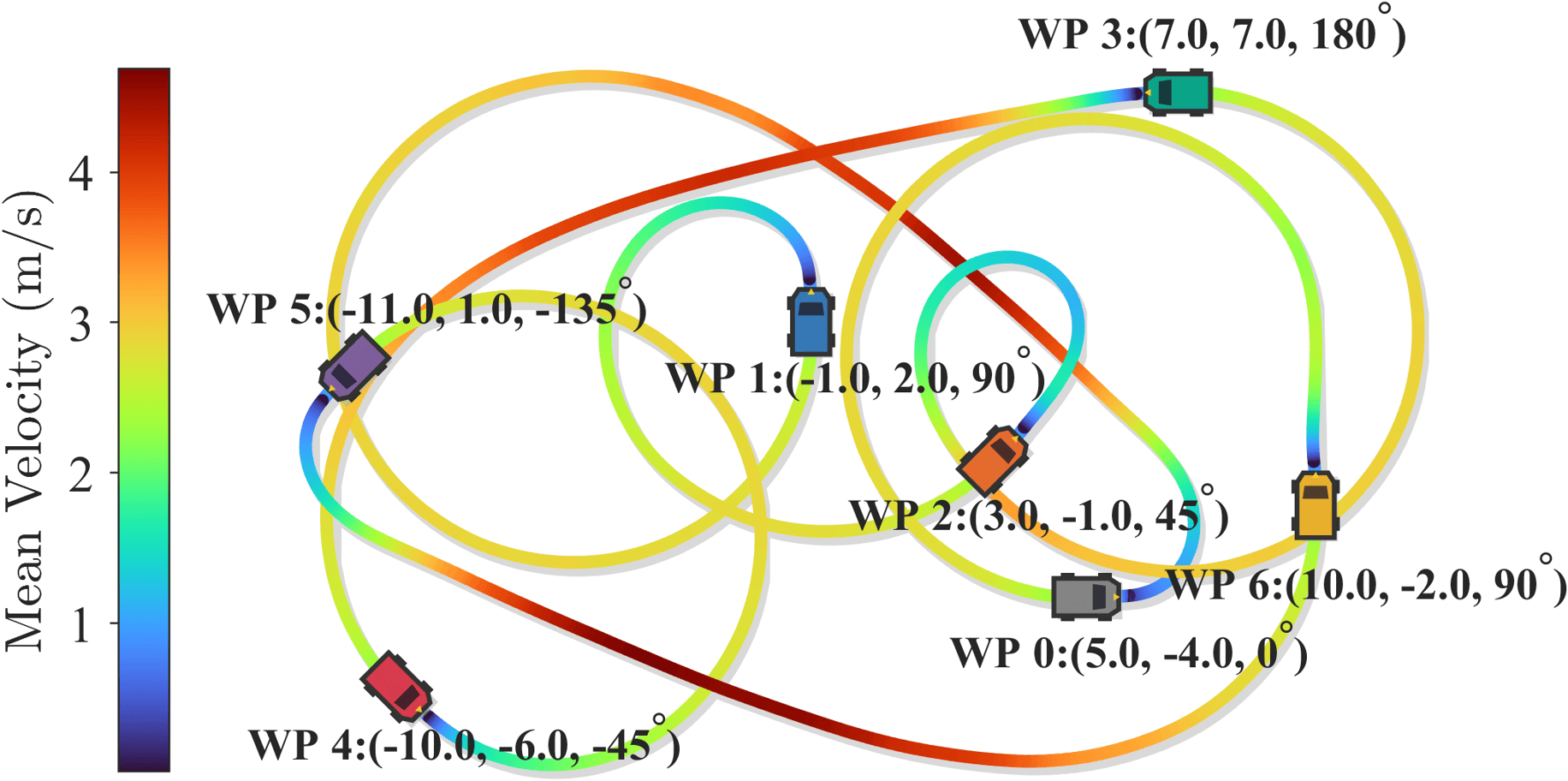}
        \label{fig5:sub_a}
    }%
    \\
    \vspace{-2mm}%
    \subfloat[]{
        \includegraphics[width=\linewidth]{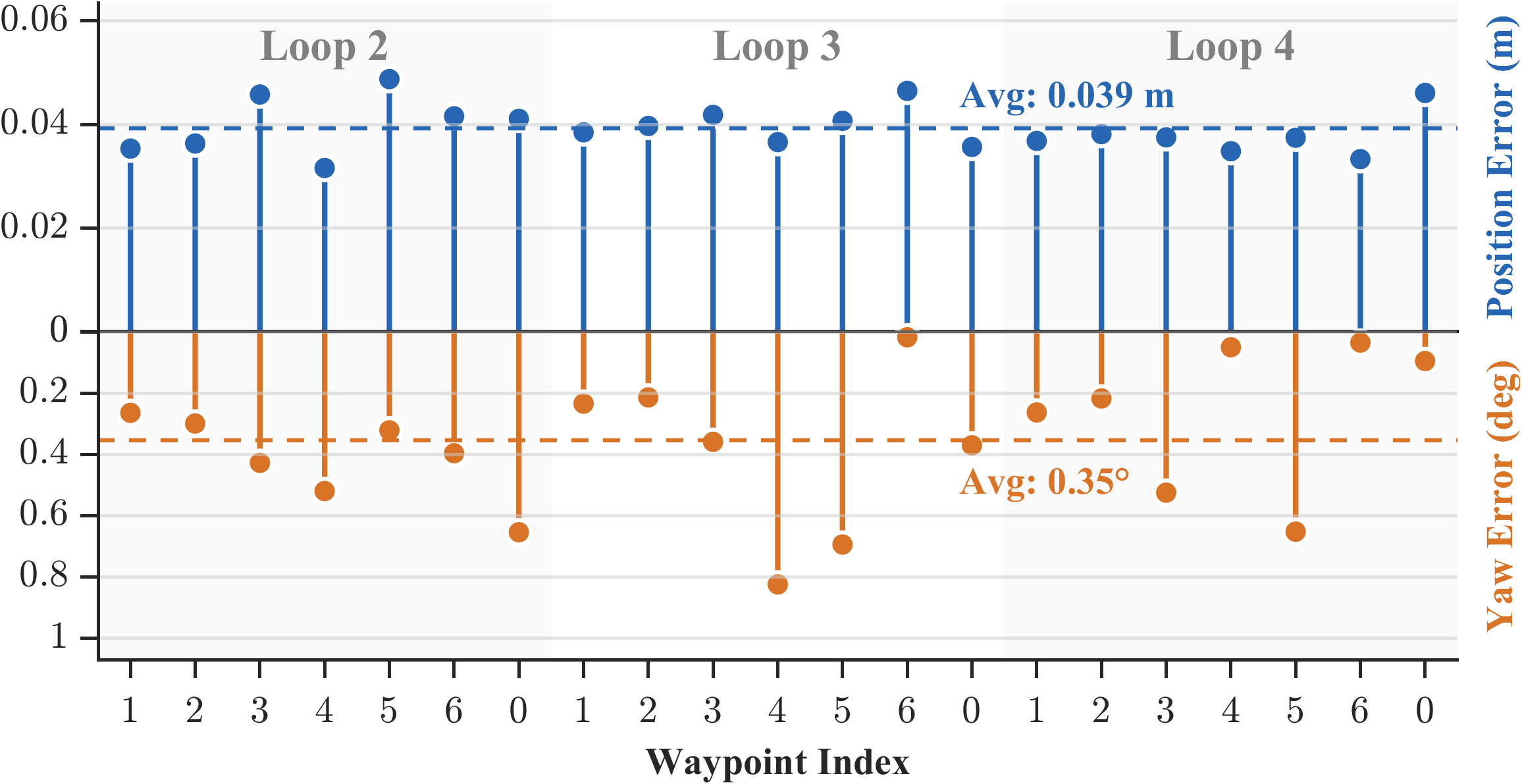}
        \label{fig5:sub_b}
    }
\end{minipage}
\begin{minipage}{0.64\linewidth}
    \centering
    \subfloat[]{
        \includegraphics[width=\linewidth]{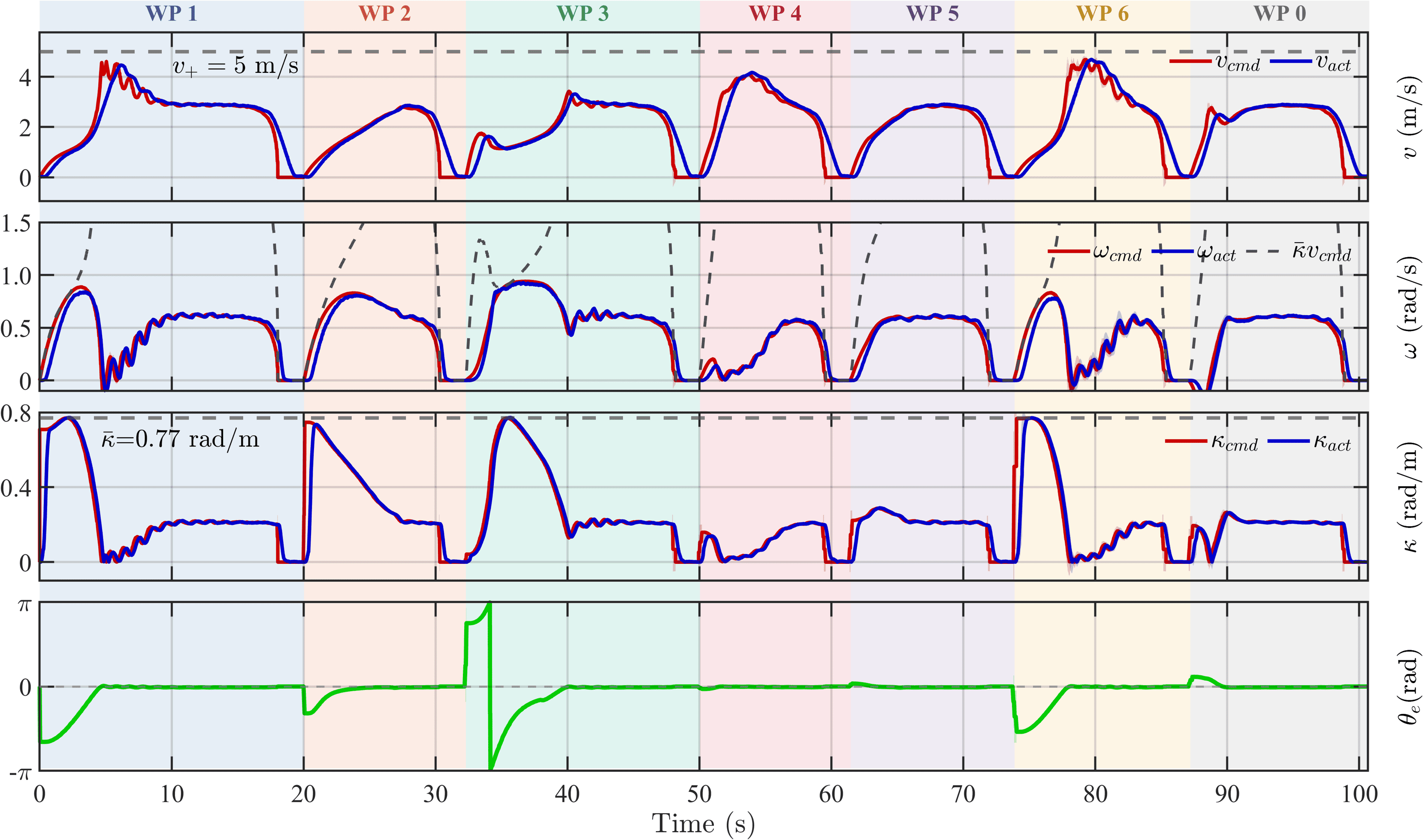}
        \label{fig5:sub_c}
    }
\end{minipage}
}
\\[-2mm]
\noindent\makebox[\linewidth][c]{
    \subfloat[]{
        \includegraphics[width=0.24\linewidth]{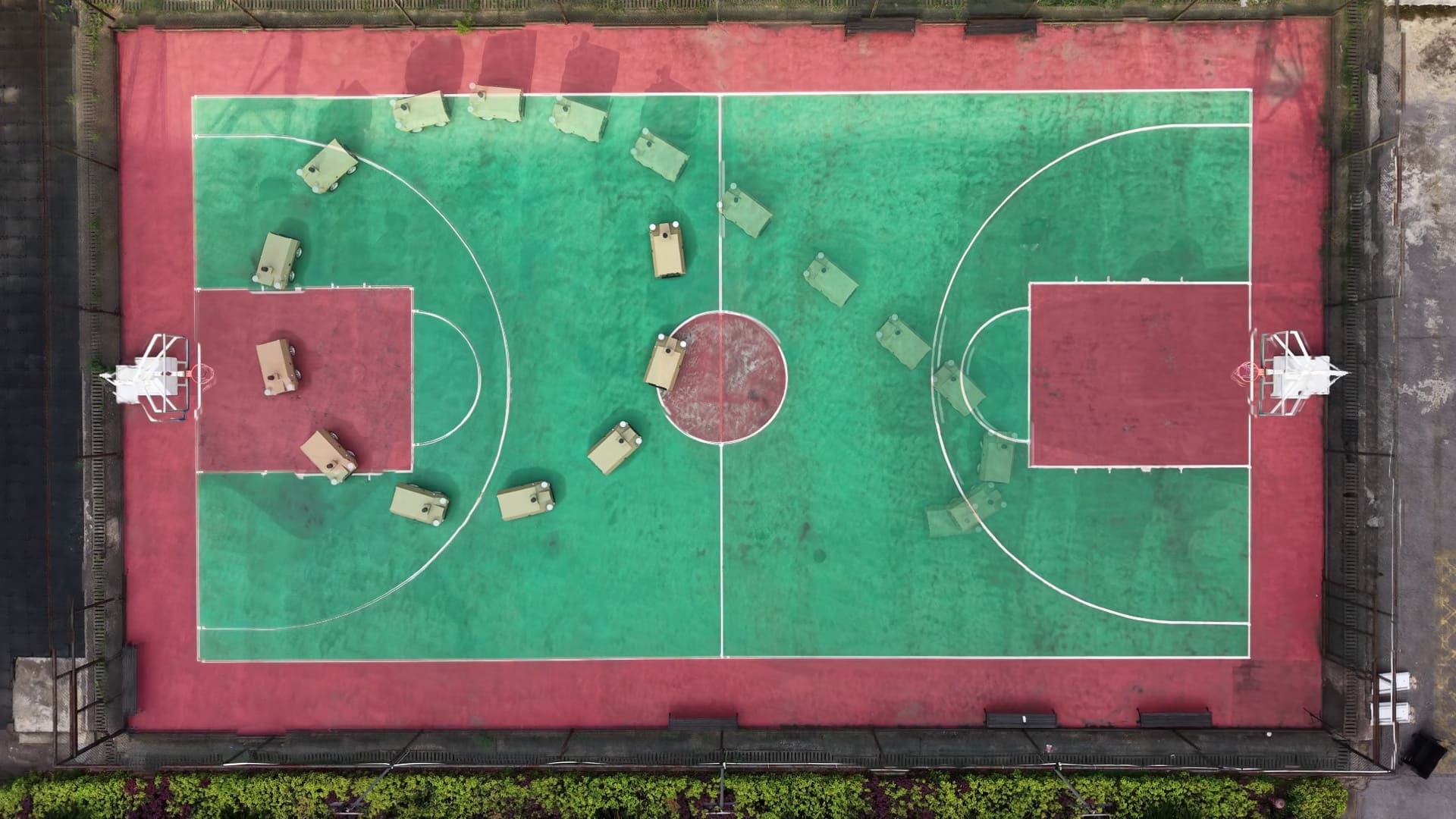}
        \label{fig5:sub_d}
    }%
    \hspace{-2mm}%
    \subfloat[]{
        \includegraphics[width=0.24\linewidth]{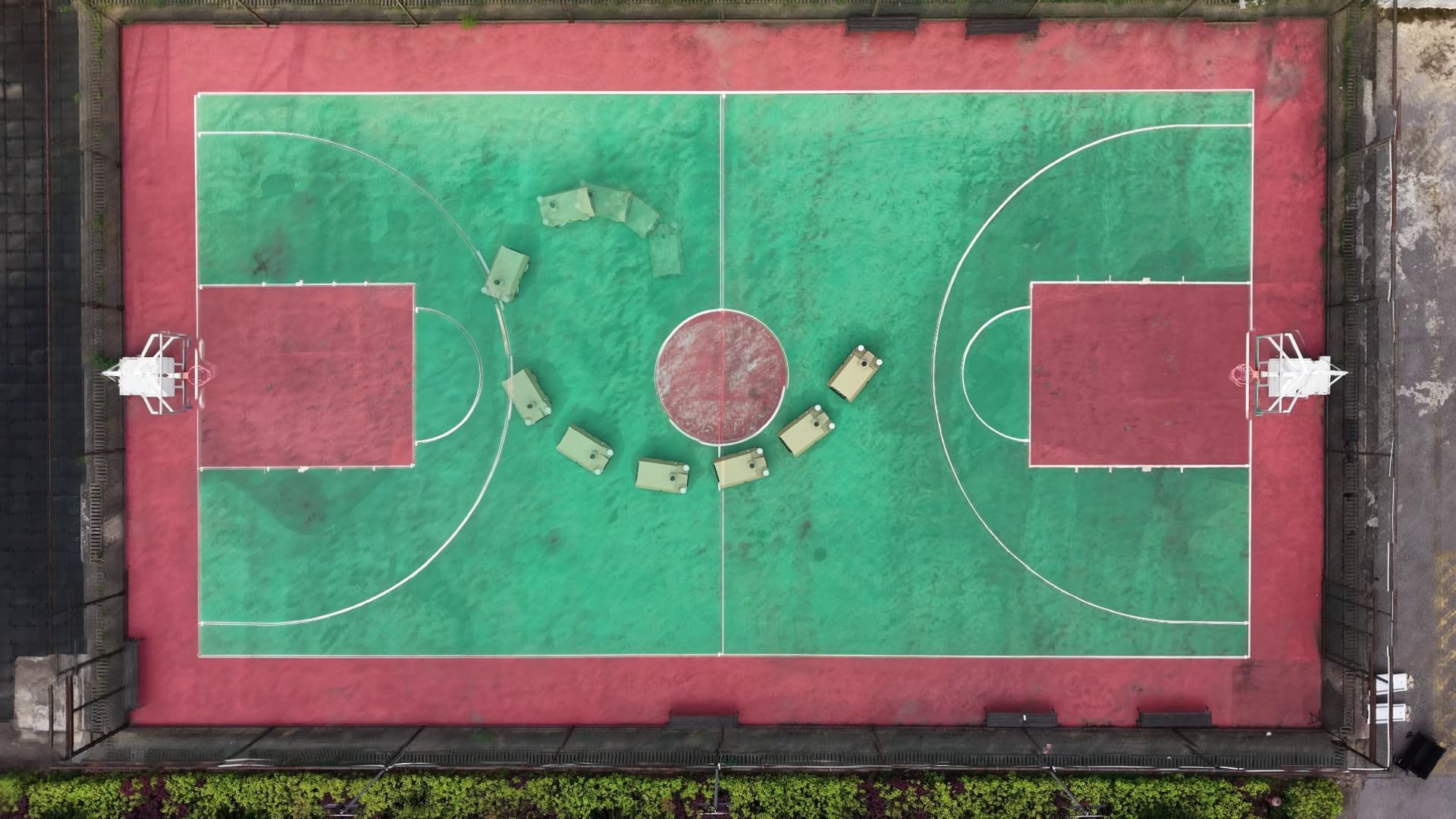}
        \label{fig5:sub_e}
    }%
    \hspace{-2mm}%
    \subfloat[]{
        \includegraphics[width=0.24\linewidth]{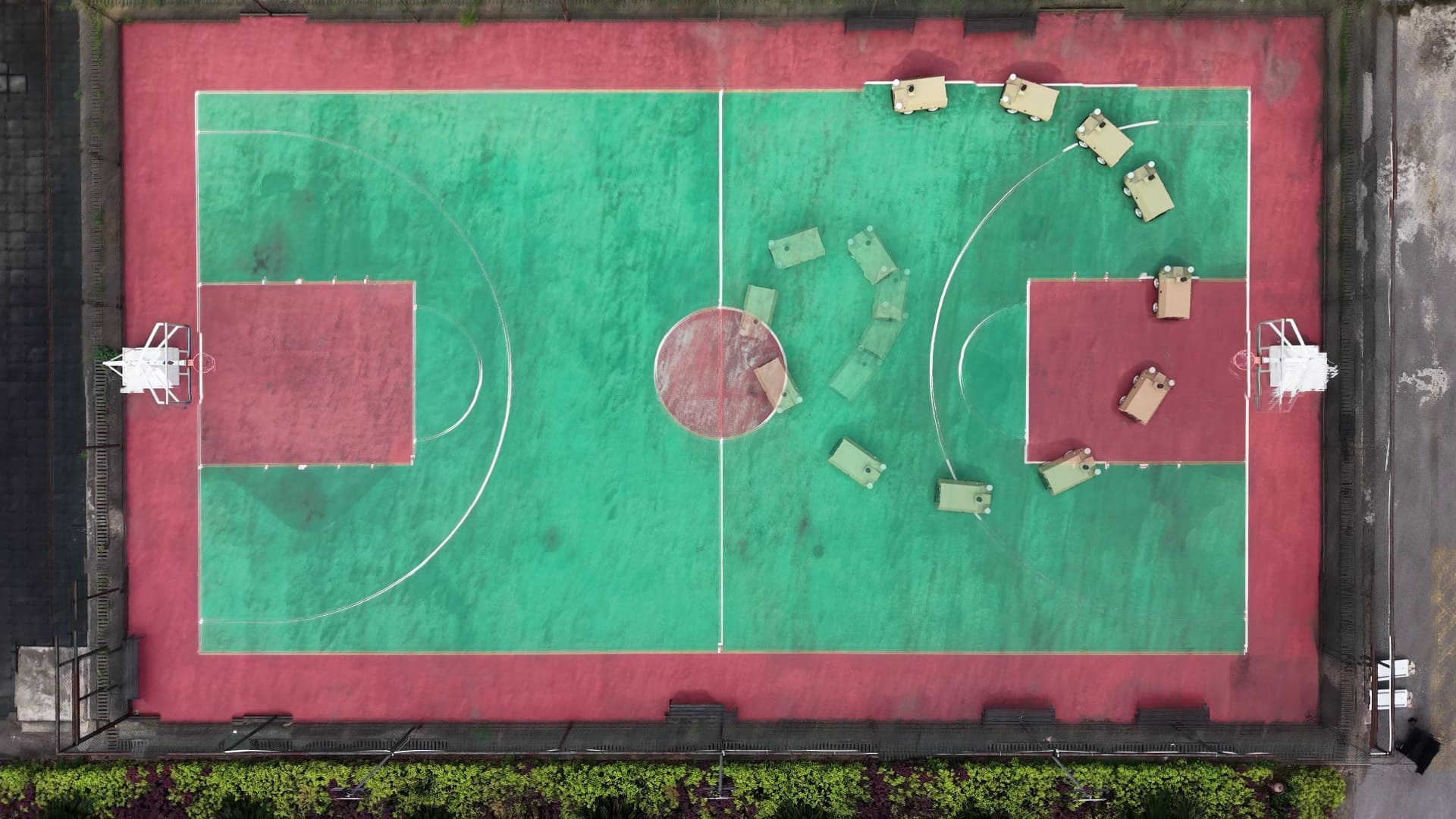}
        \label{fig5:sub_f}
    }%
    \hspace{-2mm}%
    \subfloat[]{
        \includegraphics[width=0.24\linewidth]{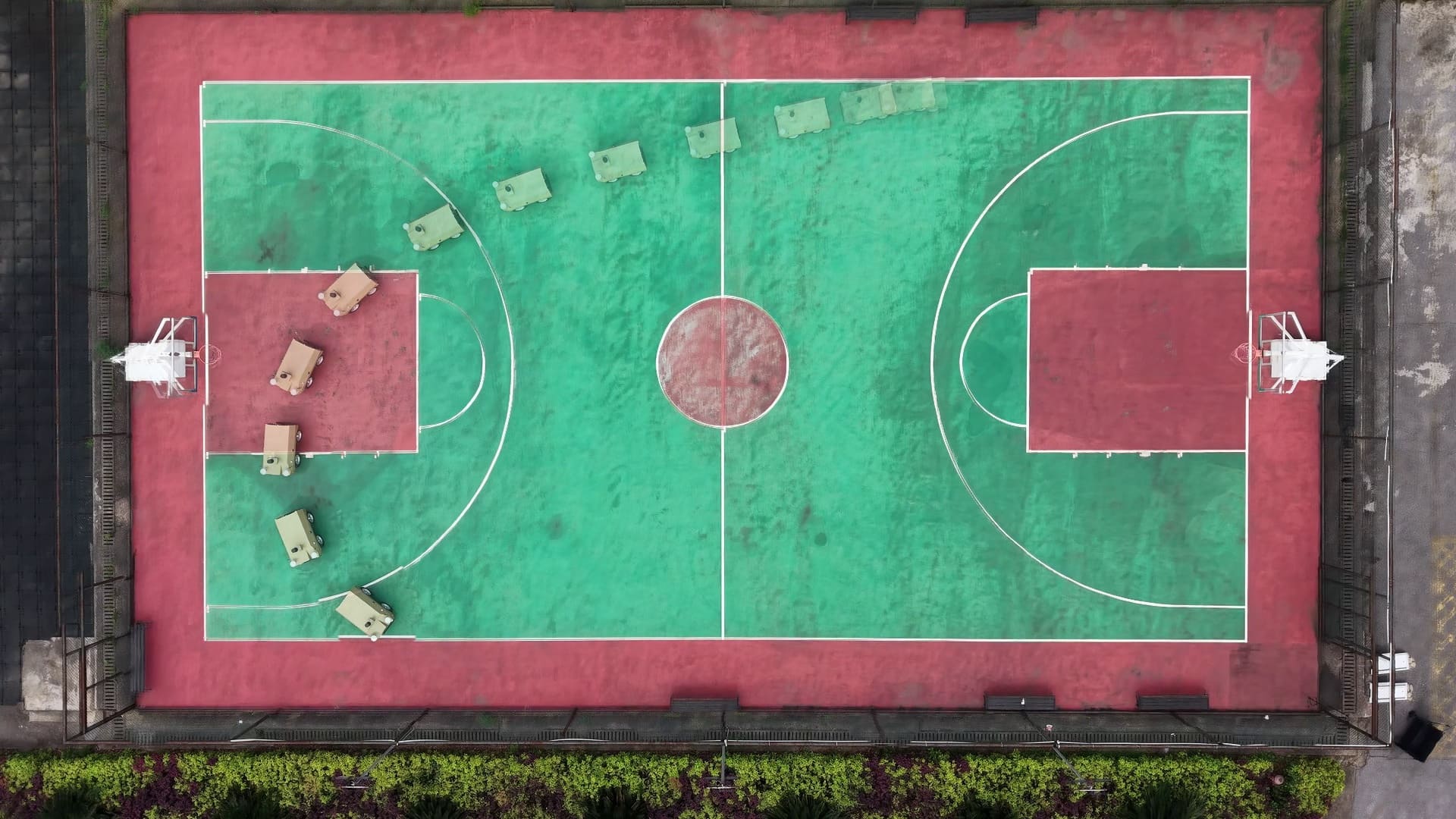}
        \label{fig5:sub_g}
    }%
} 
\\[-2mm] 
\noindent\makebox[\linewidth][c]{
    \subfloat[]{
        \includegraphics[width=0.24\linewidth]{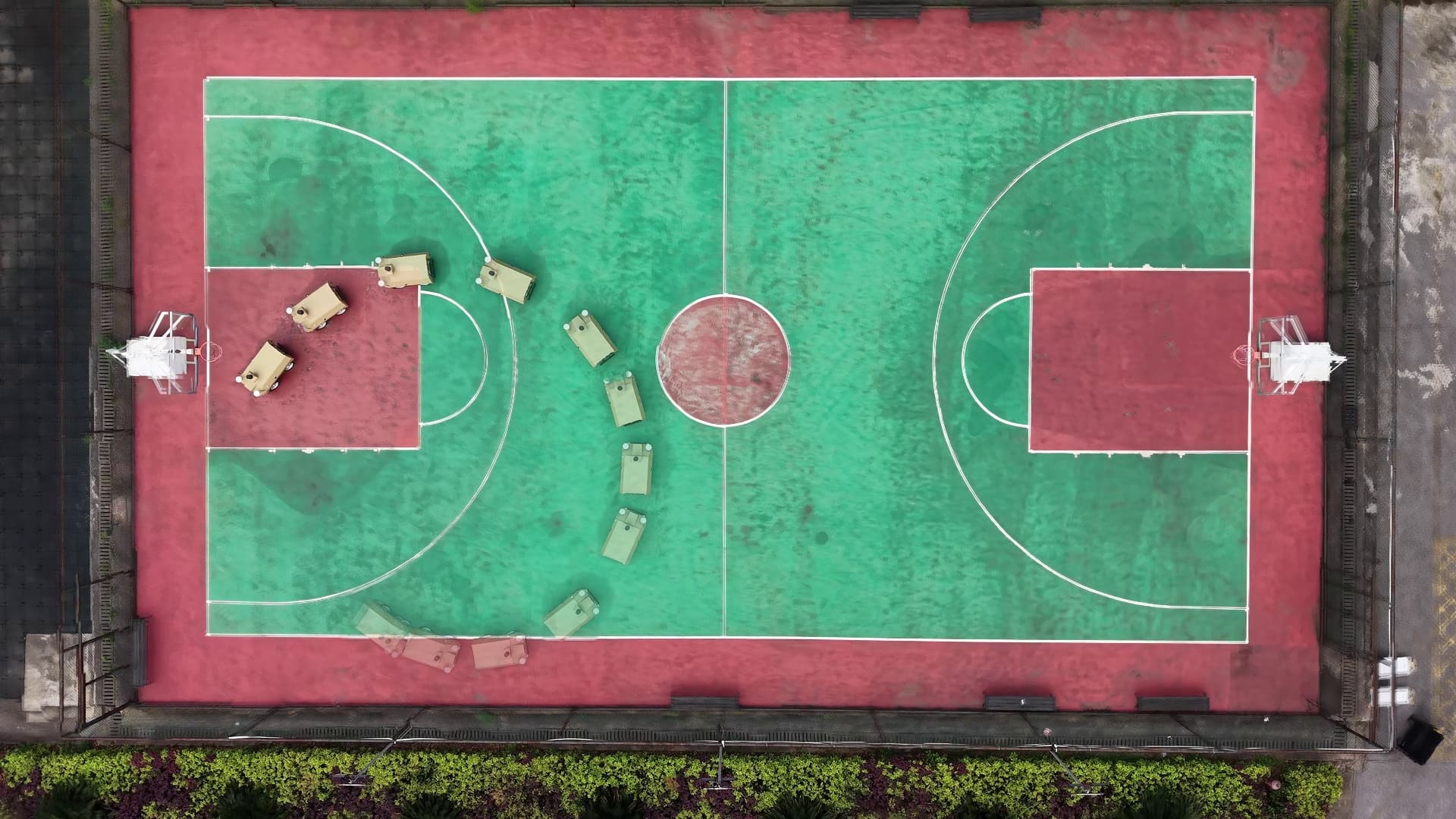}
        \label{fig5:sub_h}
    }%
    \hspace{-2mm}%
    \subfloat[]{
        \includegraphics[width=0.24\linewidth]{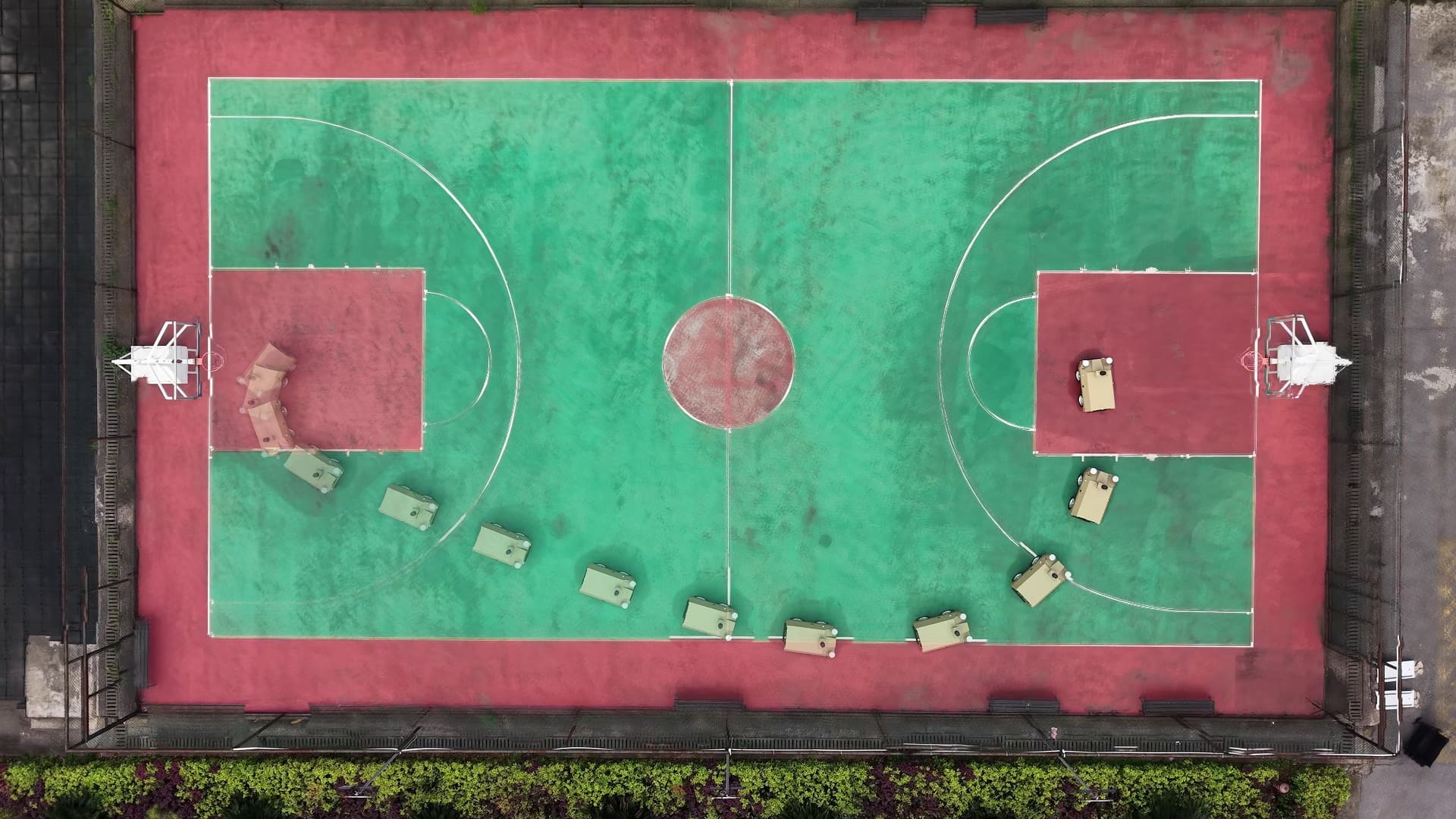}
        \label{fig5:sub_i}
    }%
    \hspace{-2mm}%
    \subfloat[]{
        \includegraphics[width=0.24\linewidth]{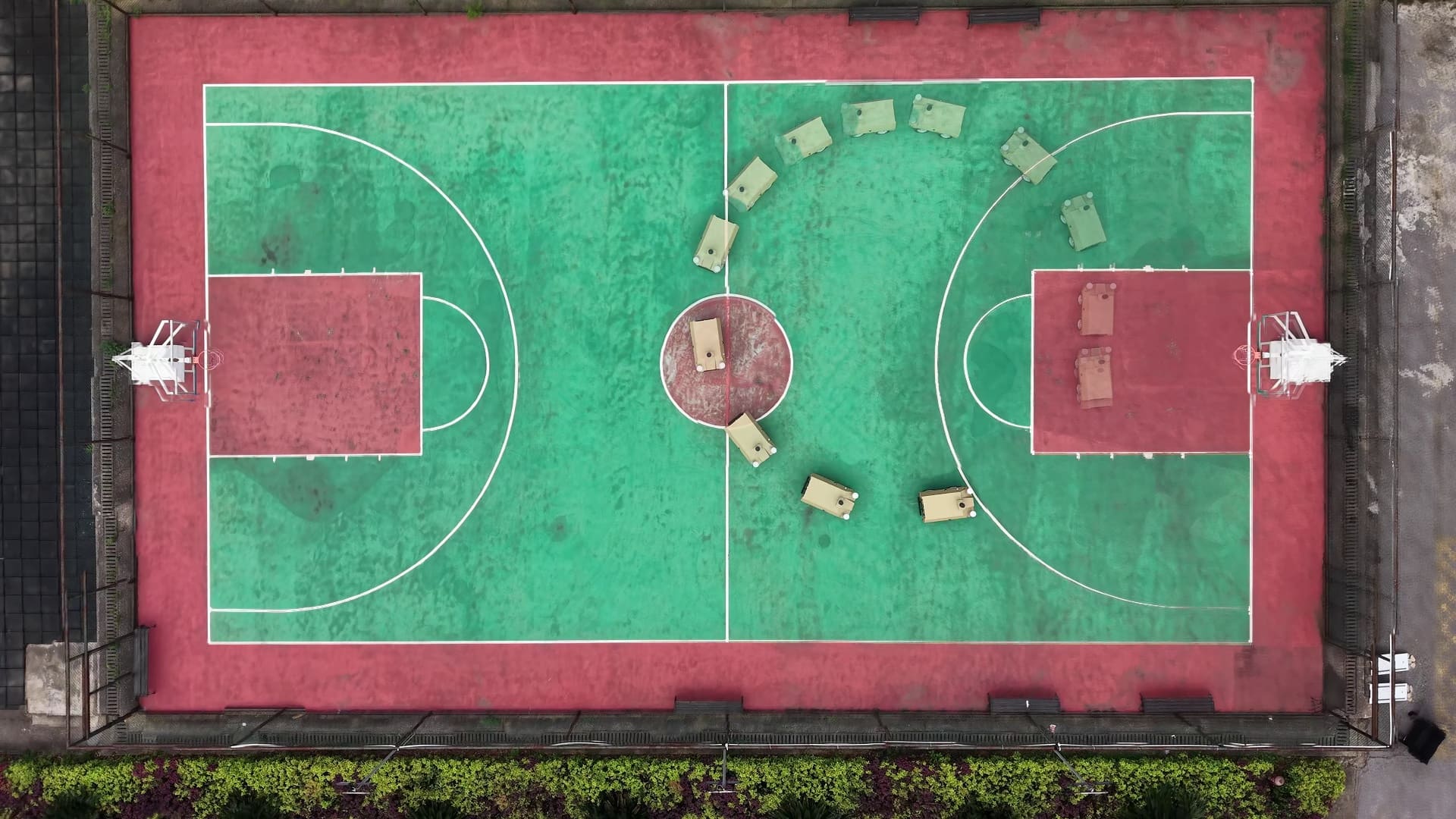}
        \label{fig5:sub_j}
    }%
    \hspace{-2mm}%
    \subfloat[]{
        \includegraphics[width=0.24\linewidth]{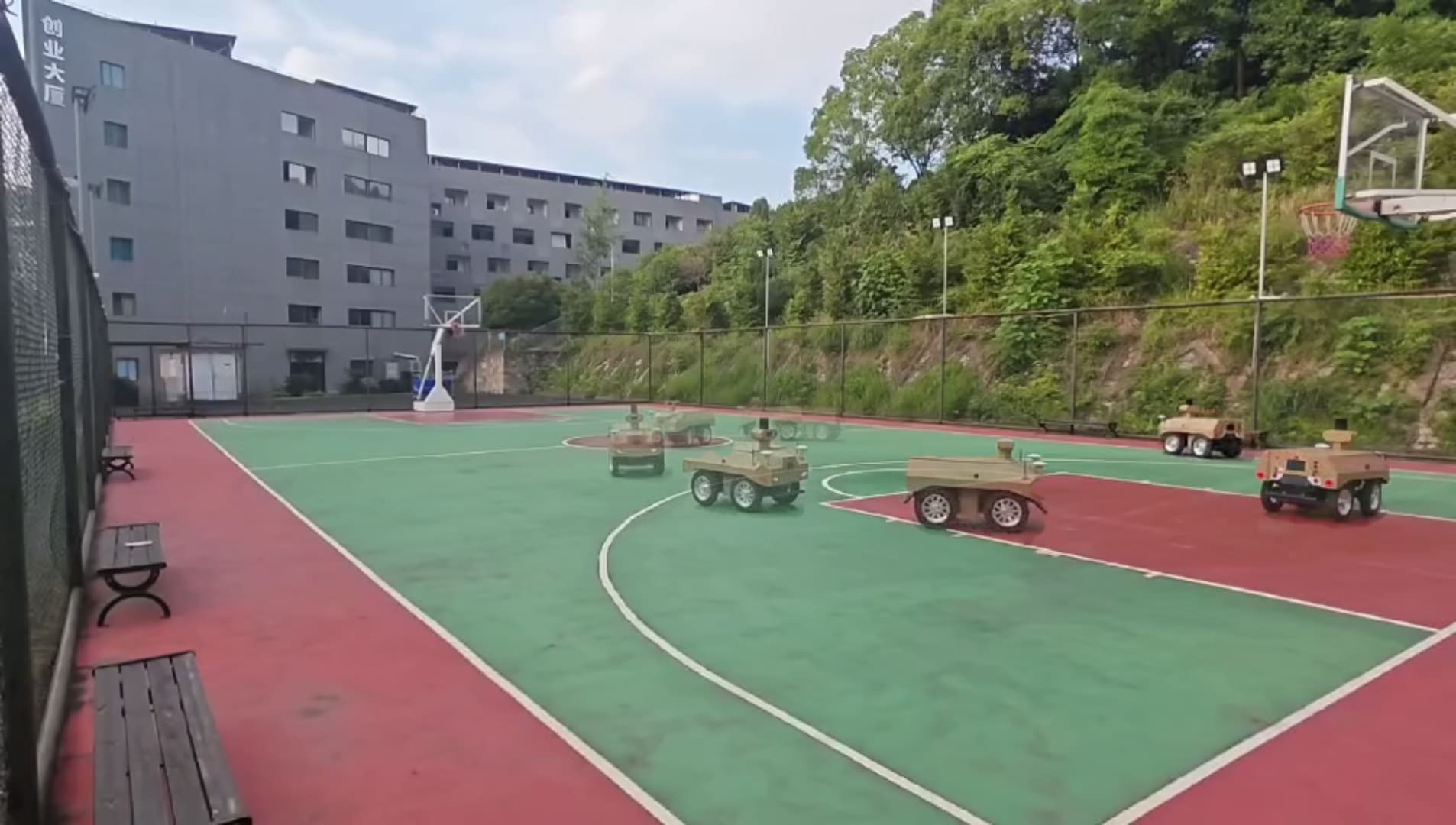}
        \label{fig5:sub_k}
    }%
} 
\caption{Results of the multi-waypoint experiment. (a) Spatial distribution of the waypoints $\boldsymbol{q}^i_d = [\boldsymbol{\xi}_d^{i\top}, \theta^i_d]^\top,i=0,1,\dots,6$ together with the average velocity heat map of the vehicle over loops 2–4. (b) Position and heading errors at each waypoint over three consecutive loops. (c) Average commanded (red) and measured (blue) control inputs. Shaded regions indicate the corresponding $1\sigma$ standard-deviation bands. Gray dashed lines denote the input limits, and the green curve denotes the heading error. (d)-(j) Vehicle trajectories for the segments WP~0$\rightarrow$WP~1, WP~1$\rightarrow$WP~2, WP~2$\rightarrow$WP~3,
WP~3$\rightarrow$WP~4, WP~4$\rightarrow$WP~5, WP~5$\rightarrow$WP~6, and WP~6$\rightarrow$WP~0, respectively. (k) Side view of a representative portion of the vehicle motion.}
\label{fig:exp01}
\end{figure*}

In the experiment, a rear-wheel-drive Ackermann-steered vehicle with a minimum turning radius of $\rho = 1.3$ m autonomously executed the FT-C2VF-based motion-planning algorithm. As illustrated in \cref{fig:exper01_set}, the hardware setup includes a Bynav X36D Integrated Navigation System for localization and heading estimation, alongside a TZTEK GEACX1 onboard computer. This computer publishes the forward speed $v$ and the front-wheel steering angle $\delta$ to the chassis via Robot Operating System (ROS). The steering angle is mapped to the angular velocity via $\omega = \frac{v}{L}\tan\delta$, where the wheelbase is $L = 0.66$ m. Our algorithm runs at a fixed frequency of 50 Hz, and all relevant data, including control commands $\boldsymbol{u}_{cmd} = [v_{cmd}, \omega_{cmd}]^\top$ and measured chassis states $\boldsymbol{u}_{act} = [v_{act}, \omega_{act}]^\top$, are recorded locally at 1000 Hz.

Since the tracking error in a practical discrete system cannot strictly converge to zero, local smoothing is applied to prevent the fractional-power term from inducing singularities at $\theta_e = 0$. We introduce a continuous function $h_\epsilon(\sin\theta_e)$ to replace the nonlinear feedback $\sin\theta_e$ in \eqref{command_kappa}, defined as
$$
h_\epsilon(\sin\theta_e) = \begin{cases} \lceil \sin\theta_e \rfloor^\alpha, & |\sin\theta_e| \ge \epsilon_d \\ 
\epsilon_d^{\alpha - 1} \sin\theta_e, & |\sin\theta_e| < \epsilon_d \end{cases}
$$
where $\epsilon_d = 0.08$. Thus, the commanded curvature becomes $\kappa_c = \kappa_\chi - k_\omega h_\epsilon(\sin\theta_e)$. To generate control inputs that initialize from zero, we modify the linear velocity control law to $v_{cmd} =  v_r (1 - e^{-0.45 t})(1 - e^{-k_v \Xi^\beta})$, where $t$ is the runtime. The remaining parameters are set as follows: $v_- = 0$ m/s, $v_+ = 5$ m/s, $\bar{\omega} = 1.5$ rad/s, $\bar{\kappa} = 0.7692$ rad/m, $\gamma = 0.5$, $r = 5.20$ m, $\alpha = 0.66$, $\beta = 0.45$, $k_v = 0.95$, and $\mu = 5.2$.

The experiment involves seven waypoints as shown in \cref{fig5:sub_a}, denoted by WP $i$ ($i = 0, 1,\dots, 6$), where the configuration of the $i$-th waypoint is defined as $\boldsymbol{q}^i_d = [\boldsymbol{\xi}_d^{i\top}, \theta^i_d]^\top$, with WP 0 being the initial configuration. We consider convergence to WP $i$ achieved if the vehicle state satisfies $\boldsymbol{q} \in \mathcal{Q}^i_\epsilon = \big\{ \boldsymbol{q} \in \mathcal{C}:\|\boldsymbol{\xi} - \boldsymbol{\xi}^i_d\| \le 0.05 \text{m}, |\theta - \theta^i_d| \le \frac{\pi}{180} \text{ rad} \big\}$. Upon reaching a waypoint, the vehicle pauses for $2~\mathrm{s}$ before proceeding to the next waypoint. The vehicle executes four consecutive loops, traversing the sequence WP 0 $\to$ WP 1 $\to \dots \to$ WP 6 $\to$ WP 0. Since the system initializes directly at WP 0, the tracking error at WP 0 during the first loop is zero. To accurately assess the position and heading errors across all waypoints, we discard the data from the first loop and restrict our analysis to the second through fourth loops.

The experimental results are presented in \cref{fig:exp01}. Due to measurement noise, motor response delay, and unmodeled dynamics, minor deviations exist between $\boldsymbol{u}_{cmd}$ and $\boldsymbol{u}_{act}$. As observed in \cref{fig5:sub_c}, $v_{cmd}$ exhibits continuous fluctuations at high speeds. This phenomenon occurs because the elevated speeds cause $|\kappa_c|$ to fluctuate frequently near zero, thereby inducing oscillations in \eqref{eq:vlaw}. Nevertheless, despite these real-world disturbances at a maximum speed approaching $5~\mathrm{m/s}$, the vehicle consistently converges to each configuration waypoint with bounded trajectory curvature and without control input saturation, as depicted in the snapshots across \crefrange{fig5:sub_d}{fig5:sub_k}. Furthermore, \cref{fig5:sub_b} shows that the average position and heading errors are merely $0.039~\mathrm{m}$ and $0.35^\circ$, respectively, which further validates the effectiveness and robustness of the proposed method in hardware implementations.

\begin{remark}
Additional experiments conducted with a maximum linear velocitiy of $v_+ = 4$ m/s show that reducing $v^{+}$ decreases the fluctuation amplitude of $v_{cmd}$. Furthermore, although the proposed algorithm theoretically eliminates chattering caused by vector field Jacobian singularities, the finite-time convergent controller may still induce high-frequency oscillations when implemented on discrete-time hardware. This observation motivates replacing the original nonlinear feedback term with a continuous approximation, which suppresses the oscillations caused by the unbounded feedback gain as  $\theta_e \to 0$ and thereby improves the practical implementability of the controller.
\hfill $\blacktriangleleft$
\end{remark}

\section{Conclusion and Future Work}\label{sec:7}


In this paper, we have proposed a new framework that jointly develops the FT-C2VF approach and a saturation-free control law to solve the finite-time generalized motion planning problem for curvature-constrained nonholonomic robots. The proposed method not only guarantees almost-global finite-time convergence of the robot to the desired configuration but also ensures the well-posedness and effectiveness of the closed-loop system at all times, since the control inputs remain within their prescribed bounds, so saturation does not compromise system stability. The effectiveness of the theoretical results has been validated through numerical simulations, comparative experiments, and outdoor experiments on an Ackermann-steered vehicle.

Our work opens up several directions for future research. 1) While $\mathcal{M}_r$ is defined as a circle in this study, future efforts could extend it to arbitrary manifolds, thereby facilitating curvature-constrained navigation for more complex tasks. 2) To address collision avoidance under curvature constraints, two potential approaches warrant exploration: (i) constructing composite vector fields (e.g., see \cite{9785912}); and (ii) incorporating high-order control barrier functions \cite{9516971}. 3) Although the current experiments are based on the planar Dubins-car kinematic model, extending the FT-C2VF planning framework to full six-degree-of-freedom models in 3D space remains a compelling subject for further investigation.

{
\appendices
\section{Proofs for Sections \ref{sec:2} and \ref{sec:3}}
\subsection{Proof of Lemma \ref{lem:curvature_vf}}\label{app:curvature_vf}
\begin{IEEEproof} Fix an arbitrary point
$\boldsymbol{\xi}_0\in\mathbb{R}^2\setminus\mathcal{W}$, and let
$\boldsymbol{\xi}(t)$ be the integral curve of $\boldsymbol{\chi}$ satisfying
$\boldsymbol{\xi}(0)=\boldsymbol{\xi}_0$, i.e.,
$\dot{\boldsymbol{\xi}}(t)=\boldsymbol{\chi}(\boldsymbol{\xi}(t))$.
Since $\boldsymbol{\xi}_0\notin\mathcal{W}$, the curve is regular at $t=0$. The curvature of a regular planar curve is given in
\cite[Ch.~13.4]{hass2022thomas}. Using the two-dimensional cross product
$\boldsymbol{a}\times\boldsymbol{b}:=
-\boldsymbol{a}^{\top}\boldsymbol{E}\boldsymbol{b}$, we write its signed curvature as
\begin{equation}\label{eq:curvature}
\kappa_{\boldsymbol{\xi}}(t)
=
\frac{\dot{\boldsymbol{\xi}}(t)\times\ddot{\boldsymbol{\xi}}(t)}
{\|\dot{\boldsymbol{\xi}}(t)\|^3}
=
-\frac{
\dot{\boldsymbol{\xi}}(t)^{\top}
\boldsymbol{E}
\ddot{\boldsymbol{\xi}}(t)}
{\|\dot{\boldsymbol{\xi}}(t)\|^3}.
\end{equation}
Moreover, by the chain rule,
$\ddot{\boldsymbol{\xi}}(0)
=
\boldsymbol{J}_{\chi}(\boldsymbol{\xi}_0)
\boldsymbol{\chi}(\boldsymbol{\xi}_0)$.
Substituting this identity and
$\dot{\boldsymbol{\xi}}(0)=\boldsymbol{\chi}(\boldsymbol{\xi}_0)$ into
\eqref{eq:curvature}, and using
$\boldsymbol{E}^{\top}=-\boldsymbol{E}$, gives
\eqref{eq:curvature_vf}. By Definition~\ref{def:curvature_vf},
$\kappa_{\chi}(\boldsymbol{\xi}_0)=\kappa_{\boldsymbol{\xi}}(0)$. Since
$\boldsymbol{\xi}_0$ is arbitrary, the proof is complete.
\end{IEEEproof}

\subsection{Proof of Lemma \ref{lem:001}}\label{app:001}
\begin{IEEEproof} 
First, we show that $\kappa_\chi(\eta)$ is well defined. For any two
points $\boldsymbol\xi_1,\boldsymbol\xi_2\in\mathbb R^2\setminus
\mathcal W$ with $\|\boldsymbol\xi_1\|=\|\boldsymbol\xi_2\|$, there
exists $\boldsymbol R\in SO(2)$ such that
$\boldsymbol\xi_2=\boldsymbol R\boldsymbol\xi_1$. From
\eqref{eq:FT-C2VF}, together with
$\eta(\boldsymbol R\boldsymbol\xi)=\eta(\boldsymbol\xi)$,
$\nabla\phi(\boldsymbol R\boldsymbol\xi)=
\boldsymbol R\nabla\phi(\boldsymbol\xi)$, and
$\boldsymbol E\boldsymbol R=\boldsymbol R\boldsymbol E$, one has
$\boldsymbol\chi(\boldsymbol R\boldsymbol\xi)
=
\boldsymbol R\boldsymbol\chi(\boldsymbol\xi)$.
Consequently, wherever $\boldsymbol J_\chi$ exists, 
$\boldsymbol J_\chi(\boldsymbol R\boldsymbol\xi)
=
\boldsymbol R\boldsymbol J_\chi(\boldsymbol\xi)\boldsymbol R^\top$.
Substituting these relations into \eqref{eq:curvature_vf} and using
$\boldsymbol R^\top\boldsymbol R=\boldsymbol I$ and
$\boldsymbol E\boldsymbol R=\boldsymbol R\boldsymbol E$ yields
$\kappa_\chi(\boldsymbol\xi_2)
=
\kappa_\chi(\boldsymbol R\boldsymbol\xi_1)
=
\kappa_\chi(\boldsymbol\xi_1)$.
Therefore, the signed curvature is constant on each circle centered at
the origin and depends only on $\eta=\|\boldsymbol\xi\|/r$. Hence
$\kappa_\chi(\eta)$ is well defined on $(0,+\infty)$.

Next, we derive its explicit expression. The integral curve of \eqref{eq:FT-C2VF} satisfies $\dot{\boldsymbol{\xi}} = \boldsymbol{\chi}(\boldsymbol{\xi})$; thus,
\begin{equation}\label{eq:dotxi}
\dot{\boldsymbol{\xi}} = 2(\psi_{\tau} \boldsymbol{E} + \psi_{\nu} \boldsymbol{I})\boldsymbol{\xi}.
\end{equation}
By using $\eta = \|\boldsymbol{\xi}\|/r$, $\boldsymbol{\xi}^\top \boldsymbol{E} \boldsymbol{\xi} = 0$ and \eqref{eq:equal_psi}, it can be deduced that
\begin{equation}\label{eq:fenmu}
\|\dot{\boldsymbol{\xi}}\| = 2\|\boldsymbol{\xi}\|\sqrt{ \psi_{\tau}^2+\psi_{\nu}^2} = 2r\eta,
\end{equation}
and the time derivative of $\eta$ is
\begin{equation}\label{eq:doteta}
\dot{\eta} = \frac{\boldsymbol{\xi}^\top \dot{\boldsymbol{\xi}}}{r\|\boldsymbol{\xi}\|} = \frac{2\boldsymbol{\xi}^\top( \psi_{\tau} \boldsymbol{E}+\psi_{\nu} \boldsymbol{I})\boldsymbol{\xi}}{r^2\eta} =2\eta \psi_{\nu}.
\end{equation}
Taking the time derivative of \eqref{eq:dotxi} yields
\begin{equation}\label{eq:ddot}
\ddot{\boldsymbol{\xi}} = 2( \dot{\psi}_{\tau} \boldsymbol{E}+\dot{\psi}_{\nu} \boldsymbol{I})\boldsymbol{\xi} + 2( \psi_{\tau} \boldsymbol{E}+\psi_{\nu} \boldsymbol{I} )\dot{\boldsymbol{\xi}}.
\end{equation}
Differentiating \eqref{eq:equal_psi} with respect to $\eta$ and time, respectively, produces the following equations:
\begin{subequations}
\begin{align}
    \psi_{\tau} \psi_{\tau}' + \psi_{\nu} \psi_{\nu}' &= 0, \label{eq:equal_eta} \\ 
    {\psi}_{\tau} \dot{\psi}_{\tau} + \psi_{\nu} \dot{\psi}_{\nu} &= 0. \label{eq:equal_time}
\end{align}
\end{subequations}
Hence,
\begin{equation}\label{eq:kn}
{\psi}_{\nu} \dot{\psi}_{\tau} - {\psi}_{\tau} \dot{\psi}_{\nu}
= ({\psi}_{\nu} \psi_{\tau}' - \psi_{\tau} \psi_{\nu}')\dot{\eta}
\overset{\eqref{eq:equal_eta}}{=} \frac{ \psi_{\tau}^2+\psi_{\nu}^2 }{\psi_{\nu}} \psi_{\tau}' \dot{\eta}
\overset{\eqref{eq:doteta}}{=} 2\eta \psi_{\tau}'.
\end{equation}
The numerator of \eqref{eq:curvature} can be simplified as
\begin{align}\label{eq:numerator}
&\dot{\boldsymbol{\xi}}^\top \boldsymbol{E} \ddot{\boldsymbol{\xi}} \notag\\&\overset{\eqref{eq:ddot}}{=} 2\dot{\boldsymbol{\xi}}^\top \boldsymbol{E} ( \dot{\psi}_{\tau} \boldsymbol{E}+\dot{\psi}_{\nu} \boldsymbol{I})\boldsymbol{\xi} + 2\dot{\boldsymbol{\xi}}^\top \boldsymbol{E} ( \psi_{\tau} \boldsymbol{E}+\psi_{\nu} \boldsymbol{I} )\dot{\boldsymbol{\xi}} \notag\\
&\overset{\eqref{eq:dotxi}}{=}  4\boldsymbol{\xi}^\top \big[  ( {\psi}_{\tau} \dot{\psi}_{\tau}+{\psi}_{\nu} \dot{\psi}_{\nu} )\boldsymbol{E}+(\psi_{\tau} \dot{\psi}_{\nu} - {\psi}_{\nu} \dot{\psi}_{\tau})\boldsymbol{I} \big]\boldsymbol{\xi} - 2{\psi}_{\tau} \|\dot{\boldsymbol{\xi}}\|^2 \notag\\
&\overset{\eqref{eq:equal_time}}{=} -4(\psi_{\nu} \dot{\psi}_{\tau} - {\psi}_{\tau} \dot{\psi}_{\nu})\|\boldsymbol{\xi}\|^2 - 8r^2\eta^2 {\psi}_{\tau}\notag\\
&\overset{\eqref{eq:kn}}{=} -8r^2\eta^2 (\psi_{\tau} + \eta \psi_{\tau}').
\end{align}
Substituting \eqref{eq:fenmu} and \eqref{eq:numerator} into \eqref{eq:curvature} results in \eqref{eq:lemma1}.
\end{IEEEproof}

\subsection{Proof of Proposition \ref{cor:001}}\label{app:Corollary1}
\begin{IEEEproof}{ Equation \eqref{eq:kappa_prep} follows directly by replacing $\boldsymbol{\chi}$ with $\boldsymbol{E}\boldsymbol{\chi}$ in the proof of Lemma \ref{lem:001}. Furthermore,
\begin{align}\label{eq:hengdengshi}
\kappa_\chi\psi_\tau+\kappa_\perp\psi_\nu
&= \frac{1}{r} \left( \frac{\psi_{\tau}}{\eta} + \psi_{\tau}' \right)\psi_\tau 
+ \frac{1}{r} \left( \frac{\psi_{\nu}}{\eta} + \psi_{\nu}' \right)\psi_\nu \notag \\
&= \frac{1}{r\eta}(\psi_{\tau}^2+\psi_{\nu}^2) 
+ \frac{1}{r}(\psi_{\tau}\psi_{\tau}' + \psi_{\nu}\psi_{\nu}')\overset{\eqref{eq:equal_eta}}{=} \frac{1}{r\eta}.\notag
\end{align}
}\end{IEEEproof}

\subsection{Proof of Proposition \ref{propos:001}} \label{app:002}
\begin{IEEEproof}{ The continuity of the curvature is guaranteed by the continuous differentiability of $\psi_\tau(\eta)$ in conjunction with \eqref{eq:lemma1}. The first derivative of $\psi_\tau(\eta)$ is given by
\begin{equation}\label{eq:psi_prime}
\psi_{\tau}'(\eta) = 
\begin{cases}
(1+\gamma)(1 - \eta)^{\gamma}, & \eta \in [0,1] \\ 
-(1+\gamma)\eta^{-2}(1 - \eta^{-1})^{\gamma}, & \eta \in (1,+\infty)
\end{cases}
\end{equation}
Substituting \eqref{eq:kt_def} and \eqref{eq:psi_prime} into \eqref{eq:lemma1} results in \eqref{eq:kappa_piecewise}. Evaluating the limit of \eqref{eq:kappa_piecewise} as $\eta \to 0^+$ or $\eta \to +\infty$ establishes the bounds presented in \eqref{eq:kappa_max_min}.

To establish the strict decrease of $\kappa_\chi(\eta)$, it is sufficient to show that $\kappa_\chi'(\eta) < 0$ for $\eta \in (0, +\infty)$. Differentiating \eqref{eq:lemma1} with respect to $\eta$  gives
\begin{equation}\label{eq:kappa_prime}
\kappa_\chi'(\eta) = \frac{1}{r} \left( \frac{\eta \psi_{\tau}'(\eta) - \psi_{\tau}(\eta)}{\eta^2} + \psi_{\tau}''(\eta) \right).
\end{equation}

We analyze this across two sub-intervals.

\textit{Case 1 ($\eta \in (0, 1)$):} Differentiating \eqref{eq:psi_prime} produces $\psi_\tau''(\eta) = -\gamma(1+\gamma)(1-\eta)^{\gamma-1}$. Since $\gamma > 0$, we have $\psi_\tau''(\eta) < 0$, indicating that $\psi_\tau'(\eta)$  strictly decreases on $(0, 1)$. Given $\psi_\tau(0) = 0$, applying the integral properties of strictly decreasing functions yields $\psi_\tau(\eta) = \int_0^\eta \psi_\tau'(s) ds > \int_0^\eta \psi_\tau'(\eta) ds = \eta \psi_\tau'(\eta)$. Consequently, $\eta \psi_\tau'(\eta) - \psi_\tau(\eta) < 0$. Because both terms inside the parenthesis of \eqref{eq:kappa_prime} are negative, $\kappa_\chi'(\eta) < 0$ holds for $\eta \in (0, 1)$.

\textit{Case 2 ($\eta \in (1, +\infty)$):} To simplify the analysis, let $\varepsilon = \eta^{-1} \in (0, 1)$. By defining a scaled auxiliary function $h(\varepsilon) = r \kappa_\chi(\varepsilon^{-1})$, proving that $\kappa_\chi(\eta)$ strictly decreases with respect to $\eta$ is equivalent to demonstrating that $h(\varepsilon)$ strictly decreases with respect to $\varepsilon$ (i.e., $h'(\varepsilon) > 0$). The first derivative of $h(\varepsilon)$ is computed as
\begin{equation*}
h'(\varepsilon) = 1 - (1-\varepsilon)^{1+\gamma} - \varepsilon(1+\gamma)(1-\varepsilon)^\gamma + \gamma(1+\gamma)\varepsilon^2(1-\varepsilon)^{\gamma-1}.\end{equation*}
Observe that $h'(0) = 0$. The second derivative can be factored into $h''(\varepsilon) = \gamma(1+\gamma)\varepsilon(1-\varepsilon)^{\gamma-2} \big[ 3 - (\gamma+2)\varepsilon \big].$ The sign of $h''(\varepsilon)$ for $\varepsilon \in (0,1)$ depends on the root $\varepsilon_0 = \frac{3}{\gamma+2}$. If $\gamma \le 1$, we have $\varepsilon_0 \ge 1$, ensuring $h''(\varepsilon) > 0$ for $\varepsilon \in (0, 1)$. Since $h'(0) = 0$, $h'(\varepsilon)$ remains positive on $(0, 1)$. Alternatively, if $\gamma > 1$, the root falls within $\varepsilon_0 \in (0, 1)$. Under this condition, $h''(\varepsilon) > 0$ on $(0, \varepsilon_0)$ and $h''(\varepsilon) < 0$ on $(\varepsilon_0, 1)$. This implies that $h'(\varepsilon)$ increases to a maximum and subsequently decreases. However, evaluating the limit as $\varepsilon \to 1^-$ reveals $\lim_{\varepsilon \to 1^-} h'(\varepsilon) = 1 > 0$. Since $h'(\varepsilon)$ starts at $0$, increases, and ends at a positive value of $1$ without dropping below zero, it follows that $h'(\varepsilon) > 0$ for all $\varepsilon \in (0, 1)$. Therefore, $\kappa_\chi'(\eta) = \frac{1}{r} h'(\varepsilon) \frac{d\varepsilon}{d\eta} = -\frac{1}{r\eta^2} h'(\varepsilon) < 0$ for $\eta \in (1, +\infty)$. 

Because $\kappa_\chi'(\eta) < 0$ across both sub-intervals and $\kappa_\chi(\eta)$ is continuous at $\eta = 1$, the curvature is strictly decreasing over $\eta \in (0, +\infty)$, as illustrated in \cref{fig:kappa}. 
}\end{IEEEproof}

\begin{figure}[!t]
\centering
\includegraphics[width=3in]{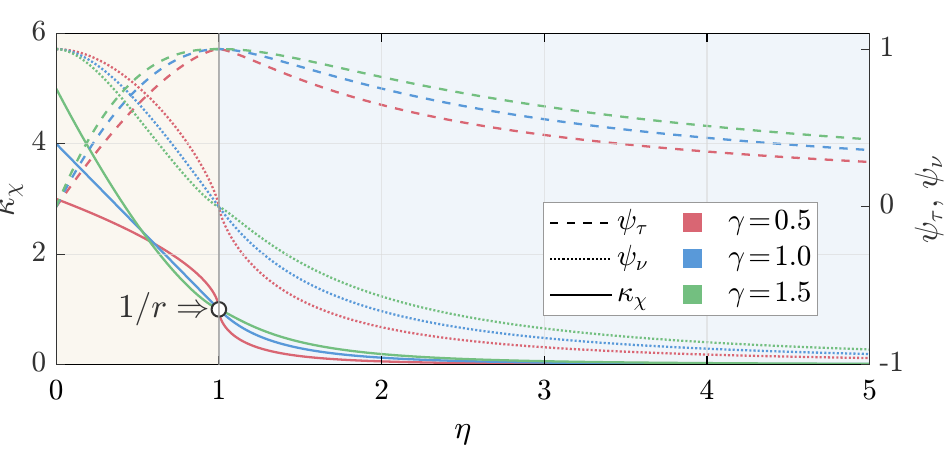}
\caption{Profiles of  $\kappa_\chi$ and $\psi_\tau, \psi_\nu$ for $\gamma \in \{0.5, 1, 1.5\}$.}
\label{fig:kappa}
\end{figure}

\subsection{Existence and Uniqueness of Forward Solutions to \eqref{eq:ODE}}
\label{app:003}
\begin{IEEEproof} First, we analyze the smoothness of the vector field and the global existence of solutions to \eqref{eq:ODE}. The smoothness of $\boldsymbol{\chi}$ depends on the nonlinear gain functions. We have $\eta(\boldsymbol{\xi})\neq 1$ and $\psi_{\tau}, \psi_{\nu} \in C^1$ for any point $\boldsymbol{\xi}\in \mathbb{R}^2 \setminus (\mathcal{W}\cup\mathcal{M}_r)$. By the chain rule for compositions, $\boldsymbol{\chi}$ is of class $C^1$ in this region. For any point $\boldsymbol{\xi}_0 \in \mathcal{M}_r$, it follows that $ \lim_{\boldsymbol{\xi}_1 \to \boldsymbol{\xi}_0,\|\boldsymbol{\xi}_1\| \in \mathcal{S}_1} \boldsymbol{\chi}(\boldsymbol{\xi}_1)=\lim_{\boldsymbol{\xi}_2 \to \boldsymbol{\xi}_0,\|\boldsymbol{\xi}_2\| \in \mathcal{S}_2} \boldsymbol{\chi}(\boldsymbol{\xi}_2)=2\boldsymbol{E}\boldsymbol{\xi}.$
    Thus, $\boldsymbol{\chi}$ is of class $C^0$ on $\mathbb{R}^2 \setminus\mathcal{W}$. By the Peano existence theorem \cite[Thm.~2.19]{teschl2012ordinary}, \eqref{eq:ODE} admits at least one forward solution for any initial point $\boldsymbol{\xi}_0 \in \mathbb{R}^2 \setminus \mathcal{W}$. Moreover, since $\|\boldsymbol{\chi}(\boldsymbol{\xi})\| = 2\|\boldsymbol{\xi}\|$ from \eqref{eq:daxiao}, the vector field satisfies the global linear growth condition, and thus the solution has no finite escape time \cite[Thm.~2.17]{teschl2012ordinary}. This guarantees the global existence of the forward solution for all $t \ge 0$. Next, we prove the uniqueness of the forward solution. The analysis is divided into the following two cases:
    
    \textit{Case 1 ($\boldsymbol{\xi}(t) \notin \mathcal{M}_r$):} The vector field $\boldsymbol{\chi}$ satisfies the local Lipschitz continuity condition in this region. By the Picard-Lindel\"of theorem \cite[Thm.~3.1]{khalil2002nonlinear}, the trajectory before reaching $\mathcal{M}_r$ is unique.

\textit{Case 2 ($\boldsymbol{\xi}(t) \in \mathcal{M}_r$):} 
From \eqref{eq:doteta}, the radial ratio along any solution satisfies
$\dot{\eta}=2\eta\psi_{\nu}(\eta):=g(\eta)$. By construction,
$g(\eta)>0$ for $\eta\in(0,1)$, $g(1)=0$, and $g(\eta)<0$ for
$\eta\in(1,+\infty)$. Suppose that $\eta(t_0)=1$. We first show that
the solution cannot leave $\eta=1$ in forward time. If there existed
$t_1>t_0$ such that $\eta(t_1)>1$, define $s=\inf\{t\in[t_0,t_1]:\eta(t)>1\}$.
By continuity, $\eta(s)=1$ and there exists $\varepsilon>0$ such that
$\eta(t)>1$ for $t\in(s,s+\varepsilon)$. Hence, for any
$t\in(s,s+\varepsilon)$, $\eta(t)-1 =\int_s^t g(\eta(\tau))\,d\tau\le 0$, because $g(\eta(\tau))<0$ whenever $\eta(\tau)>1$. This contradicts
$\eta(t)>1$. The case $\eta(t_1)<1$ can be proved analogously by using
$g(\eta)>0$ for $\eta\in(0,1)$. Therefore, any solution satisfying
$\eta(t_0)=1$ must satisfy $\eta(t)\equiv1$ for all $t\ge t_0$. Hence,
$\mathcal M_r$ is positively invariant. On $\mathcal M_r$, one has
$\psi_\tau(1)=1$ and $\psi_\nu(1)=0$, and thus \eqref{eq:ODE} reduces to
$\dot{\boldsymbol\xi}=2\boldsymbol E\boldsymbol\xi$, whose solution is
unique. Therefore, forward uniqueness holds after the trajectory reaches
$\mathcal M_r$.
\end{IEEEproof}


\subsection{Proof of Proposition \ref{propos:003}}\label{app:004}

    \begin{IEEEproof}{ Define a Lyapunov function candidate $\Phi(\eta) = (\eta - 1)^2$; then $\dot{\Phi}=-4\eta|\eta-1|\sqrt{1-\psi_\tau^2(\eta)}$ by combining \eqref{eq:equal_psi} and \eqref{eq:doteta}. We abbreviate $\eta(\boldsymbol{\xi}(0))$ as $\eta_0$ and consider the following two cases.

    \textit{Case 1:} For $\eta_0 \in (0,1]$, it follows that $\dot{\eta} \ge 0$, implying $\eta(t) \ge \eta_0$. Therefore,
    \begin{align}
        \dot{\Phi}(\eta)&= -4 \eta (1-\eta)\sqrt{(1-\eta)^{1+\gamma} \big[2 - (1-\eta)^{1+\gamma}\big]} \notag\\
        &\le -4 \eta (1-\eta)^{\frac{3+\gamma}{2}} \le -4\eta_0 (\Phi(\eta))^{\frac{3+\gamma}{4}}.\notag
    \end{align}
    
    \textit{Case 2:} For $\eta_0 \in (1,+\infty)$, it follows that $\dot{\eta} < 0$, implying $\eta(t) < \eta_0$. Then,
    \begin{align}
        \dot{\Phi}(\eta) &= -4 \eta (\eta-1)\sqrt{(1-\eta^{-1})^{1+\gamma} \big[2 - (1-\eta^{-1})^{1+\gamma}\big]} \notag\\
        &\le -4\eta (\eta-1)(1-\eta^{-1})^{\frac{1+\gamma}{2}} 
        \le -4 (\Phi(\eta))^{\frac{3+\gamma}{4}}.\notag
    \end{align}
    
    Summarizing the above two cases, $\Phi$ satisfies the differential inequality $\dot{\Phi} \le -c \Phi^\alpha$, where $\alpha = \frac{3+\gamma}{4} \in (\frac{3}{4}, 1)$ and $c>0$ is a constant. By the finite-time stability theorem~\cite[Thm.~4.2]{doi:10.1137/S0363012997321358},  $\Phi$ converges to zero in finite time. This implies that the trajectory of \eqref{eq:ODE} converges to the circle manifold $\mathcal{M}_r$ at time $T_\chi(\boldsymbol{\xi}_0)$. Moreover, the convergence time satisfies $T_\chi(\boldsymbol{\xi}_0) \le \frac{\Phi(\eta_0)^{1-\alpha}}{c(1-\alpha)}$. Taking $c = 4\eta_0$ and $c = 4$, one obtains \eqref{eq:T_chi}.
    }\end{IEEEproof}

\section{Proofs for Section \ref{sec:4}}
\subsection{Proof of Theorem \ref{thm:admissible_control}}\label{app:005}
    \begin{IEEEproof} First, we prove that $\boldsymbol{u}=[v,\omega]^\top \in \mathcal{U}$. For two constants $a,b>0$, the inequality $0 \le (a^{-1} + b^{-1})^{-1} \le \min(a, b)$ holds. Therefore, $v_r$ in \eqref{eq:cmv} satisfies $v_{-} \le v_r \le \min\big(v_{+}, \bar{\omega}/|\kappa_c|\big)$. From \eqref{eq:vlaw}, by defining $\varrho = 1 - e^{-k_v \Xi^\beta} \in [0, 1)$,  $v = (1-\varrho)v_{-} + \varrho v_r$ is a convex combination of $v_{-}$ and $v_r$, satisfying $v_{-} \le v < v_{r} \le v_+$. For \eqref{eq:omegalaw}, one can calculate that $|\omega| = v |\kappa_c| < v_r |\kappa_c| \le \left(\frac{\bar{\omega}}{|\kappa_c|}\right) |\kappa_c| = \bar{\omega}.$ The trajectory curvature of the robot is $|\kappa_c| \le \bar{\kappa}$ by \eqref{eq:kappa_c_max}. Therefore, the control input satisfies $\boldsymbol{u} \in \mathcal{U}$ and no external saturation operator is required.

    It remains to establish the almost-global $C^1$-smoothness of the proposed method. Let $\mathcal{Q}_{\mathcal{W}}=\{\boldsymbol{q}=[\boldsymbol{\xi}^\top,\theta]^\top\in\mathcal{C}:\boldsymbol{\xi} \in \mathcal{W}\}$ denote the singular set lifted to the configuration space. Since the finite-time design uses $\gamma\in(0,1)$, the functions $\psi_{\nu}$, $\angle\boldsymbol{\chi}$, and $\kappa_{\chi}$ are generally not $C^1$ on the desired manifold $\eta=1$. Consequently, define $\mathcal{Q}_{\mathcal{M}}=\{\boldsymbol{q}\in\mathcal{C}\setminus \mathcal{Q}_{\mathcal{W}}: \eta(\boldsymbol{\xi})=1\}$. Furthermore, define $\mathcal{Q}_1=\{\boldsymbol{q} \in\mathcal{C}\setminus(\mathcal{Q}_{\mathcal{W}}\cup \mathcal{Q}_{\mathcal{M}}):\sin\theta_e=0\}$, $\mathcal{Q}_2=\{\boldsymbol{q} \in\mathcal{C}:\Xi(\boldsymbol{q})=0\}=\{\boldsymbol{q}_d\}$, and $\mathcal{Q}_3=\{\boldsymbol{q}\in\mathcal{C}\setminus(\mathcal{Q}_\mathcal{W}\cup \mathcal{Q}_\mathcal{M}\cup \mathcal{Q}_1):\kappa_c=0\}$. Let $\mathcal{C}^\star=\mathcal{C}\setminus(\mathcal{Q}_\mathcal{W}\cup \mathcal{Q}_\mathcal{M}\cup \mathcal{Q}_1\cup \mathcal{Q}_2\cup \mathcal{Q}_3)$. On $\mathcal{C}^\star$, the vector field is nonsingular with $\eta\neq1$, $\sin\theta_e\neq0$, $\Xi\neq0$, and $\kappa_{c}\neq0$. As a result, variables $\theta_e$, $\kappa_{\chi}$, $\kappa_{c}$, $|\kappa_{c}|$, $v_r(|\kappa_c|)$, and $\Xi^{\beta}$ are $C^1$. Therefore, control inputs $v$ and $\omega=v\kappa_c$ are $C^1$-smooth on $\mathcal{C}^\star$.
    Finally, we demonstrate that the excluded set has zero measure. Specifically, $\mathcal{Q}_{\mathcal{W}}=\mathcal{W}\times\mathbb{S}^1$ is one-dimensional; $\mathcal{Q}_{\mathcal{M}}$ and $\mathcal{Q}_1$ are locally defined by single scalar equations in the three-dimensional configuration manifold; and $\mathcal{Q}_2=\{\boldsymbol{q}_d\}$. Thus, these sets have zero measure. For $\mathcal{Q}_3$, the equation $\kappa_c=0$ yields $\kappa_\chi=(\bar\kappa-\kappa_\chi)\lceil\sin\theta_e\rceil^\alpha$. Since $0<\kappa_\chi<\bar{\kappa}$, this implies $\sin\theta_{e}>0$ and $\sin\theta_e=\left(\frac{\kappa_\chi(\eta)}{\bar{\kappa}-\kappa_\chi(\eta)}\right)^{1/\alpha}$. The right-hand side depends solely on $\eta$. Thus, for each fixed $\boldsymbol{\xi}$, $\theta_e$ has finite admissible values, indicating $\mathcal{Q}_3$ has zero measure. Therefore, $\mathcal{Q}_\mathcal{W}\cup \mathcal{Q}_\mathcal{M}\cup \mathcal{Q}_1\cup \mathcal{Q}_2\cup \mathcal{Q}_3$ has zero measure in $\mathcal{C}$, making $\mathcal{C}^\star$ a full-measure subset. Consequently, the proposed control law is almost globally $C^1$-smooth.
     \end{IEEEproof}

\subsection{The Evolution of $\theta_e$ and $\eta$}\label{app:derivation}
\begin{IEEEproof}
To streamline the derivation, define the orthonormal basis matrices
$\boldsymbol{e}_\chi = [\hat{\boldsymbol{\chi}}, \boldsymbol{E}\hat{\boldsymbol{\chi}}] \in SO(2)$ and $\boldsymbol{e}_\xi = [\hat{\boldsymbol{\xi}}, \boldsymbol{E}\hat{\boldsymbol{\xi}}] \in SO(2)$, associated with the vector field $\boldsymbol{\chi}$ and the radial direction
$\hat{\boldsymbol{\xi}} = \boldsymbol{\xi}/\|\boldsymbol{\xi}\|$, respectively.
The heading vector is represented as $\boldsymbol{d}(\theta) = \boldsymbol{e}_\chi \boldsymbol{h}_e$, where $\boldsymbol{h}_e = [\cos\theta_e,\sin\theta_e]^\top$. Normalizing \eqref{eq:FT-C2VF} yields
$\boldsymbol{e}_\chi = \boldsymbol{e}_\xi \begin{bmatrix} \psi_\nu & -\psi_\tau \\ \psi_\tau & \psi_\nu \end{bmatrix},$
which further gives $\hat{\boldsymbol{\xi}} = \boldsymbol{e}_\chi \begin{bmatrix} \psi_\nu \\ -\psi_\tau \end{bmatrix}.$ Projecting the robot velocity $\dot{\boldsymbol{\xi}} = v\boldsymbol{d}(\theta)$
onto the radial direction $\hat{\boldsymbol{\xi}}$ gives
\begin{align*}
    \dot{\eta} &= \frac{1}{r} \dot{\boldsymbol{\xi}}^\top \hat{\boldsymbol{\xi}} = \frac{v}{r} ( \boldsymbol{e}_\chi \boldsymbol{h}_e )^\top \left( \boldsymbol{e}_\chi \begin{bmatrix} \psi_\nu \\ -\psi_\tau \end{bmatrix} \right) = \eqref{eq:eta_1}.
\end{align*}

For the direction of the vector field $\boldsymbol{\chi}$, its time derivative is given by $\dot{\angle\boldsymbol{\chi}}=\frac{
(\boldsymbol{E}\boldsymbol{\chi})^\top\boldsymbol{J}_{\boldsymbol{\chi}}\dot{\boldsymbol{\xi}}}{\|\boldsymbol{\chi}\|^2}.$ Since
$\dot{\boldsymbol{\xi}}=v\boldsymbol{d}(\theta)$, it follows that
\begin{align}
    \dot{\angle\boldsymbol{\chi}}
    &=
    v
    \underbrace{
    \frac{
    (\boldsymbol{E}\boldsymbol{\chi})^\top
    \boldsymbol{J}_{\boldsymbol{\chi}}\hat{\boldsymbol{\chi}}
    }{
    \|\boldsymbol{\chi}\|^2
    }}_{\kappa_\chi}
    \cos\theta_e
    +
    v
    \underbrace{
    \frac{
    (\boldsymbol{E}\boldsymbol{\chi})^\top
    \boldsymbol{J}_{\boldsymbol{\chi}}\boldsymbol{E}\hat{\boldsymbol{\chi}}
    }{
    \|\boldsymbol{\chi}\|^2
    }}_{\kappa_\perp}
    \sin\theta_e \notag \\
    &=
    v\left(
        \kappa_\chi\cos\theta_e
        +
        \kappa_\perp\sin\theta_e
    \right),
    \label{eq:angle_dot_curvature}
\end{align}
Therefore, combining $\dot{\theta}_e=\omega-\dot{\angle\boldsymbol{\chi}}$ with \eqref{eq:angle_dot_curvature} yields \eqref{eq:theta_e_1}.

\end{IEEEproof}

\subsection{Proof of Theorem \ref{thm:finit_time}}\label{app:007}

\begin{IEEEproof} We prove the result in four steps. 

\textit{1) Absence of additional equilibria on the positive branch:} For $\eta \in (0,1)$, $\psi_{\nu} > 0$ yields $\Upsilon_+ = (\kappa_\chi - \frac{1}{r\eta}) - k_\omega \psi_{\nu}^\alpha$ via \eqref{eq:trap_eq}. Defining $\varepsilon_1 = 1 - \eta \in (0,1)$ and substituting \eqref{eq:kappa_piecewise} provides the upper bound for the first term:
$$    
\kappa_\chi - \frac{1}{r\eta} = \frac{\varepsilon_1^\gamma}{r}\left(1+\gamma - \frac{\varepsilon_1}{1-\varepsilon_1}\right) < \frac{1+\gamma}{r}\varepsilon_1^\gamma.
$$
If $\varepsilon_1 \ge \frac{1+\gamma}{2+\gamma}$, then $\kappa_\chi - \frac{1}{r\eta} \le 0$, trivially ensuring $\Upsilon_+ < 0$. Conversely, if $\varepsilon_1 < \frac{1+\gamma}{2+\gamma}$, the gain $k_\omega$ satisfies
$$    
k_\omega = \bar{\kappa} - \kappa_\chi \overset{\eqref{eq:kappa_max_min}}{>} \bar{\kappa} - \frac{2+\gamma}{r} \overset{\eqref{eq:r_choose}}{\ge} \frac{1+\gamma}{r}.
$$
Furthermore, \eqref{eq:kn_def_a} and \eqref{eq:alpha_choose} imply
$$   
\psi_{\nu}^\alpha = \varepsilon_1^{\frac{\alpha(1+\gamma)}{2}}\left(2-\varepsilon_1^{1+\gamma}\right)^{\frac{\alpha}{2}} > \varepsilon_1^{\frac{\alpha(1+\gamma)}{2}} \ge \varepsilon_1^\gamma.
$$
Combining these bounds yields $k_\omega \psi_{\nu}^\alpha > \frac{1+\gamma}{r} \varepsilon_1^\gamma > \kappa_\chi - \frac{1}{r\eta}$, which confirms $\Upsilon_+ < 0$.

For $\eta \in (1,+\infty)$, $\psi_{\nu} \le 0$ dictates $\Upsilon_+ = (\kappa_\chi - \frac{1}{r\eta}) + k_\omega |\psi_{\nu}|^\alpha$. Letting $\varepsilon_2 = 1 - \eta^{-1} \in (0,1)$ and applying \eqref{eq:kappa_piecewise} gives
$$    
\left|\kappa_\chi - \frac{1}{r\eta}\right| = \left|-\frac{1-\varepsilon_2}{r}\varepsilon_2^\gamma(1+\gamma-\gamma \varepsilon_2)\right| < \frac{1+\gamma}{r}\varepsilon_2^\gamma.
$$
With \eqref{eq:r_choose}, it follows that
$$
k_\omega = \bar{\kappa} - \kappa_\chi \ge \bar{\kappa} - \frac{1}{r} \ge \frac{2(1+\gamma)}{r}.
$$
Additionally, \eqref{eq:kn_def_b} gives $|\psi_{\nu}|^\alpha = \varepsilon_2^{\frac{\alpha(1+\gamma)}{2}}\left(2-\varepsilon_2^{1+\gamma}\right)^{\frac{\alpha}{2}} > \varepsilon_2^{\frac{\alpha(1+\gamma)}{2}}$, which alongside \eqref{eq:alpha_choose} confirms
$$    
k_\omega |\psi_{\nu}|^\alpha > \frac{2(1+\gamma)}{r}\varepsilon_2^{\frac{\alpha(1+\gamma)}{2}} \ge \frac{2(1+\gamma)}{r}\varepsilon_2^\gamma > \left| \kappa_\chi - \frac{1}{r\eta} \right|.
$$
Thus, $\Upsilon_+ > 0$ holds for all $\eta \in (1,+\infty)$. 

\begin{figure}[!t]
\centering
\includegraphics[width=3in]{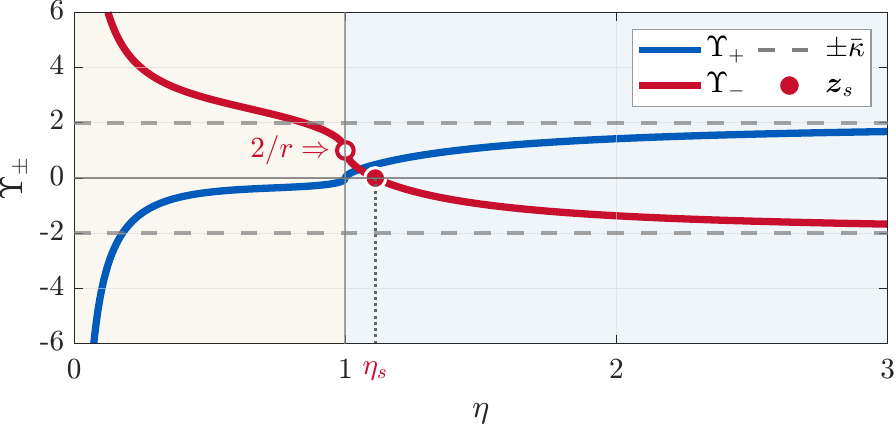}
\caption{Plot of the function $\Upsilon_{\pm}$.}
\label{fig:Upsilon}
\end{figure}

\begin{figure}[!t]
\centering
\includegraphics[width=3in]{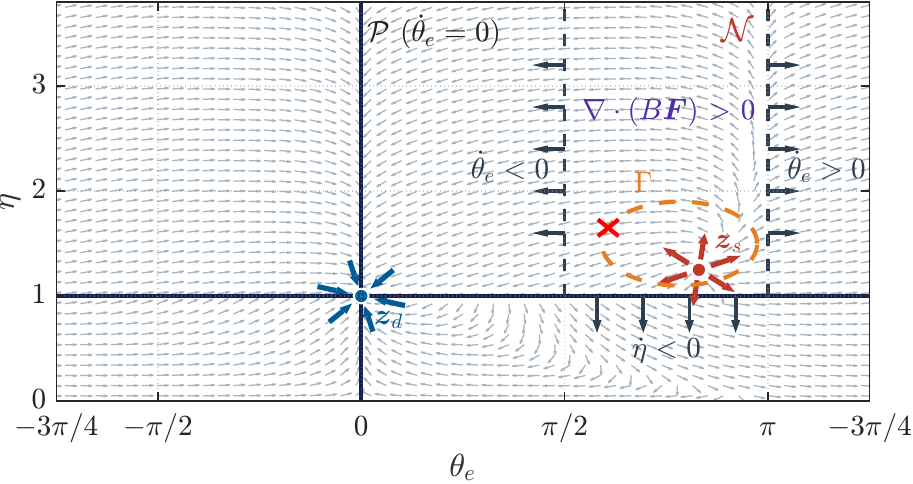}
\caption{Vector field orientation along the boundary of region $\mathcal{N}$, supporting the non-existence of nontrivial periodic orbits.}
\label{fig:limit_fig}
\end{figure}

\begin{figure*}[!ht] 
        \begin{align}
        \boldsymbol{J}(\boldsymbol{z}^*) &=
    \begin{bmatrix}
    \pm \dfrac{1}{\psi_{\nu}} \Big[ \kappa_\chi (1 \mp \alpha \psi_{\tau}) - \dfrac{\psi_{\tau}}{r\eta} (1 - \alpha) \Big] & 
    \kappa'_\chi (1 \mp \psi_{\tau}) - k_\omega' \lceil \pm\psi_{\nu}\rfloor^\alpha\mp \kappa_\perp' \psi_{\nu} \label{eq:J} \\
    \mp \dfrac{1}{r} & 
    \mp \dfrac{\psi_{\tau}'}{r \psi_{\nu}}
    \end{bmatrix},  \\
        \mathrm{Tr}(\boldsymbol{J}) &= \pm \frac{1}{\psi_{\nu}} \Big[ \kappa_\chi (1 \mp \alpha \psi_{\tau}) - \frac{\psi_{\tau}}{r\eta} (1 - \alpha) - \frac{\psi_{\tau}'}{r} \Big], \label{eq:Tr} \\
        \mathrm{Det}(\boldsymbol{J}) &= \pm \frac{1}{r} \Upsilon_\pm'= \pm \frac{1}{r} \Big[ \kappa_\chi' \left( 1 + \lceil\pm \psi_{\nu}\rfloor^\alpha \right)  \mp \alpha k_\omega |\psi_{\nu}|^{\alpha-1} \psi_{\nu}'\pm \frac{1}{r\eta^2} \Big]. \label{eq:Det}
    \end{align}
    \hrulefill
\end{figure*}

\begin{lemma}\label{lemma3}
For any equilibrium point $\boldsymbol{z}^* \in \mathcal{E} \setminus \{\boldsymbol{z}_{d}\}$ of system \eqref{eq:bihuan}, the Jacobian matrix, its trace, and its determinant are given by \eqref{eq:J}-\eqref{eq:Det}. The sign is chosen according to the branch in \eqref{eq:trap_eq} of the equilibrium.
\end{lemma}

\begin{IEEEproof}
For brevity, the independent variable $\eta$ is omitted in the subsequent derivations. Let the Jacobian matrix be denoted as $\boldsymbol{J}(\boldsymbol{z}^*) = [J_{ij}]_{2 \times 2}$. Using the relation $\frac{\partial}{\partial\theta_e}\lceil\sin\theta_e\rfloor^\alpha = \alpha|\sin\theta_e|^{\alpha-1}\cos\theta_e$, the partial derivative of $\dot{\theta}_e$ with respect to $\theta_e$ is computed. Moreover, noting that $\eta^* \neq 1$ and $\psi_{\nu} \neq 0$ at the undesired equilibrium point, the expression for $J_{11}$ can be simplified as
\begin{align}
&J_{11} = \kappa_\chi\sin\theta_e^* - \alpha k_\omega |\sin\theta_e^*|^{\alpha-1} \cos\theta_e^* - \kappa_\perp \cos\theta_e^* \notag \\
&\overset{\eqref{eq:geometric_identity}}{=} \pm \Big[ \kappa_\chi \psi_{\nu} - \alpha k_\omega |\psi_{\nu}|^{\alpha-1} \psi_{\tau} - \kappa_\perp \psi_{\tau} \Big] \label{eq:J11_} \\
&\overset{\eqref{eq:identity},\eqref{eq:trap_eq}}{=} \pm \Bigg[ \kappa_\chi \psi_{\nu} - \frac{\alpha \psi_{\tau}}{\psi_{\nu}} \left( \pm \kappa_\chi - \frac{1}{r\eta} \right) - \left( \frac{1}{r\eta \psi_{\nu}} - \kappa_\chi \frac{\psi_{\tau}}{\psi_{\nu}} \right) \psi_{\tau} \Bigg] \notag \\
&= \pm \frac{1}{\psi_{\nu}} \Bigg[ \kappa_\chi (1 \mp \alpha \psi_{\tau}) - \frac{\psi_{\tau}}{r\eta} (1 - \alpha) \Bigg]. \label{eq:J11}
\end{align}
Similarly, by taking the partial derivatives of the remaining terms in \eqref{eq:bihuan}, the following components are derived:
\begin{align}
J_{12} &= \kappa_\chi' (1 \mp \psi_{\tau}) - k_\omega' \lceil\pm \psi_{\nu}\rfloor^\alpha \mp \kappa_\perp' \psi_{\nu}, \label{eq:J12} \\
J_{21} &= \frac{1}{r} \big( \mp \psi_{\tau}^2  \mp \psi_{\nu}^2\big) = \mp \frac{1}{r},\label{eq:J21} \\
J_{22} &= \pm \frac{1}{r} (\psi_{\nu}' \psi_{\tau} - \psi_{\tau}' \psi_{\nu}) \overset{\eqref{eq:equal_eta}}{=} \mp \frac{\psi_{\tau}'}{r \psi_{\nu}}. \label{eq:J22}
\end{align}
Adding \eqref{eq:J11} and \eqref{eq:J22} yields \eqref{eq:Tr}. The determinant in \eqref{eq:Det} is computed using \eqref{eq:J11_},\eqref{eq:J12}-\eqref{eq:J22}.
\end{IEEEproof}

\textit{2) Uniqueness and instability of the undesired equilibrium $\boldsymbol{z}_s$:} From \eqref{eq:trap_eq}, if $\eta \in (0,1)$, we have $\Upsilon_- = \kappa_\chi + \frac{1}{r\eta} + k_\omega \psi_{\nu}^\alpha > 0$. If $\eta \in (1,+\infty)$, the function becomes $\Upsilon_- = \kappa_\chi + \frac{1}{r\eta} - k_\omega|\psi_{\nu}|^\alpha$. Evaluating the boundaries yields $\Upsilon_-(1) = \frac{2}{r}$ and $\lim_{\eta \to +\infty} \Upsilon_-(\eta) = -\bar{\kappa}$. By the intermediate value theorem, $\Upsilon_-$ has at least one root $\eta_s \in (1,+\infty)$.

Utilizing \eqref{eq:Det}, along with $\kappa_\chi'<0$ in Lemma \ref{lem:001} and $\psi_{\nu}'<0$, we obtain
\begin{equation}\label{eq:eqeqeqeqeq}
    \Upsilon_-' = \kappa'_\chi(1 + |\psi_{\nu}|^\alpha) - \frac{1}{r\eta^2} + \alpha k_\omega |\psi_{\nu}|^{\alpha-1} \psi_{\nu}' < 0.
\end{equation}
Since $\Upsilon_-$ is strictly decreasing, the root $\eta_s$ is unique (see Fig.~\ref{fig:Upsilon}). Substituting $\boldsymbol{z}_s$ into \eqref{eq:Tr}, the trace of the Jacobian evaluates to
\begin{align*}
\mathrm{Tr}(\boldsymbol{J}(\boldsymbol{z}_s)) 
&= -\frac{1}{\psi_{\nu}}  \underbrace{\left( \kappa_\chi - \frac{\psi_{\tau}}{r\eta_s} - \frac{\psi_{\tau}'}{r} \right)}_{\overset{\eqref{eq:lemma1}}{=}0} -\alpha \frac{\psi_{\tau}}{\psi_{\nu}} \left( \kappa_\chi + \frac{1}{r\eta_s} \right)  \notag \\
&= -\frac{\alpha \psi_{\tau}}{\psi_{\nu}} \left( \kappa_\chi + \frac{1}{r\eta_s} \right) > 0,
\end{align*}
noting that $\psi_{\nu}<0$ while all other parameters are positive. From \eqref{eq:Det} and \eqref{eq:eqeqeqeqeq}, the Jacobian determinant is $\mathrm{Det}(\boldsymbol{J}(\boldsymbol{z}_s)) = -\frac{1}{r}\Upsilon_-'(\eta_s) > 0$. Based on standard equilibrium classification criteria \cite[Sect.~2.2]{khalil2002nonlinear}, the eigenvalues of $\boldsymbol{J}(\boldsymbol{z}_s)$ have strictly positive real parts, proving that $\boldsymbol{z}_s$ is an unstable node or focus (i.e., a repeller).

\textit{3) Non-existence of nontrivial periodic orbits:} Consider the one-dimensional $\mathcal{P}=\{z\in\mathcal{H}:\theta_e=0\}$. From \eqref{eq:bihuan}, $\dot{\theta}_{e}|_{\theta_{e}=0}\equiv0,$ so $\mathcal{P}$ cannot be crossed by any trajectory. Hence a nontrivial periodic orbit, if it exists, cannot intersect $\mathcal{P}$; otherwise it would coincide with a trajectory contained in $\mathcal{P}$, which is not a nontrivial closed orbit. Therefore it is contained in $\mathcal{H}\setminus\mathcal{P}$. Here a connected component means a maximal connected subset. On the cylinder, $\mathcal{H}\setminus\mathcal{P}=(\mathbb{S}^1\setminus\{0\})\times\mathbb{R}_{>0}$ is diffeomorphic to the planar strip $(0,2\pi)\times\mathbb{R}_{>0}$ and is in fact connected. Thus any nontrivial
periodic orbit can be represented as a Jordan curve in this planar strip, and the Poincaré index theorem applies \cite[Cor.~2.1]{khalil2002nonlinear}. Since $\boldsymbol{z}_d\in\mathcal{P}$, such a periodic orbit cannot enclose $\boldsymbol{z}_d$ and must enclose $\boldsymbol{z}_s$. At $\boldsymbol{z}_s$, one has $\eta_s > 1$, $\sin\theta_{es} = -\psi_{\nu} > 0$, and $\cos\theta_{es} = -\psi_{\tau} < 0$. Thus, $\boldsymbol{z}_s$ resides in the open region $\mathcal{N} = \{ (\theta_e, \eta) :\theta_e \in (\pi/2, \pi), \eta > 1 \}$. We will establish that $\mathcal{N}$ is negatively invariant; trajectories may exit $\mathcal{N}$ but can never enter it. The boundary of $\mathcal{N}$ comprises three segments:

 B.1: $\eta = 1$ and $\theta_e \in (\pi/2, \pi)$. Here, $\psi_{\nu}(1)=0$, $\psi_{\tau}(1)=1$, and $\sin\theta_e > 0$. Consequently, $\dot{\eta}|_{\eta=1} = -\frac{1}{r}\sin\theta_e < 0$, meaning trajectories cross the boundary downwards.
 
B.2: $\theta_e = \pi$ and $\eta > 1$. Evaluating the vector field yields $\dot{\theta}_e|_{\theta_e=\pi} = 2\kappa_\chi > 0$, meaning trajectories exit to the right.
 
 B.3: $\theta_e = \pi/2$ and $\eta > 1$. Utilizing $\kappa < \frac{1}{r}$ from \eqref{eq:kappa_piecewise}, $\psi_{\nu} < 0$, $\psi_{\nu}' > 0$, and \eqref{eq:kappa_prep}, we deduce $\kappa_\perp = \frac{\psi_{\nu}}{r\eta} + \frac{\psi_{\nu}'}{r} > -\frac{1}{r}$. Therefore,
    \begin{align}\dot{\theta}_e\Big|_{\pi/2} 
    &= 2\kappa(\eta) - \bar{\kappa} - \kappa_\perp(\eta)  \notag\\
    &\overset{\eqref{eq:r_choose}}{<} \frac{2}{r} - \frac{3+2\gamma}{r} - \left(-\frac{1}{r}\right)  = -\frac{2\gamma}{r} < 0,
    \end{align}
causing trajectories to exit to the left.

Therefore, any hypothetical periodic orbit enclosing $\boldsymbol{z}_s$ must lie entirely inside $\mathcal{N}$. With the Dulac function $B(\eta) = \eta > 0$, the divergence of the weighted vector field $B\boldsymbol{F}$ is
\begin{equation*}\nabla \cdot (B\boldsymbol{F}) =-\alpha \eta k_\omega |\sin\theta_e|^{\alpha-1} \cos\theta_e >0
\end{equation*}
in $\mathcal{N}$, as shown in Fig.~\ref{fig:limit_fig}. The Bendixson-Dulac criterion \cite[Prop.~1.207]{2006Ordinary} excludes all periodic orbits in  $\mathcal{N}$. Therefore \eqref{eq:bihuan} admits no nontrivial periodic orbits, and hence no limit cycles.

\begin{lemma}\label{lemma:sbas}
    For any initial condition $\boldsymbol{z}_0\in \mathcal{H}$, let $\boldsymbol{z}(\cdot;\boldsymbol{z}_0)$ denote the maximal forward solution, defined on $[0,T_{\max}(\boldsymbol{z}_0))$. Define
    $\mathcal{W}_\eta:=\left\{\boldsymbol{z}_0\in\mathcal{H}:\inf_{0\leq t<T_{\max}(\boldsymbol{z}_0)}\eta(t;\boldsymbol{z}_0)=0\right\}.$
    Then $\mathcal{W}_\eta$ has measure zero in $\mathcal{H}$.
\end{lemma}
\begin{figure}[!t]
\centering
\includegraphics[width=1in]{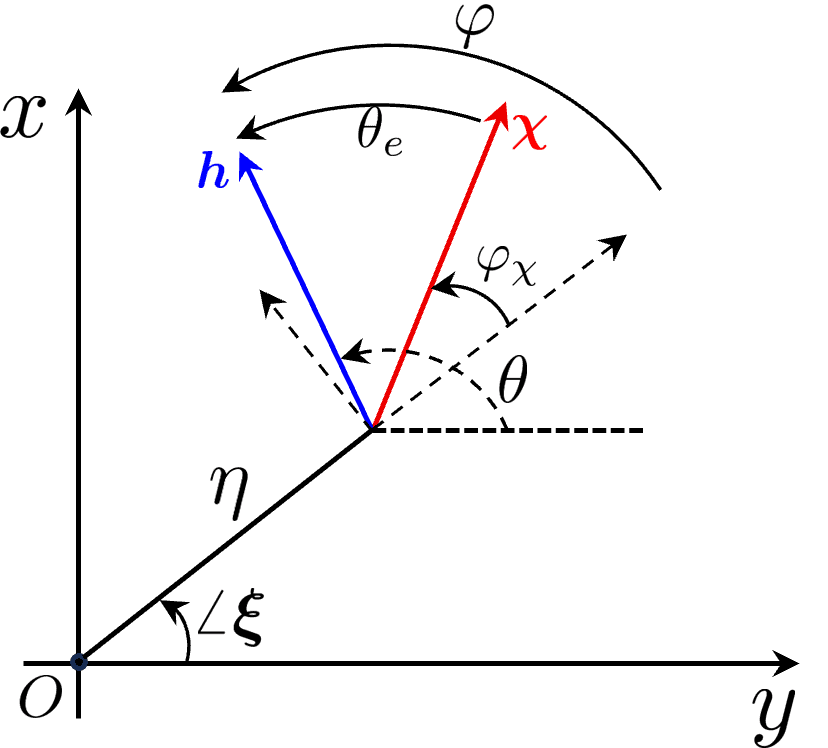}
\caption{Angular relationships in the $(\eta,\varphi)$ coordinates.}
\label{fig:shiyi}
\end{figure}
\begin{IEEEproof}
Let $\varphi_\chi(\eta)=\operatorname{atan2}(\psi_\tau(\eta),\psi_\nu(\eta))$ and $\varphi=\theta_e+\varphi_\chi(\eta)$, as shown in \cref{fig:shiyi}. Since $\cos\varphi_\chi=\psi_\nu$ and $\sin\varphi_\chi=\psi_\tau$, \eqref{eq:bihuan} gives $\dot\eta=\frac{1}{r}\cos\varphi$. For $\eta>0$, multiplying the vector field by the positive factor $\eta$ only reparameterizes time and preserves the orbits. Hence, by orbital equivalence \cite[Def.~2.4]{Kuznetsov2023ElementsOA}, the phase trajectories are identical to those of
\begin{equation}\label{eq:polar_sys}
\eta'=\frac{\eta}{r}\cos\varphi,\quad\varphi'=\eta\kappa_c(\eta,\varphi-\varphi_\chi(\eta))-\frac{1}{r}\sin\varphi.
\end{equation}
It admits a continuous extension to $\eta=0$. On this boundary, $\eta'=0$ and $\varphi'=-\frac{1}{r}\sin\varphi$, so the only boundary equilibria are $\boldsymbol{p}_0=(0,0)$ and $\boldsymbol{p}_{\pi}=(0,\pi)$. Near $\boldsymbol{p}_0$, we have $\cos\varphi>0$, and hence $\eta'=\frac{\eta}{r}\cos\varphi>0$ for $\eta>0$. Thus, no interior trajectory can approach $\boldsymbol{p}_0$ in forward time. Therefore, any interior trajectory satisfying $\inf_{0\le t<T_{\max}}\eta(t)=0$ can approach the boundary only through $\boldsymbol{p}_{\pi}$. To analyze $\boldsymbol{p}_{\pi}$, let $\varpi=\varphi-\pi$ and define $K_0(\eta,\varpi)=\kappa_c(\eta,\pi+\varpi-\varphi_\chi(\eta))$. Then, near $\boldsymbol{p}_{\pi}$, system \eqref{eq:polar_sys} becomes
\begin{equation}\label{eq:polar_sys1}
\eta'=-\frac{\eta}{r}\cos\varpi,\quad\varpi'=\frac{1}{r}\sin\varpi+\eta K_0(\eta,\varpi).
\end{equation}
Since $\kappa_c$ is bounded near $\eta=0$, $K_0$ is also bounded. Choose a sufficiently small $\delta>0$ such that $|\sin\varpi|\geq|\varpi|/2$ for $0<|\varpi|<\delta$, and select $M>2r|K_0|$. For $\varpi>M\eta$, $\varpi^{\prime}\geq\frac{\varpi}{2r}-|K_0|\eta>\left(\frac{M}{2r}-|K_0|\right)\eta>0$. For $\varpi<-M\eta$, $\varpi^{\prime}\leq\frac{\varpi}{2r}+|K_0|\eta<-\left(\frac{M}{2r}-|K_0|\right)\eta<0$. Hence, in a sufficiently small neighborhood outside the conical region $\mathcal S_{M,\delta}=\{(\eta,\varpi):0<\eta<\delta,\ |\varpi|\le M\eta\}$, $\varpi\varpi^{\prime}>0$, indicating that $|\varpi|$ strictly increases. Thus, trajectories cannot converge to $(0,0)$. Any internal trajectory approaching $\boldsymbol{p}_{\pi}$ must eventually enter $\mathcal S_{M,\delta}$. Within $\mathcal S_{M,\delta}$, the blow-up transformation $x_1=\eta^\alpha$ and $x_2=\frac{\varpi}{\eta}$ yields $\eta=x_1^{\frac{1}{\alpha}}$ and $\varpi=x_1^{\frac{1}{\alpha}}x_2$ \cite[Sect.~7.1]{Kuehn2015Multiple}. Applying this to \eqref{eq:polar_sys1} generates the blown-up system:
\begin{align*} 
x_1'&=-\frac{\alpha}{r}x_1\cos(x_1^{\frac{1}{\alpha}}x_2),\\ x_2'&=K(x_1,x_2)+\frac{x_2}{r}\left[\operatorname{sinc}(x_1^{\frac{1}{\alpha}}x_2)+\cos(x_1^{\frac{1}{\alpha}}x_2)\right],\end{align*}
where $K(x_1,x_2)=\kappa_c\left(x_1^{\frac{1}{\alpha}},\pi+x_1^{\frac{1}{\alpha}}x_2-\varphi_\chi(x_1^{\frac{1}{\alpha}})\right)$ and $\operatorname{sinc}\varpi=\frac{\sin\varpi}{\varpi}$. Next, $K$ is verified to be $C^1$ near the considered equilibrium point. The Taylor expansions of \eqref{eq:kt_def} and \eqref{eq:kn_def} yield $\psi_\tau(\eta)=(1+\gamma)\eta+O(\eta^2)$ and $\psi_\nu(\eta)=1+O(\eta^2)$. These subsequently give $\varphi_\chi(\eta)=(1+\gamma)\eta+O(\eta^2)$, ensuring the ratio $\frac{\varphi_\chi(\eta)}{\eta}$ admits a $C^1$ extension at $\eta=0$ with $\lim_{\eta \to 0^+}\frac{\varphi_\chi(\eta)}{\eta} = 1+\gamma$. Furthermore, \eqref{eq:lemma1} and \eqref{eq:kappa_max_min} imply $\kappa_\chi(\eta)=\frac{2(1+\gamma)}{r}+O(\eta)$. The fractional power term constitutes the only potential source of non-smoothness in $\kappa_c$. Define $\Lambda(\eta,x_2)=\left(\frac{\varphi_\chi(\eta)}{\eta}-x_2\right)\operatorname{sinc}\left(\varphi_\chi(\eta)-\eta x_2\right)$, and extend it to $\eta=0$ by setting $\Lambda(0,x_2)=1+\gamma-x_2$. Because both $\frac{\varphi_\chi(\eta)}{\eta}$ and the $\operatorname{sinc}$ function are $C^1$, $\Lambda$ remains $C^1$ within the analyzed neighborhood. Applying the relation $\sin(\cdot)=(\cdot)\operatorname{sinc}(\cdot)$ yields $\sin\left(\pi+\eta x_2-\varphi_\chi(\eta)\right)=\eta\Lambda(\eta,x_2)$. At $x_1=0$, $K(0,x_2)=\kappa_\chi(0)$, which simplifies the dynamics to $x_2'=\kappa_\chi(0)+\frac{2}{r}x_2$. Consequently, the blown-up system possesses an equilibrium point at $\bar{\boldsymbol{p}}_{\pi}=(0,x_{2\pi})$, where $x_{2\pi}=-\frac{r\kappa_\chi(0)}{2}=-(1+\gamma)$. Substituting this value yields $\Lambda(0,x_{2\pi})=1+\gamma-x_{2\pi}=2(1+\gamma)>0$. By continuity, the condition $\Lambda(\eta,x_2)>0$ holds within a sufficiently small neighborhood of $(0,x_{2\pi})$. Therefore, within this local region, $\left\lceil\sin\left(\pi+\eta x_2-\varphi_\chi(\eta)\right)\right\rceil^\alpha=x_1\Lambda(x_1^{\frac{1}{\alpha}},x_2)^\alpha.$ Because $\Lambda>0$ and all constituent terms are $C^1$, the blown-up system is $C^1$ around this equilibrium point.

Evaluating the Jacobian matrix at the equilibrium $\bar{\boldsymbol{p}}_{\pi}$ yields $\boldsymbol{J}(\bar{\boldsymbol{p}}_{\pi})=\left[\begin{smallmatrix}-\frac{\alpha}{r} & 0 \\ * & \frac{2}{r}\end{smallmatrix}\right].$ The corresponding eigenvalues are $-\frac{\alpha}{r}<0$ and $\frac{2}{r}>0$, which confirms that $\bar{\boldsymbol{p}}_{\pi}$ is a saddle point. Based on the local stable manifold theorem \cite[Thm.~9.4]{teschl2012ordinary}, the local stable set of $\bar{\boldsymbol{p}}_{\pi}$ constitutes a one-dimensional submanifold. Under the blow-down mapping $\Pi(x_1,x_2)\mapsto(\eta,\varphi)=(x_1^{\frac{1}{\alpha}},\pi+x_1^{\frac{1}{\alpha}}x_2)$, the image of this stable manifold in the original $(\eta,\varphi)$ plane remains a one-dimensional submanifold. Because $\mathcal H$ is a two-dimensional manifold, any one-dimensional submanifold inherently has measure zero. Consequently, the set of initial states converging to $\boldsymbol{p}_\pi$ has measure zero. All initial conditions that satisfy $\inf_{0\le t<T_{\max}}\eta(t)=0$ are contained within the global backward-time trajectories of this one-dimensional stable manifold. The uniqueness of solutions ensures that the global stable set preserves its one-dimensional structure, thereby remaining a measure-zero subset of $\mathcal H$. Finally, the coordinate transformation $(\theta_e,\eta)\mapsto(\varphi,\eta)$ is a local $C^1$ diffeomorphism, and the orbital equivalence transformation preserves the trajectory sets. Therefore, the set $\mathcal{W}_\eta$ has measure zero.
\end{IEEEproof}

\textit{4) Almost-global finite-time stability of $\boldsymbol{z}_d$:}  For any initial condition $\boldsymbol{z}_0 \in\mathcal{H}\setminus(\mathcal{W}_\eta\cup\{\boldsymbol{z}_s\}),$ the definition of $\mathcal{W}_\eta$ gives a constant $\varepsilon>0$ such that $\eta(t)\geq\varepsilon$ for all $0\leq t<T_{\max}$. It remains to exclude escape to infinity. As $\eta \to +\infty$, $\psi_\nu(\eta)\to-1,\psi_\tau(\eta)\to0,\kappa_\chi(\eta)\to0,\kappa_\perp(\eta)\to0.$ Therefore, for sufficiently large $\eta$, the heading-error dynamics in \eqref{eq:bihuan} is dominated by the term $-\bar{\kappa}\left\lceil\sin\theta_e\right\rfloor^\alpha$. In finite time the heading error enters an arbitrarily small neighborhood of $\theta_e = 0$; in that neighborhood, the radial dynamics satisfies $\dot{\eta}<0$ because $\psi_\nu < 0$. Consequently, no positive semi-orbit can escape to $\eta=+\infty$. Hence every solution starting outside $\mathcal{W}_\eta$  is ultimately contained in a compact
annulus $\mathcal{A}_{\varepsilon,R} = [\varepsilon,R]\times \mathbb{S}^1,0<\varepsilon<R<+\infty$. Its $\omega$-limit set is thus nonempty, compact, connected, and invariant \cite[Sect.~6.3]{teschl2012ordinary}. The Poincaré–Bendixson theorem \cite[Lemma~2.1]{khalil2002nonlinear}, applied on this annulus, implies that the $\omega$-limit set must contain an equilibrium or a periodic orbit. Nontrivial periodic orbits have been excluded, and the only equilibria are $z_{d}$ and the repeller $z_{s}$. Therefore every trajectory starting from $\mathcal{H}\setminus(\mathcal{W}_\eta\cup\{\boldsymbol{z}_s\})$ converges asymptotically to $z_{d}$. For any compact neighborhood $\mathcal{B}_\delta = \{\boldsymbol{z}:\|\boldsymbol{z} - \boldsymbol{z}_0\|\leq \delta\}$, there exists a finite time $T_1 > 0$ such that trajectories enter and remain in $\mathcal{B}_\delta$ for all $t \ge T_1$. Let $\Delta_1 = \eta - 1$ and $\Delta_2 = \sin\theta_e$. Inside a sufficiently small neighborhood $\mathcal{B}_{\delta^\prime}$, $\Delta_1\to 0$ and $\Delta_2\to 0$ as $\boldsymbol{z}\to \boldsymbol{z}_0$.

Consider the Lyapunov candidate $V_\theta = 1 - \cos\theta_e \ge 0$, whose time derivative is
\begin{equation}\label{eq:dotv}
\dot{V}_\theta = \kappa_\chi\sin\theta_e (1-\cos\theta_e) - \kappa_\perp\sin^2\theta_e - k_\omega|\sin\theta_e|^{1+\alpha}.
\end{equation}
To analyze the term $-\kappa_\perp\sin^2\theta_e$, we define the ratio $\Theta(\eta, \theta_e) = \frac{|\kappa_\perp\sin^2\theta_e|}{|\sin\theta_e|^{1+\alpha}} = |\kappa_\perp| |\sin\theta_e|^{1-\alpha}.$
The following lemma is presented to estimate this ratio.

\begin{lemma}\label{eq:lemma_delta_relation}
     There exists a constant $C>0$ and a finite time $T_f<+\infty$ such that for all $t > T_f$,
    $$|\Delta_2(t)| \leq C|\Delta_1(t)|^{\frac{1+\gamma}{2}}.$$
\end{lemma}
\begin{IEEEproof} Let $m = \frac{1+\gamma}{2}$. Since $0<\gamma<1$, $m\in(\frac{1}{2},1)$. As $(\Delta_1,\Delta_2)\to(0,0)$, \eqref{eq:kt_def}, \eqref{eq:kn_def}, and \eqref{eq:kappa_prep} yield $\psi_\tau-1=O(|\Delta_1|^{2m})$, $\psi_\nu=O(|\Delta_1|^m)$, and $\kappa_\perp=O(|\Delta_1|^{m-1})$. From \eqref{eq:kappa_max_min} and \eqref{eq:r_choose}, there exists a constant $\underline{k}>0$ such that the gain $k_\omega = \bar{\kappa} - \kappa_\chi \ge \underline{k} >0$. If $\Delta_2$ does not reach zero in finite time, its continuity and $\Delta_2(t)\to 0$ ensure its sign eventually becomes constant. Assume $\Delta_2>0$ eventually. The trajectory then resides in the $\Delta_1<0$ region. To verify this, assume $\Delta_1>0$ is sufficiently small. Equation \eqref{eq:bihuan} yields $\dot{\Delta}_1=\frac1r\left(\psi_\nu\cos\theta_e-\psi_\tau\Delta_2\right)$. Since $\psi_\nu<0$, $\psi_\tau>0$, and $\cos\theta_e>0$, $\dot{\Delta}_1\leq -c_0\Delta_1^m$ for some constant $c_0>0$. Thus, $\Delta_1$ cannot remain positive while converging to zero. Moreover, at $\Delta_1=0$ and $\Delta_2>0$, the radial equation yields $\dot{\Delta}_1=-\frac{\Delta_2}{r}<0$. Hence, there exists a finite time $T_1$ such that $\Delta_1(t)<0$ and $\Delta_2(t)>0$ for all $t>T_1$. For $\Delta_1<0$ and $\Delta_2>0$, choose a sufficiently large constant $M>0$ and define $g=\Delta_2-M(-\Delta_1)^m$. We demonstrate the trajectory enters $g<0$ in finite time and remains there. On the boundary $g=0$, $\Delta_2=M(-\Delta_1)^m$. Combining this with the radial equation for a sufficiently large $M$ yields the following inequality for some constant $c_1>0$:
\begin{equation}\label{eq:delta_1}
-\dot{\Delta}_1\geq c_1(-\Delta_1)^m\end{equation}
Since $\Delta_2=\sin\theta_e$, \eqref{eq:bihuan} gives
$$\dot{\Delta}_2=\cos\theta_e\,\dot{\theta}_e=
\cos\theta_e\left[
\kappa_\chi(1-\cos\theta_e)-\kappa_\perp\Delta_2-k_\omega\Delta_2^\alpha
\right].$$ 
Applying the previous order estimates on $g=0$ yields $\kappa_\chi(1-\cos\theta_e)=O(\Delta_2^2)=O((-\Delta_1)^{2m})$, $-\kappa_\perp\Delta_2=O(|\Delta_1|^{m-1})\,O((-\Delta_1)^m)=O((-\Delta_1)^{2m-1})$, and $-k_\omega\Delta_2^\alpha\leq -\underline{k}M^\alpha(-\Delta_1)^{m\alpha}$. From \eqref{eq:alpha_choose}, $m\alpha<2m-1<2m$. As $\Delta_1\to0^-$, the negative term $-\underline{k}M^\alpha(-\Delta_1)^{m\alpha}$ dominates. In a sufficiently small neighborhood on $g=0$, the following holds for some constant $c_2>0$:
\begin{equation}
\label{eq:delta_2}
\dot{\Delta}_2\leq -c_2(-\Delta_1)^{m\alpha}\end{equation}
\eqref{eq:delta_1} and \eqref{eq:delta_2} together imply
\begin{align*}
\dot g &=\dot{\Delta}_2-Mm(-\Delta_1)^{m-1}(-\dot{\Delta}_1)\\& \leq -c_2(-\Delta_1)^{m\alpha}-Mmc_1(-\Delta_1)^{2m-1}\leq0.
\end{align*}
The boundary $g=0$ is thus unidirectional, preventing the trajectory from escaping $g<0$ to $g>0$. If the trajectory never enters $g<0$ in finite time, eventually $g\geq0$, meaning $\Delta_2\geq M(-\Delta_1)^m$. Applying the radial equation and order estimates with a sufficiently large $M$ yields $-\dot{\Delta}_1 \geq c_3(-\Delta_1)^m>0$ for some constant $c_3>0$. This contradicts $-\Delta_1(t)\to0$. The trajectory must enter $g<0$ in finite time and remain there due to this unidirectionality. Thus, there exists a finite time $T_f$ such that for all $t>T_f$, $\Delta_2(t)\leq M(-\Delta_1(t))^m=M|\Delta_1(t)|^m.$ The $\Delta_2<0$ case follows by similar reasoning and is proven by defining $\tilde g=-\Delta_2-M\Delta_1^m$.
\end{IEEEproof} 

Based on Lemma \ref{eq:lemma_delta_relation}, the limit of $\Theta(\eta, \theta_e)$ can be estimated. Near $z_d$, $\kappa_\perp(\Delta_1)=O(|\Delta_1|^{\frac{\gamma-1}{2}})$ implies that $|\kappa_\perp|\le C_\perp|\Delta_1|^{m-1}$ holds for some constant $C_\perp>0$. Therefore, in conjunction with Lemma \ref{eq:lemma_delta_relation},
\begin{align*}
   \lim_{\boldsymbol{z} \to \boldsymbol{z}_{d}} \Theta(\eta, \theta_e)
    &=
    \lim_{\Delta_1 \to 0} |\kappa_\perp||\sin\theta_e|^{1-\alpha} \\
    &\le
    \lim_{\Delta_1 \to 0} C_\perp C^{1-\alpha}
    |\Delta_1|^{\gamma -\frac{(1+\gamma)\alpha}{2}} 
\end{align*}
Given the parameter constraint \eqref{eq:alpha_choose}, we guarantee that
$\lim_{\boldsymbol{z} \to \boldsymbol{z}_{d}} \Theta(\eta, \theta_e) = 0. $ Thus, there exists a finite time $T_2 \ge T_1$ ensuring $\Theta(\eta, \theta_e) \le \frac{1}{4}\underline{k}$ for all $t \ge T_2$.

Similarly, using \eqref{eq:lemma1} and \eqref{eq:kt_def}, we see that the term $\kappa_\chi\sin\theta_e(1-\cos\theta_e)= O(|\Delta_2|^3)$ is strictly dominated by the stabilizing term $k_\omega|\sin\theta_e|^{1+\alpha} = O(|\Delta_2|^{1+\alpha})$ within a smaller neighborhood $\mathcal{B}_\delta \subseteq \mathcal{B}_{\delta^\prime}$. Hence, a finite time $T_3 \ge T_2$ exists such that $\frac{\kappa_\chi\sin\theta_e(1-\cos\theta_e)}{k_\omega|\sin\theta_e|^{1+\alpha}} \le \frac{1}{4}\underline{k}$ for all $t \ge T_3$.

Consequently, for all $t \ge T_3$, the combined disturbance is bounded above by
$\big| \kappa_\chi\sin\theta_e (1-\cos\theta_e) - \kappa_\perp\sin^2\theta_e \big| \le  \frac{1}{2} \underline{k}|\sin\theta_e|^{1+\alpha}$. Substituting this bound into \eqref{eq:dotv}, we obtain
\begin{equation}\label{eq:yxsj}
\dot{V}_{\theta}\leq \left(\frac{1}{2} \underline{k} - k_\omega \right)|\sin\theta_{e}|^{1+\alpha} \leq -\frac{1}{2} \underline{k} |\sin\theta_{e}|^{1+\alpha}.
\end{equation}
Inside $\mathcal{B}_\delta$, the heading error obeys $|\theta_e| \le \delta < \pi$, so $|\sin\theta_e| = 2\left|\sin\left(\frac{\theta_e}{2}\right)\right|\cos\left(\frac{\theta_e}{2}\right) \ge \sqrt{2}\cos\left(\frac{\delta}{2}\right) V_\theta^{\frac{1}{2}}$. Inserting this into \eqref{eq:yxsj} results in $\dot{V}_{\theta} \leq -c_\theta V_{\theta}^{\frac{1+\alpha}{2}}$, where $c_\theta = \frac{1}{2}\underline{k}\left(\sqrt{2}\cos\left(\frac{\delta}{2}\right)\right)^{1+\alpha} > 0$ and $\frac{1+\alpha}{2} \in (0.5, 1)$. By the finite-time stability theorem, $V_\theta$ reaches zero in finite time, indicating $\theta_e = 0$ is achieved in finite time, after which the trajectory remains trapped on the invariant manifold $\mathcal{P}$. Subsequently, applying Proposition \ref{propos:003}, the trajectory converges to the circle manifold $\mathcal{M}_r$ in finite time, concluding the proof of almost-global finite-time stability for $\boldsymbol{z}_{d}$. \end{IEEEproof}

\subsection{Proof of Theorem \ref{thm:velocity_convergence}}\label{app:008}
\begin{IEEEproof}
 By Theorem \ref{thm:finit_time}, trajectories starting outside $\mathcal{Q}_c$ reach $\mathcal{M}_r$ in finite time, achieving $\eta = 1$ and $\theta_e = 0$. The heading thus aligns with the FT-C2VF, simplifying \eqref{eq:Xi} to $\Xi = \|\boldsymbol{\xi} - \boldsymbol{\xi}_d\|^2$. To analyze the subsequent motion along $\mathcal{M}_r$, let $s(t) \in [0, 2\pi r)$ denote the remaining arc length to $\boldsymbol{\xi}_d$, governed by $\dot{s} = -v$. Geometrically, the metric evaluates to $\Xi(s) = \|\boldsymbol{\xi} - \boldsymbol{\xi}_d\|^2 = 4r^2 \sin^2\left(\frac{s}{2r}\right)$. For $s \in (0, 2\pi r)$, $\Xi(s) > 0$ guarantees $v > v_- \ge 0$, ensuring $s(t)$ strictly decreases. We consider two cases for $v_-$.

\textit{Case 1 ($v_- > 0$):} Since \eqref{eq:vlaw} ensures $v \ge v_-$, we have $\dot{s} \le -v_-$. Thus, $s(t)$ reaches zero in finite time bounded by $\Delta T = 2\pi r / v_-$. Thereafter, the robot traverses $\mathcal{M}_r$ periodically, revisiting $\boldsymbol{q}_d$ every $\Delta T$.

\textit{Case 2 ($v_- = 0$):} Here, the velocity reduces to $v = v_r \big(1 - e^{-k_v \Xi^\beta(s)}\big)$ with $v_r > 0$. As $s(t)$ strictly decreases, it enters $s \in [0, \pi r]$ in finite time. Applying Jordan's inequality ($\sin x \ge \frac{2}{\pi}x$ for $x \in [0, \pi/2]$) yields $\Xi(s) \ge 4r^2 \left( \frac{2}{\pi} \frac{s}{2r} \right)^2 = \frac{4}{\pi^2} s^2$. Furthermore, for a sufficiently small $x \ge 0$, there exists a constant $c_1 > 0$ such that $1 - e^{-x} \ge c_1 x$. Since $\Xi(s) \to 0$ as $s \to 0$, the velocity is ultimately lower-bounded by
$$
v \ge v_r c_1 k_v \Xi^\beta(s) \ge v_r c_1 k_v \left(\frac{4}{\pi^2}\right)^\beta s^{2\beta} = c_2 s^{2\beta},
$$
where $c_2 = v_r c_1 k_v \big(\frac{4}{\pi^2}\big)^\beta > 0$. This yields $\dot{s} \le -c_2 s^{2\beta}$. Because $\beta \in (0, 0.5)$ implies $2\beta \in (0, 1)$, $s(t)$ reaches zero in finite time. Upon reaching $s=0$, $\Xi(0)=0$ yields $v=0$, and the trajectory remains stationary at $\boldsymbol{q}_d$ afterwards.
\end{IEEEproof}

\bibliographystyle{IEEEtran}
\bibliography{IEEEabrv.bib}

\end{document}